\documentclass[twoside,11pt]{article}

\usepackage{amsthm}
\usepackage{blindtext}

\usepackage[preprint,nohyperref ]{jmlr2e}

\usepackage[english]{babel}

\usepackage[utf8]{inputenc} 
\usepackage[T1]{fontenc}    
\usepackage[colorlinks=true]{hyperref}       
 \usepackage{url}            
 \usepackage{booktabs}       
\usepackage{amsfonts}       
\usepackage{nicefrac}       
\usepackage{microtype}      
\usepackage{xcolor}         
\usepackage{float}          
\usepackage{amsmath}

\usepackage{centernot}
\usepackage{bm}
\usepackage{mathrsfs}
\usepackage{algorithm}
\usepackage{algpseudocode}
\usepackage{subcaption}
\usepackage{mathtools}

\newcommand{\R}{\mathbb{R}}

\newcommand{\rd}{\mathrm{d}}

\DeclareMathOperator*{\argmin}{\arg\!\min}
\DeclareMathOperator*{\argmax}{\arg\!\max}
\DeclareMathOperator*{\rank}{rank}

\newtheorem{thm}{Theorem}
\newtheorem{prop}{Proposition}
\newtheorem{lem}{Lemma}

\newtheorem{assumption}{\protect\assumptionname\ignorespaces}
\providecommand{\assumptionname}{A}

\usepackage{lastpage}
\jmlrheading{23}{2026}{1-\pageref{LastPage}}{1/21; Revised 5/22}{9/22}{21-0000}{Hannah Sansford, Nick Whiteley and Patrick Rubin-Delanchy}

\usepackage{todonotes}

\firstpageno{1}

\begin{document}

\title{Uncovering Locally Low-dimensional Structure in Networks by Locally Optimal Spectral Embedding}

\author{\name Hannah Sansford \email hannah.sansford@bristol.ac.uk \\
       \addr School of Mathematics\\
       University of Bristol, UK\\
       \AND
       \name Nick Whiteley \email nick.whiteley@bristol.ac.uk \\
       \addr School of Mathematics\\
       University of Bristol, UK\\
       \AND
       \name Patrick Rubin-Delanchy \email patrick.rubin-delanchy@ed.ac.uk \\
       \addr School of Mathematics\\
       University of Edinburgh, UK\\}
\maketitle

\begin{abstract}%
Standard Adjacency Spectral Embedding (ASE) relies on a global low-rank assumption often incompatible with the sparse, transitive structure of real-world networks, causing local geometric features to be ‘smeared’. To address this, we introduce Local Adjacency Spectral Embedding (LASE), which uncovers locally low-dimensional structure via weighted spectral decomposition. Under a latent position model with a kernel feature map, we treat the image of latent positions as a locally low-dimensional set in infinite-dimensional feature space. We establish finite-sample bounds quantifying the trade-off between the statistical cost of localisation and the reduced truncation error achieved by targeting a locally low-dimensional region of the embedding. Furthermore, we prove that sufficient localisation induces rapid spectral decay and the emergence of a distinct spectral gap, theoretically justifying low-dimensional local embeddings. Experiments on synthetic and real networks show that LASE improves local reconstruction and visualisation over global and subgraph baselines, and we introduce UMAP-LASE for assembling overlapping local embeddings into high-fidelity global visualisations.

\end{abstract}

\begin{keywords}
spectral embedding, graph representation learning, manifold learning, dimension reduction, network analysis
\end{keywords}

\section{Introduction}
Spectral embedding \citep{sussman2012consistent} provides a reformulation and generalisation of principal component analysis  (PCA) for dimension reduction, to operate on an input adjacency (or similarity matrix), rather than a set of vectors. 
This can be of interest when the only information we have about objects is how they connect, or, more often, when a complex dataset has some element of connectivity that we wish to exploit. This is the situation usually faced in the analysis of complex systems such as social networks \citep{o2008analysis}, transport \citep{shanmukhappa2018spatial}, computer hardware \citep{fyrbiak2019graph} and software \citep{boldi2020fine}, AI models \citep{ameisen2025circuit}, computer networks \citep{adams2016dynamic}, the brain \citep{de2014graph}, ecosystems \citep{chiu2011unifying}, knowledge bases \citep{peng2023knowledge}, and more. 

The methodological applications of spectral embedding are diverse and generally analogous to those of PCA in helping reduce noise as well as dimension. As PCA followed by $K$-means is rate-optimal under a Gaussian mixture model \citep{loffler2021optimality}, spectral clustering (spectral embedding followed by clustering) is asymptotically efficient under a stochastic block model \citep{tang2022asymptotically} --- and a pragmatic choice in practice. Replacing the clustering step with other types of analysis, we can target other models, such as the mixed membership \citep{airoldi2008mixed,rubin2017consistency}, degree-corrected\citep{karrer2011stochastic,jin2015fast,qin2013regularized,lyzinski2014perfect,lei2015consistency,passino2020spectral}, popularity adjusted \citep{sengupta2018block,koo2023popularity}, and t-stochastic block models \citep{fang2023t}. The scope of spectral embedding extends well beyond these simple statistical models and in practice it is integrated in many techniques, such as UMAP \citep{mcinnes2018umap} and graph neural networks \citep{wu2020comprehensive}, and probably many technologies that we use every day, such as internet search \citep{larson2020making,liu2011spectral}.

Applied researchers are often frustrated by a collection of related issues. Low-dimensional embeddings are seen to ``smear'' local structure \citep{seshadhri2020impossibility}, and statistical performance across various tasks continues to improve substantially as one increases the embedding dimension \citep{sansford2023implications}, to a point where computation becomes unmanageable. To combat this, researchers may turn to other graph embedding methods, e.g. DeepWalk \citep{perozzi2014deepwalk,DeepWalk_ToT} or node2vec \citep{grover2016node2vec}, thought to provide a better view of local structure. However, these methods lack the direct simplicity of PCA and its theoretical understanding, and may be harder to integrate safely within larger pipelines. Alternatively, researchers may apply spectral embedding to a subgraph corresponding to a neighbourhood of interest, and find they are able to obtain a sharper view of that neighbourhood in low dimension \citep{sansford2023implications}. The realities of graph-analysis as part of a more complex problem usually make this possible: there is additional information about at least some of the nodes or edges (e.g. spatial location) that allows us to define a neighbourhood of interest. More broadly, even when a neighbourhood is not specified a priori, local graph partitioning methods can recover an informative local subgraph around a query node for downstream analysis \citep{andersen2006local, spielman2013local}.

Why is this local approach effective? An explanation which would tie several empirical observations together is that the graph is generated over a low-dimensional latent space in which connections mostly happen locally. A graph following this model can be sparse and transitive, two common features of real-world graphs \citep{chanpuriya2021power,sansford2023implications}. Those properties are incompatible with the low-rank assumptions underpinning spectral embedding \citep{seshadhri2020impossibility}. A kernel-manifold relationship deriving from Mercer's theorem allows us to view the latent space as mapping to a low-dimensional manifold embedded in infinite-dimensional space \citep{rubin2020manifold,whiteley2021matrix,whiteley2022statistical} (in the present paper use the word ``manifold'' for a set of low intrinsic dimension embedded in infinite-dimensional Euclidean space -- precise details of which are given later). This provides an intuitive explanation for the improvement offered by the local approach: While spectral embedding tries and fails to capture the entire manifold in a low-dimensional subspace, the local approach only targets a patch of this manifold which, under suitable curvature assumptions, can presumably be much better approximated within a low-dimensional subspace.

In this paper, we make this explanation precise and introduce a family of local spectral embeddings which includes subgraph embedding as an example but usually offers improvements. In particular, the subgraph method is somewhat degenerate in making zero use of the information in the connections between the subgraph and the remaining population. The discontinuous in/out aspect of selecting a subgraph is also highly unwieldy for various graph analysis tasks. Our embeddings are also available in an inductive mode which only requires neighbourhood information for a seed set.

\paragraph{Outline and contributions.}
This paper is organised as follows. Section~\ref{sec:related_work} reviews related work on spectral embedding and localisation approaches for representation learning. Section~\ref{sec:background} recalls the latent position model interpretation of adjacency spectral embedding (ASE) and the decomposition of error into statistical and truncation components that underpins our analysis. Section~\ref{sec:LASE} introduces Local Adjacency Spectral Embedding (LASE), a weighted generalisation of ASE that emphasises a region of interest through node-specific weights. We provide a population-level characterisation of its target and show that LASE can be interpreted as an optimal low-rank feature map for a localised adjacency operator. Building on the analogous results for ASE from \cite{tang2013universally}, our main theoretical results provide non-asymptotic bounds that decompose the embedding error of LASE into a finite-sample statistical term (variance) and a truncation term (bias) arising from finite-dimensional projection onto a local region. We further show that sufficient localisation induces rapid spectral decay and an emergent eigengap, providing theoretical justification for low-dimensional embeddings. Section~\ref{sec:experiments} presents experiments on synthetic and real networks validating the theory and illustrating the benefits of localisation. We introduce UMAP-LASE, a local-to-global procedure that assembles overlapping local embeddings into high-fidelity global visualisations. Section~\ref{sec:discussion} concludes with a discussion.


\paragraph{Notation.}  For a matrix $\mathbf{M}$, $\|\mathbf{M}\|$ denotes the spectral norm; for a vector $\mathbf{v}\in \ell_2$ or $\mathbb{R}^d$, $\|\mathbf{v}\| = (\sum_i v_i^2)^{1/2}$ denotes the $\ell_2$ or Euclidean norm, where
$\ell_2 \coloneqq \{ (v_i)_{i\ge1} : \sum_{i\ge1} v_i^2 < \infty \}$.
For an integer $n\geq 1$ we write $[n]\coloneqq\{1,\ldots,n\}$.

\section{Related Work}
\label{sec:related_work}

Our work sits at the intersection of spectral graph theory, manifold learning, and local matrix approximation. In this section, we contextualise LASE within these fields, highlighting how it bridges the gap between global spectral consistency and local geometric fidelity.

\paragraph{Theoretical foundations of spectral embedding.}
A large body of work establishes consistency and distributional limits for ASE under latent position and random graph models. For the random dot product graph \citep{young2007random} and related models, \citet{sussman2012consistent, tang2013universally} provided foundational consistency results. Subsequent work established asymptotic normality for scaled spectral embeddings \citep{athreya2016limit} and sharper row-wise perturbation bounds \citep{cape2019two, abbe2020entrywise}. Parallel developments under blockmodel-type assumptions have analysed spectral clustering in sparse regimes and quantified misclassification rates \citep{rohe2011spectral, lei2015consistency}. Across these settings, guarantees typically require spectral separation and fast eigenvalue decay. However, as noted by \citet{seshadhri2020impossibility}, real-world graphs often exhibit sparsity and high triangle density, features that are incompatible with these global low-rank assumptions. Addressing this, \citet{sansford2023implications} demonstrate that such properties can be recovered from a low-dimensional manifold embedded in infinite-dimensional space.

\paragraph{Importance weighting.}
The use of node-specific weights in LASE (Algorithm \ref{alg:lase}) evokes importance weighting techniques in machine learning \citep{kimura2024short}. In graph analysis, \citet{bonald2018weighted} introduced a weighted Laplacian spectral embedding, though without the theoretical analysis of statistical error or localisation perspective provided here. 

\paragraph{Local PCA and matrix approximation.}
Conceptually, LASE is related to local PCA, which inspects data in small sub-regions to determine intrinsic dimensionality \citep{fukunaga1971algorithm}. This approach acknowledges that global PCA is limited by its linear assumption when data resides on a non-linear manifold \citep{fukunaga1971algorithm, rubin2020manifold}. In matrix approximation, several works have argued that global low-rank assumptions are often insufficient, proposing instead that different regions be approximated by distinct low-rank components \citep{lee2016llorma, arias2017spectral}.

\paragraph{Nonlinear manifold learning.}
LASE shares goals with nonlinear dimensionality reduction methods like Locally Linear Embedding (LLE) \citep{roweis2000nonlinear} and ISOMAP \citep{tenenbaum2000global}, which preserve neighbourhood relationships. Other techniques, such as UMAP \citep{mcinnes2018umap} and Laplacian Eigenmaps \citep{belkin2003laplacian}, emphasise neighbourhood preservation via graph-based smoothing. Our method differs by applying localisation directly via a weighted spectral decomposition, retaining the interpretability of ASE under the Latent Position Model. Furthermore, strategies for combining overlapping local embeddings into global visualisations, such as tangent space alignment \citep{zhang2004principal} or the \texttt{local2global} approach \citep{jeub2023local2global}, provide the context for our UMAP-LASE procedure described in Section \ref{sec:experiments}.

\section{Theoretical properties of Adjacency Spectral Embedding}
\label{sec:background}

In this section we outline definitions and facts about ASE under the Latent Position Model; this material is not new, but we need to present it here to help explain LASE in Section \ref{sec:LASE}.

The procedure for positive semidefinite ASE of an undirected graph with $n$ vertices
and adjacency matrix $\mathbf{A}\in\{0,1\}^{n\times n}$ is given in Algorithm \ref{alg:ase}.
\begin{algorithm}[H]
\caption{Adjacency Spectral Embedding (ASE)}
\label{alg:ase}
\textbf{Input:} Symmetric adjacency matrix $\mathbf{A} \in \{0,1\}^{n \times n}$ and fixed embedding dimension $r\geq 1$.
\begin{algorithmic}[1] 
    \State Compute the $r$ largest eigenvalues of $\mathbf{A}$, 
    $\lambda_1 \ge \lambda_2 \ge \cdots \ge \lambda_r$, 
    and corresponding orthonormal eigenvectors 
    $\mathbf{u}_1, \ldots, \mathbf{u}_r \in \mathbb{R}^n$.
    \State Define 
    $\mathbf{U} = [\mathbf{u}_1 \mid \cdots \mid \mathbf{u}_{r}] 
    \in \mathbb{R}^{n \times r}$ 
    and 
    $\mathbf{\Lambda} = \mathrm{diag}(\lambda_1, \ldots, \lambda_{r}) 
    \in \mathbb{R}^{r \times r}$.
    \State Compute
    \[
    \hat{\mathbf{X}} = \mathbf{U} \mathbf{\Lambda}^{1/2} 
    \in \mathbb{R}^{n \times r},
    \]
    and denote the $i$-th row by $\hat{X}_i^\top$.
\end{algorithmic}
\textbf{Output:} Embedding vectors 
$\hat{X}_1, \ldots, \hat{X}_n \in \mathbb{R}^{r}$.
\end{algorithm}
In Algorithm \ref{alg:ase}, the embedding dimension $r$ is a fixed parameter, but should be such that $\lambda_r \geq 0$. Under the generative model for $\mathbf{A}$ assumed in this work (Definition \ref{def:latent_position_model} below), it can be shown that for any $r\geq 1$, $\lambda_r \geq 0$ holds with high probability for sufficiently large $n$. In practice, the choice of $r$ is often guided by inspecting the eigenvalues of $\mathbf{A}$, for example via the position of an ``elbow'' in a scree plot. 

The ASE embedding $\hat{\mathbf{X}}$  as in Algorithm \ref{alg:ase}, or any orthogonal transformation of that embedding, solves the Frobenius norm minimisation problem:
\begin{equation}\label{eq:EYM}
\hat{\mathbf{X}} = \argmin_{\mathrm{rank}(\mathbf{X})\leq r}\|\mathbf{A} - \mathbf{X}\mathbf{X}^\top\|_F,   
\end{equation}
by the Eckart-Young-Mirksy theorem \citep{eckart1936approximation, mirsky1960symmetric}.
This characterises ASE as an optimal low-rank approximation to the observed adjacency matrix $\mathbf{A}$. Under a probabilistic model, however, $\mathbf{A}$ may be viewed as a noisy realisation of a matrix of connection probabilities. To make this connection precise and to analyse the statistical properties of $\hat{\mathbf{X}}$, we assume that $\mathbf{A}$ is generated according to the Latent Position Model of \citet{hoff2002latent}, in the form considered by \citet{tang2013universally}.
\begin{definition}[Latent position model] 
\label{def:latent_position_model}
    Let $\mathcal{Z}$ be a compact metric space; let $\mu$ be a Borel probability measure supported on $\mathcal{Z}$; let $f: \mathcal{Z} \times \mathcal{Z} \rightarrow [0,1]$ be a continuous, symmetric, positive semi-definite kernel; and let $Z_1,\ldots,Z_n$ be i.i.d. draws from $\mu$. Given $Z_1,\ldots,Z_n$, the elements $A_{ij}$ of $\mathbf{A}$ are conditionally independent and distributed:
\begin{equation}\label{eq:bernoulli}
    A_{ij} \mid Z_1, \ldots, Z_n \overset{ind}{\sim} \mathrm{Bernoulli}\left\{f(Z_i, Z_j)\right\}, \quad \text{for $i < j$}.
    \end{equation} 
\end{definition}

In the setting of Definition \ref{def:latent_position_model},
Mercer's theorem \citep[Thm. 4.49]{steinwart2008support} asserts the existence of a collection of functions $(u_{k})_{k\geq1}$ which are
orthonormal in $L_{2}(\mathcal{Z},\mu)$, and nonnegative
real numbers $\lambda_{1}\geq\lambda_{2}\geq\cdots$, such that 
\begin{equation}\label{eq:mercer_expansion}
f(x,y)=\sum_{k=1}^{\infty}\lambda_{k}u_{k}(x)u_{k}(y),
\end{equation}
where the convergence is absolute and uniform in $x,y$. The $(u_k,\lambda_k)_{k\geq 1}$ are eigen-function/value pairs of the integral operator $\mathcal{A}: L_2(\mathcal{Z}, \mu) \to L_2(\mathcal{Z}, \mu)$, 
$$\mathcal{A}g(x) \coloneqq \int_{\mathcal{Z}} f(x,y) g(y) \mu (\mathrm{d}y),$$
which is trace-class, since with $Z\sim\mu$, $\sum_{k\geq 1}\lambda_k = \sum_{k\geq 1}\lambda_k \mathbb{E}[|u_k(Z)|^2] = \mathbb{E}[f(Z,Z)]\leq 1 $, using $\sup_{x,y}f(x,y)\leq 1$.

Defining the \emph{Mercer feature map}  $\phi:\mathcal{Z}\to\ell_{2}$, $\phi(x)\coloneqq[\lambda_{1}^{1/2}u_{1}(x),\lambda_{2}^{1/2}u_{2}(x),\cdots]$,
we can regard evaluation of $f$ at some $x,y$ as evaluation of the $\ell_2$
inner-product: 
\begin{equation}\label{eq:mercer_inner_prod_form}
f(x,y)\equiv\left\langle \phi(x),\phi(y)\right\rangle.
\end{equation}
Under mild assumptions, the set $\mathcal{M}\coloneqq\{\phi(x);x\in\mathcal{Z}\}$ equipped with the $\ell_2$ distance  conveys the topological structure of $\mathcal{Z}$. For example, if for all $x, y\in\mathcal{Z}$ with $x \neq y$, there exists $a\in\mathcal{Z}$ such that $f(x,a)\neq f(y,a)$, then  $\phi$ is a homeomorphism, so that $\mathcal{Z}$ and $\mathcal{M}$ are
topologically equivalent.  A proof of this fact is given in Appendix \ref{app:ase}; see \citet{whiteley2022statistical} for further discussion of the manifold structure of $\mathcal{M}$ under various assumptions on $f$.

The ASE embedding can then be regarded as approximating the point cloud $\{\phi(Z_{1}),\ldots,\phi(Z_{n})\}$, with approximation error which can be separated into two components: \emph{truncation} error
which arises since the ASE embedding is of dimension $r<\infty$ but, in general $\phi(Z_{i})$ has more than $r$ (possibly infinitely many) non-zero entries; 
and \emph{statistical} error which arises since under the latent position
model the adjacency matrix $\mathbf{A}$ involves the randomness of
$Z_{1},\ldots,Z_{n}$ and the Bernoulli-distributed variables in (\ref{eq:bernoulli}).

Let $\phi^{(r)}(x)$ denote the truncation of the vector $\phi(x)\in\ell_2$ to its first
$r$ entries. Then by Schmidt's low-rank approximation theorem \citep{schmidt1907theorie} which we quote in full in Appendix \ref{app:ase} (infinite-dimensional analog of the Eckart-Young-Mirsky Theorem), the map $\phi^{(r)}$, or any orthogonal
transformation thereof, achieves best rank-$r$ approximation to $\phi$
in that, for any $r\geq1$,
\begin{equation}
\phi^{(r)}=\argmin_{\tilde{\phi}:\mathcal{Z}\to\mathbb{R}^r}\mathbb{E}\left[\left|\left\langle \phi(Z),\phi(Z^\prime)\right\rangle -\left\langle \tilde{\phi}(Z),\tilde{\phi}(Z^\prime)\right\rangle \right|^{2}\right],\quad\text{where}\quad Z,Z^\prime\stackrel{\mathrm{iid}}{\sim}\mu.\label{eq:phi^r_opt}
\end{equation}
This makes precise the sense in which the geometry of $\{\phi^{(r)}(x);x\in\mathcal{Z}\}$
resembles that of  $\mathcal{M}$.
Moreover, if we pad the vector $\phi^{(r)}(x)$ with zeros so it is
the same length as $\phi(x)$, it follows from \eqref{eq:mercer_expansion},  \eqref{eq:mercer_inner_prod_form} and $\int_{\mathcal{Z}} |u_k(x)|^2 \mu(\mathrm{d}x)=1$ that the truncation error associated with
$\phi^{(r)}$ can be written: 
\begin{equation}\mathbb{E}\left[\left\Vert \phi(Z_i)-\phi^{(r)}(Z_i)\right\Vert ^{2}\right]=\sum_{k>r}\lambda_{k},\label{eq:ase_trunc_err}
\end{equation}
for any $i=1,\ldots,n$.

In their analysis of the statistical error associated with ASE, \cite{tang2013universally} showed that there exists an orthogonal matrix $\mathbf{Q} \in \R^{r \times r}$ such that  $\hat{\mathbf{X}}\mathbf{Q}$ is probabilistically close to the point cloud $\{\phi^{(r)}(Z_{1}),\ldots,\phi^{(r)}(Z_{n})\}$.  In particular, by \citet[eq. 3.9]{tang2013universally},  the statistical error can be bounded:
\begin{equation}
\label{eq:ase_stat_error}  \mathbb{E}\left[\left\Vert \hat{X}_i^\top\mathbf{Q}-\phi^{(r)}(Z_i)\right\Vert^2 \right]^{1/2} \leq \frac{27 \sqrt{6r} }{(\lambda_r - \lambda_{r+1})^{2}} \sqrt{\frac{\log n}{n}}.
\end{equation}
Therefore, in order for the combined truncation error \eqref{eq:ase_trunc_err} and statistical error \eqref{eq:ase_stat_error} to be small, one would like to choose $r$ such that  $\sum_{k>r}\lambda_k $ is small, and the eigengap $\lambda_r - \lambda_{r+1}$ is large. 
In practice, in search of a dimension that satisfies these conditions,
one usually inspects the eigenvalues of $\mathbf{A}$ as an approximation; as $n \to \infty$ the eigenvalues of $\mathbf{A}/n$ converge to the eigenvalues of $\mathcal{A}$, in some sense.
However, it is not uncommon for the eigenvalues of $\mathbf{A}$ to decay slowly and exhibit no discernable eigengap; see, for example, the left-hand plot of Figure \ref{fig:road_network_eig_decay}.

\begin{figure}
    \centering
    \includegraphics[width=0.8\linewidth]{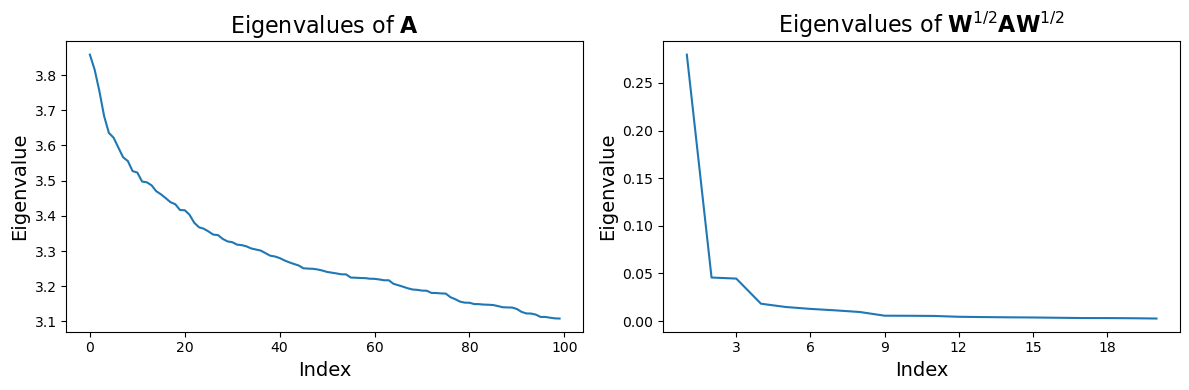}
    \caption{Left: the top 100 eigenvalues of the adjacency matrix, $\mathbf{A}$, of the Bristol Road Network (see Section \ref{sec:road_network} for details of the dataset), where $n = 3857$. Right: the top 20 eigenvalues of $\mathbf{W}^{1/2}\mathbf{A}\mathbf{W}^{1/2}$, where $\mathbf{W}$ is a diagonal matrix of weights $w_1, \ldots, w_n > 0$ selected to `localise' LASE (Algorithm \ref{alg:lase}) around a node of interest, $z$, based on the graph distance from $z$ (see Section \ref{weights_computation} - graph-distance-based weighting).}
    \label{fig:road_network_eig_decay}
\end{figure}

We make this (fairly common) positive semidefinite assumption on $f$ because it greatly simplifies the theory. If we relax this we lose access to Mercer's theorem; analogous decompositions \citep{lei2021network,rubin2020manifold} typically only hold in a weaker sense. An analogous geometric interpretation could possibly be phrased within a pseudo-Euclidean framework, as in the finite rank example of the generalised random dot product graph \citep{rubin2022statistical}, but in our view this would represent too much conceptual development for one paper. One should not use LASE, as it appears in Algorithm~\ref{alg:lase}, i.e., using only the positive eigenvalues, if $\mathbf{W}^{1/2} \mathbf{A} \mathbf{W}^{1/2}$ shows significant negative eigenvalues.

\section{Local Adjacency Spectral Embedding}\label{sec:LASE}

\subsection{The LASE algorithm}

\begin{algorithm}[H]
\caption{Local Adjacency Spectral Embedding (LASE)}
\label{alg:lase}
\textbf{Input:} Symmetric adjacency matrix $\mathbf{A} \in \{0,1\}^{n \times n}$, 
node weights $w_1, \ldots, w_n > 0$, and fixed embedding dimension $r \geq 1$.
\begin{algorithmic}[1]
\State Construct the diagonal weight matrix 
$\mathbf{W} = \mathrm{diag}(w_1, \ldots, w_n) \in \mathbb{R}^{n \times n}$.
\State Form the matrix $\mathbf{W}^{1/2}\mathbf{A}\mathbf{W}^{1/2}$ and compute the $r$ largest eigenvalues, denoted $\lambda_1 \ge \lambda_2 \ge \cdots \ge \lambda_r$,  and their corresponding orthonormal eigenvectors 
$\mathbf{u}_1, \ldots, \mathbf{u}_r \in \mathbb{R}^n$.
\State Define 
$\mathbf{U}_w = [\mathbf{u}_1 \mid \cdots \mid \mathbf{u}_{r}] 
\in \mathbb{R}^{n \times r}$ 
and 
$\mathbf{\Lambda}_w = \mathrm{diag}(\lambda_1, \ldots, \lambda_{r}) 
\in \mathbb{R}^{r \times r}$.
\State Compute the embedding
\[
\hat{\mathbf{X}} 
= \mathbf{W}^{-1/2}\mathbf{U}_w \mathbf{\Lambda}_w^{1/2} 
\in \mathbb{R}^{n \times r},
\]
and denote the $i$-th row of $\hat{\mathbf{X}}$ by 
$\hat{X}_i^\top$.
\end{algorithmic}
\textbf{Output:} Embedding vectors 
$\hat{X}_1, \ldots, \hat{X}_n \in \mathbb{R}^{r}$.
\end{algorithm}
\vspace{0.5em}

Algorithm \ref{alg:lase} defines Local Adjacency Spectral Embedding (LASE), a weighted variant of ASE in which node-specific weights adjust the contribution of each node to the spectral decomposition. The role of the weights is to emphasise a region of interest by giving greater influence to nodes that are more relevant to the local structure. The algorithm itself treats the weights as exogenous inputs; assumptions on the origin of these weights are introduced only for theoretical analysis, and practical strategies for their selection are discussed in Section~\ref{weights_computation}. Analogous to Algorithm \ref{alg:ase}, the embedding dimension $r$ is a fixed parameter and under the latent position model for any $r \geq1$, it can be shown that $\lambda_r \geq 0$ holds with high probability for sufficiently large $n$. 

\paragraph{Weights in theory.}
For our theoretical analysis, we define a reweighted probability measure $\mu_w$ on $\mathcal{Z}$ by
\begin{equation}\label{eq:mu_w_defn}
\mu_w(\mathrm{d}x) = w(x)\,\mu(\mathrm{d}x),
\end{equation}
where $w:\mathcal{Z}\to(0,\infty)$ is continuous and normalised such that $\int_{\mathcal{Z}} w(x)\,\mu(\mathrm{d}x)=1$ (hence $\mu_w$ is a probability measure). We assume that the node weights in Algorithm \ref{alg:lase} are evaluations of $w$ at the latent positions:
\begin{assumption}\label{assump:weights}
For $i=1,\ldots,n$, $w_i = w(Z_i)$.
\end{assumption}

We have in mind $w$ being such that $\mu_w$ assigns most of its mass to a neighbourhood of a point $z^\star \in \mathcal{Z}$. Although this localisation of $\mu_w$ is not required for the statistical bounds in Theorem \ref{thm:concentration} to hold, its consequences for inducing rapid spectral decay and controlling truncation error are analysed in Theorem \ref{thm:lase_trunc_err} of Section~\ref{sec:LASE_trunc_error}.

\paragraph{Weights in practice.}
Algorithm~\ref{alg:lase} only requires non-negative weights $w_1,\ldots,w_n$ and is invariant to their global scale: replacing $w_i$ by $\alpha w_i$ for any $\alpha>0$ multiplies $\mathbf{W}^{1/2}\mathbf{A}\mathbf{W}^{1/2}$ by $\alpha$ and hence scales $\mathbf{\Lambda}_w$ by $\alpha$, while the embedding $\hat{\mathbf{X}}=\mathbf{W}^{-1/2}\mathbf{U}_w\mathbf{\Lambda}_w^{1/2}$ is unchanged. Consequently, no normalisation of the weights is required to compute $\hat{\mathbf{X}}$. Normalisation is relevant, however, when the eigenvalues are examined, e.g. via scree plots or compared across weight choices, in which case we impose a consistent convention such as $\sum_{i=1}^n w_i = n$.
As stated above, for theoretical results we assume $w(x)>0$ on $\operatorname{supp}(\mu) = \mathcal{Z}$ to avoid degeneracies in $\mathbf{W}^{-1/2}$. In practice, zero weights may be used to exclude nodes, which reduces LASE to act on an induced subgraph. Concrete strategies for selecting weights in practice are deferred to Section~\ref{weights_computation}.

\subsection{Inductive LASE}
\label{sec:inductive_lase}

A practical limitation of many graph embedding methods is the cost of re-embedding as new data arrive. LASE, however, admits a natural out-of-sample (inductive) extension, that avoids recomputing the eigendecomposition.
Using the eigendecomposition in Algorithm \ref{alg:lase}, it can be shown that the LASE embedding of the $i$-th node satisfies:
$$\hat{X}_i^\top = (\mathbf{A}\mathbf{W}^{1/2})_i \mathbf{U}_w \mathbf{\Lambda}_w^{-1/2}.$$
This formulation allows us to embed an out-of-sample node $v_{n+1}$ using only its connections to the existing graph, represented by $\mathbf{a} \in \{0,1\}^n$. This inductive embedding can be computed as detailed in Algorithm \ref{alg:inductive_lase}.

\begin{algorithm}[H]
\caption{Inductive LASE}\label{alg:inductive_lase}
\textbf{Input:} Connections of the new node to the existing graph $\mathbf{a} \in \{0,1\}^n$, existing node weights $w_1, \ldots, w_n > 0$, and pre-computed $\mathbf{U}_w, \mathbf{\Lambda}_w$ from Algorithm \ref{alg:lase}.
\begin{algorithmic}[1]
\State Form the weighted connection vector $\tilde{\mathbf{a}} \in \mathbb{R}^n$ with entries $\tilde{a}_j = a_j w_j^{1/2}$.
\State Compute the inductive embedding:
$$\hat{X}_{n+1}^{\top} = \tilde{\mathbf{a}}^\top \mathbf{U}_w \mathbf{\Lambda}_w^{-1/2}.$$
\end{algorithmic}
\textbf{Output:} Embedding vector $\hat{X}_{n+1} \in \mathbb{R}^{r}$.
\end{algorithm}

This calculation is computationally efficient, requiring only matrix-vector multiplication rather than a full spectral decomposition. 
Notably, this procedure does not require a weight for the new node. Instead, the node is positioned relative to the local geometry defined by the existing weights purely through the new node's connections.
Connections to nodes with high weight (those within the region of interest) exert a stronger influence on $\hat{X}_{n+1}$ than connections to nodes with low weight, preserving the locality of the embedding for out-of-sample data.

\subsection{Optimality and interpretation of the LASE feature map }\label{sec:LASE_feature_map}

In this section we define and discuss a feature map $\phi_w$, motivating our theoretical results in Sections \ref{sec:LASE_stat_error} and \ref{sec:LASE_trunc_error} which show that the LASE embedding approximates $\{\phi_w(Z_1),\ldots,\phi_w(Z_n)\}$. We stress that throughout Sections \ref{sec:LASE_feature_map}-\ref{sec:LASE_stat_error} we are continuing to assume that $\mathbf{A}$ is generated according to the Latent Position Model as per Definition \ref{def:latent_position_model}.

We take the definition of $\phi_w$ to be Mercer feature map associated with the kernel $f$ and measure $\mu_w$ \eqref{eq:mu_w_defn} (instead of the measure $\mu$ as in Section \ref{sec:background}). That is $\phi_w:\mathcal{Z}\to \ell_2$ is given by $\phi_w(x)\coloneqq[\tilde{\lambda}_{1}^{1/2}\tilde{u}_{1}(x),\tilde{\lambda}_{2}^{1/2}\tilde{u}_{2}(x),\cdots]$ where $(\tilde{\lambda}_k,\tilde{u}_k)_{k\geq 1}$ are eigen-value/function pairs associated with the integral operator $\mathcal{A}_w:L_2(\mathcal{Z},\mu_w)\to L_2(\mathcal{Z},\mu_w)$:
$$
\mathcal{A}_w g(x)\coloneqq \int_{\mathcal{Z}} f(x,y)g(y)\mu_w(\mathrm{d}y).
$$
Thus, recalling \eqref{eq:mercer_inner_prod_form} the following two equalities hold:
\begin{equation}\label{eq:two_feaure_maps}
f(x,y) = \langle\phi(x),\phi(y)\rangle =\langle\phi_w(x),\phi_w(y)\rangle, 
\end{equation}
and combining the second equality in \eqref{eq:two_feaure_maps} with an application of Schmidt's optimality theorem to $\phi_w$, we have, for any $r\geq 1$:
\begin{equation}
\phi_{w}^{(r)}=\argmin_{\tilde{\phi}:\mathcal{Z}\to\mathbb{R}^{r}}\mathbb{E}\left[\left|\left\langle \phi(Z),\phi(Z^\prime)\right\rangle -\left\langle \tilde{\phi}(Z),\tilde{\phi}(Z^\prime)\right\rangle \right|^{2}\right],\quad\text{where}\quad Z,Z^\prime\stackrel{\mathrm{iid}}{\sim}\mu_{w},\label{eq:phi_w^r_opt}
\end{equation}
where $\phi_{w}^{(r)}(\cdot)$ is the truncation of $\phi_w(\cdot)$ to its first $r$ elements. In the sense of \eqref{eq:phi_w^r_opt}, $\phi_w^{(r)}$ achieves the optimal rank-$r$ approximation to $\phi$ when approximation accuracy is quantified with respect to the measure $\mu_w$ rather than $\mu$ as in \eqref{eq:phi^r_opt}. In particular, we have in mind situations in which $\mu_w$ is concentrated around some point in $z^\star \in\mathcal{Z}$ (explored in Section \ref{sec:LASE_trunc_error}), in which case the expectation in \eqref{eq:phi_w^r_opt} quantifies approximation accuracy in a way which focuses on a neighbourhood of $z^\star$ and we can regard $\phi_w^{(r)}$ as the \emph{locally optimal} $r$-dimensional approximation to $\phi$.

\paragraph{Illustrative example.}
Figure \ref{fig:zooming_in} illustrates how $\phi_w^{(r)}$ changes as $\mu_w$ becomes increasingly localised. The left-most column shows $\phi^{(2)}$, i.e. $\phi_w^{(2)}$ when $\mu_w = \mu$. Increasing the concentration parameter $\tau$ concentrates $\mu_w$ about a point $z^\star$, 
and $\phi_w^{(2)}$ adapts so that inner products $\langle \phi(x), \phi(y) \rangle$ are more accurately approximated for latent positions $x,y$ that are more likely under $\mu_w$. Intuitively, $\phi_w^{(2)}$ "zooms in" on the region of the manifold $\mathcal{M}= \{\phi(z): z \in \mathcal{Z} \}$ where $\mu_w$ places most mass. In this example, the weights are $w_\tau(x)=\exp(-\tau|x-z^\star|)/N_\tau$ applied to $\mu = \text{Uniform}[0,10]$, where $N_\tau$ is the normalising constant that ensures $\mu_w$ is a probability measure, as in \eqref{eq:mu_w_defn}.

\begin{figure}[h]
    \centering
    \includegraphics[width=1\linewidth]{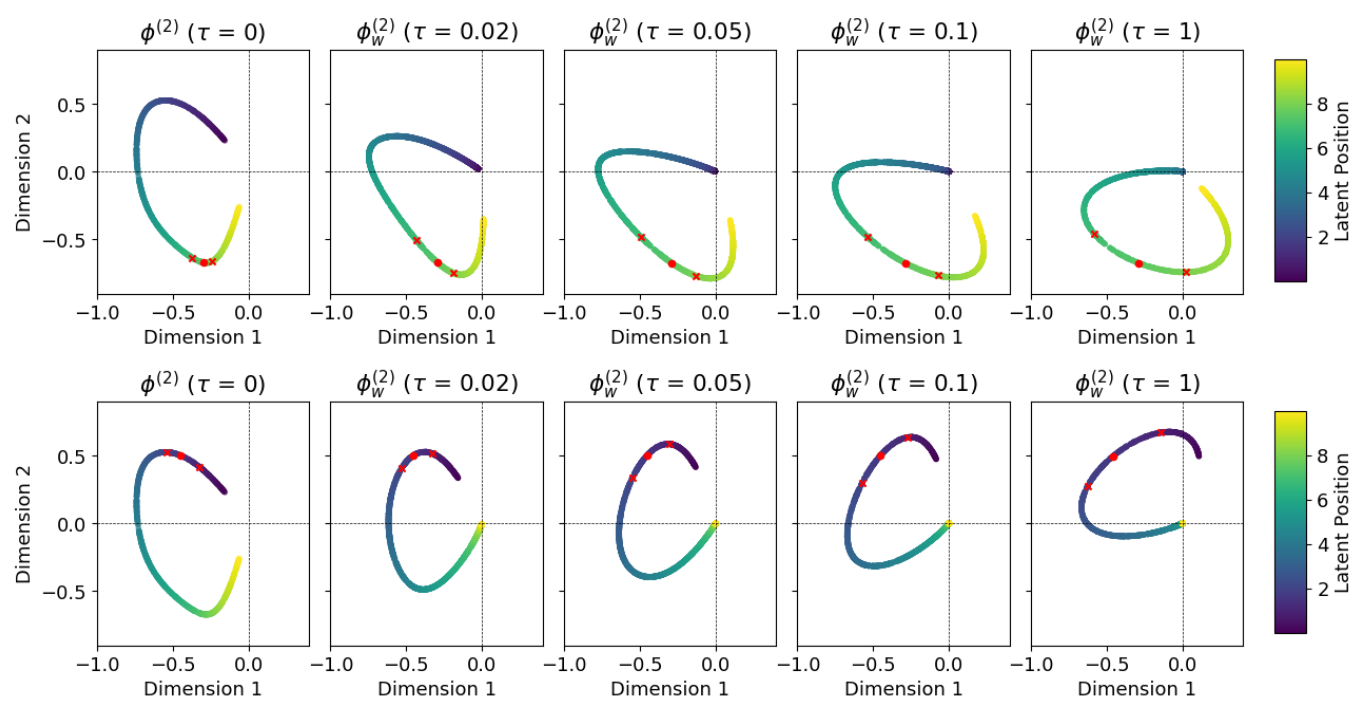}
    \caption{First two dimensions of the Mercer feature map $\phi$ (leftmost column) and the locally weighted feature map $\phi_w$ for increasing localisation (subsequent columns), controlled by the concentration parameter $\tau$.
    We use $\mu = \text{Uniform}[0,10]$ and $f(x,y)= \exp(-\tfrac{1}{2} (x-y)^2)$. Top row: localisation around $z \approx 1.5$. Bottom row: localisation around $z \approx 7.5$. In each row, $z^\star$ is indicated by a red circle and latent positions $z^\star \pm 0.5$ are indicated by red crosses. As $\tau$ increases, $\phi_w$ “zooms in” and better represents inner products for points likely under $\mu_w$. All plots align $\phi_w(z^\star)$ to $\phi(z^\star)$ using Procrustes analysis.}
    \label{fig:zooming_in}
\end{figure}

The optimality in \eqref{eq:phi_w^r_opt} is a generalisation of \eqref{eq:phi^r_opt}, in the sense the former reduces to the latter when  $w(x)=1$ for all $x$. Our next objective is to obtain a generalisation of the truncation identity \eqref{eq:ase_trunc_err} and provide more information about the relationship between $\phi$ and $\phi_w$. This is the subject of Theorem \ref{thm:projection_thm} below, for which we need the following definitions.  Consider the (infinite) matrix:
\begin{equation}\label{eq:K_w_defn}
\mathcal{K}_w := \mathbb{E} \left[ \phi(Z) \phi(Z)^\top \right], \quad\text{where}\quad Z \sim \mu_w.
\end{equation}
This defines a self-adjoint, positive operator that maps vectors $\bm{v} \in \ell_2$  to $\mathcal{K}_w \bm{v} \in \ell_2$. We show in Lemma \ref{lem:covariance_eigenvalues} in Appendix \ref{app:relating_A_K} that the spectrum of this operator is equal to that of $\mathcal{A}_w$, and we denote by $\bm{v}_k\in\ell_2$ the eigenvector of $\mathcal{K}_w$ associated with eigenvalue $\tilde{\lambda}_k$. Let $\mathrm{rank}(\mathcal{A}_w)$ denote the number of non-zero eigenvalues of $\mathcal{A}_w$.
\begin{thm}
\label{thm:projection_thm}
   \textit{ For any $r \leq \mathrm{rank}(\mathcal{A}_w)$, the following hold:}
   
   \noindent\textit{1) With $Z \sim \mu_w$,
   $$\min_{\boldsymbol{\Pi}_r} \mathbb{E} \left[ \Vert \phi(Z) - \boldsymbol{\Pi}_r \phi(Z) \Vert^2  \right] = \sum_{i = r+1}^\infty \Tilde{\lambda}_i,$$
   where the minimum is over orthogonal projections $\boldsymbol{\Pi}_r$ onto any $r$-dimensional linear 
    subspace of $\ell_2$.}

   \noindent\textit{2) With $\bm{v}_1, \cdots,\bm{v}_r \in \ell_2$ the top $r$ eigenvectors of $\mathcal{K}_w$ and $\mathbf{V}_r := [\bm{v}_1 | \cdots | \bm{v}_r]$, the orthogonal projection $\mathbf{V}_r \mathbf{V}_r^\top$ achieves the minimum in 1).}

   \noindent\textit{3) The equality $\mathbf{V}_r^\top \phi(x) = \phi_w^{(r)}(x)$  holds for all $x \in \mathcal{Z}$.}
\end{thm}
The proof is in Appendix \ref{app:relating_A_K}. Theorem \ref{thm:projection_thm} can be interpreted as follows. $\mathcal{K}_w$ is the (uncentered) population covariance matrix of the random vector $\phi(Z)$ with $Z\sim\mu_w$. Parts 1) and 2) of Theorem \ref{thm:projection_thm} indicate that when approximation error is quantified by integrating with respect to $\mu_w$, the optimal $r$-dimensional projection of $\phi(\cdot)$ is the projection onto the linear subspace spanned by the top $r$ eigenvectors $\{\bm{v}_1, \cdots , \bm{v}_r\}$ of this covariance matrix. Part 3) indicates that, expressed with respect to the basis $\{\bm{v}_1, \cdots , \bm{v}_r\}$, this optimal projection of $\phi(\cdot)$ is $\phi_w^{(r)}(\cdot)$.





From a finite-sample perspective, and drawing comparisons with the Eckart-Young-Mirsky interpretation of ASE in \eqref{eq:EYM}, we can view the output of Algorithm \ref{alg:lase} as the solution to the following weighted Frobenius norm minimisation problem:
\begin{lem}
\label{lem:local-EY}
With $\hat{\mathbf{X}}$, $\mathbf{A}$ and $\mathbf{W}$ as in Algorithm \ref{alg:lase}, we have
    $$\hat{\mathbf{X}} = \argmin_{\mathbf{X} \in \R^{n \times r}}\| \mathbf{W}^{1/2} (\mathbf{A} - \mathbf{XX}^\top) \mathbf{W}^{1/2} \|_{F}.$$
\end{lem}

\proof We have that
\begin{align*}
    \min_{\mathbf{X} \in \R^{n \times r}}\| \mathbf{W}^{1/2} (\mathbf{A} - \mathbf{XX}^\top) \mathbf{W}^{1/2} \|_{F}
    &= \min_{\mathbf{X} \in \R^{n \times r}}\| \mathbf{W}^{1/2}\mathbf{A} \mathbf{W}^{1/2} - (\mathbf{W}^{1/2} \mathbf{X}) (\mathbf{W}^{1/2}\mathbf{X})^\top \|_{F}.
\end{align*}
Let $\tilde{\mathbf{X}} := \mathbf{W}^{1/2} \mathbf{X}$, since $\mathbf{W}$ is diagonal with strictly positive diagonal entries, the linear map $\mathbf{X} \mapsto \tilde{\mathbf{X}}$ is a bijection on $\R^{n \times r}$ and preserves the column space dimension (so $\tilde{\mathbf{X}}$ ranges over all $n \times r$ matrices). Hence the original minimisation is equivalent to 
\begin{align*}
    \min_{\tilde{\mathbf{X}} \in \R^{n \times r}}\| \mathbf{W}^{1/2}\mathbf{A} \mathbf{W}^{1/2} -  \tilde{\mathbf{X}} \tilde{\mathbf{X}}^\top \|_{F},
\end{align*}
which, by the Eckart-Young-Mirsky Theorem, is minimised by $\tilde{\mathbf{X}} = \mathbf{U}_w \mathbf{\Lambda}_w^{1/2}$, where $ \mathbf{U}_w$ and  $\mathbf{\Lambda}_w$ are as defined in Algorithm \ref{alg:lase}. Thus, by substituting our definition of $\tilde{\mathbf{X}}$, we see that $\mathbf{X} = \mathbf{W}^{-1/2} \mathbf{U}_w \mathbf{\Lambda}_w^{1/2}$ achieves the minimum of interest.
\qed

This characterisation can be viewed as the finite-sample analogue of the infinite-dimensional Schmidt optimality problem in \eqref{eq:phi_w^r_opt}, foreshadowing the link between the output of the LASE algorithm and $\phi_w^{(r)}$, which we explore in detail in the following section.

\subsection{Bounding the statistical error of LASE}\label{sec:LASE_stat_error}

In this section, we bound the statistical error of LASE arising from finite-sample randomness in the observed adjacency matrix. Specifically, Theorem~\ref{thm:concentration} provides a finite-sample bound on the deviation between LASE and the locally optimal feature map $\phi_w^{(r)}$. 
As in standard ASE, this error is controlled by the spectral gap of the associated population operator, but here it is most naturally measured in a $\mathbf{W}^{1/2}$-weighted norm. This weighting reflects both the definition of the LASE algorithm, based on the matrix $\mathbf{W}^{1/2}\mathbf{A}\mathbf{W}^{1/2}$, and the fact that the optimality of $\phi_w^{(r)}$ in \eqref{eq:phi_w^r_opt} is defined with respect to the weighted measure $\mu_w$. Consequently, the error bound prioritises accuracy in regions where $\mu_w$ places most mass, while downweighting deviations in regions deemed less relevant. The resulting bound mirrors classical ASE rates, with an additional dependence on $w^*\coloneqq \sup_x w(x)$ that captures the statistical cost of localisation: highly concentrated weight functions increase variance by
effectively reducing the number of observations contributing to the embedding. This trade-off between localisation and statistical stability is explored further in Section~\ref{sec:LASE_trunc_error}.

\begin{thm}
\label{thm:concentration}
\textit{Suppose that $Z_1,\ldots,Z_n$ and $\mathbf{A}$ follow the model in Definition \ref{def:latent_position_model}
and $\tilde{\lambda}_r - \tilde{\lambda}_{r+1} > 0$ for some fixed $r\geq1$. Let $\hat{\mathbf{X}}$ be as in Algorithm \ref{alg:lase}.
With probability greater than $1 - 2\eta$, there exists an orthogonal matrix $\mathbf{Q} \in \R^{r \times r}$ such that
\begin{equation}
\label{conc_thm:eq_1}
    \left\Vert \mathbf{W}^{1/2} \left( \hat{\mathbf{X}} \mathbf{Q} - \mathbf{\Phi}_w^{(r)} \right) \right\Vert_F \leq \frac{27{w^*}^{3} }{(\tilde{\lambda}_r - \tilde{\lambda}_{r+1})^{2}}\sqrt{r \log(n / \eta)},
\end{equation}
where  $w^*\coloneqq \sup_x w(x)$ and $\mathbf{\Phi}_w^{(r)} \in \R^{n \times r}$ has $i$-th row equal to $\phi_w^{(r)}(Z_i)$. For each $i \in [n]$ and any $\epsilon > 0$,
\begin{equation}
\label{eq:3.1.2}
\mathbb{P} \left[ w(Z_i)^{1/2} \left\Vert \hat{X}_i^\top\mathbf{Q} - \phi_w^{(r)}(Z_i) \right\Vert \geq \epsilon\right] \leq \frac{27 {w^*}^3 }{\epsilon(\tilde{\lambda}_r - \tilde{\lambda}_{r+1})^{2}} \sqrt{\frac{6r\log n }{n}}.
\end{equation}}
\end{thm}


The proof of Theorem \ref{thm:concentration} can be found in Appendix \ref{app:concentration result}, and follows the main steps as the ASE consistency result of \cite[Thm. 3.1]{tang2013universally}. The second inequality provides a row-wise concentration bound, showing that for each node
$i$, the embedding error is controlled after scaling by its weight $w(Z_i)$, so that nodes with larger influence under $\mu_w$ enjoy tighter guarantees. When $w(x) = 1$ for all $x \in \mathcal{Z}$, the weighted norm reduces to the standard Frobenius norm and Theorem~\ref{thm:concentration} recovers the usual ASE concentration behaviour in \eqref{eq:ase_stat_error}.

\subsection{Bounding the truncation error of LASE}\label{sec:LASE_trunc_error}

In Section~\ref{sec:LASE_stat_error}, we showed that the statistical error bound for LASE can be viewed as a weighted analogue of the classical ASE concentration result. In this section, we draw an analogous comparison for the truncation error. For standard ASE, truncation of the Mercer feature map $\phi$ at rank $r$ incurs the mean-squared error in \eqref{eq:ase_trunc_err}. For LASE, with $Z \sim \mu$, truncation of the locally optimal feature map $\phi_w$ admits the weighted analogue
\begin{equation}\label{eq:weighted_trunc_err}
\mathbb{E}\!\left[ w(Z)\,\bigl\Vert \phi_w(Z)-\phi_w^{(r)}(Z)\bigr\Vert^2\right]
= \sum_{k>r}\tilde{\lambda}_k,
\end{equation}
which follows directly from the definition of $\phi_w$ and the orthonormality of the eigenfunctions $(\tilde{u}_k)_{k\geq1}$ in $L_2(\mathcal{Z},\mu_w)$. Consequently, controlling the truncation error of LASE reduces to understanding the decay of the eigenvalues $(\tilde{\lambda}_k)_{k>r}$.

The objective of this section is to characterise this eigenvalue decay when the weighted measure $\mu_w$ becomes increasingly concentrated around a point $z^\star\in\mathcal{Z}$, and to show that localisation induces a rapidly decaying spectrum, thereby ensuring that the truncation error of $\phi_w^{(r)}$ is small for an appropriate choice of the truncation dimension $r$.
To enable this analysis, we specialise the latent position model from Definition \ref{def:latent_position_model} to the case where $\mathcal{Z}$ is a subset of Euclidean space.


\begin{assumption}\label{ass:Z_in_R^d}
For some $d\geq 1$, the latent space $\mathcal{Z}$ is a compact subset of $\mathbb{R}^d$ with non-empty interior.
\end{assumption}
\begin{assumption}\label{ass:phi(x)_not_zero}
The point $z^\star$ is an interior point of $\mathcal{Z}$ and $\phi(z^\star)\neq 0$.
\end{assumption}
Requiring $z^\star$ to be an interior point of the compact set $\mathcal{Z}$ allows us to deploy elementary calculus arguments an open neighbourhood of $z^\star$. If $z^\star$ were allowed to be on the boundary of $\mathcal{Z}$, some other geometric assumptions about the boundary itself would be needed to make the same calculus arguments work.
The $\phi(z^\star)\neq 0$ part of \textbf{A\ref{ass:phi(x)_not_zero}} is a non-degeneracy assumption: if, conversely, $\phi(z^\star)=0$, then for any $x\in\mathcal{Z}$, $f(z^\star,x)=\langle \phi(z^\star),\phi(x)\rangle=0$, i.e., under the Latent Position Model there would be zero probability of any node connecting to a node with $Z_i=z^\star$.

\begin{assumption}\label{ass:differentiability}
For $i,j\in[d]$, the mixed partial derivatives:
$
\frac{\partial^2 f }{\partial x^{(i)} \partial y^{(j)} }
$ of $f(x,y)$
exist and are continuous, where $x=(x^{(1)},\ldots,x^{(d)})$ and $y=(y^{(1)},\ldots,y^{(d)})$. Moreover the matrix $\mathbf{H}_{x,y} \in \R^{d \times d}$ with $(i,j)$-th element
$$\mathbf{H}_{x,y}^{ij} := \frac{\partial^2 f}{\partial x^{(i)} \partial y^{(j)}}\bigg|_{(x,y)}, \hspace{10pt} \text{for } x,y \in \mathcal{Z},$$
satisfies:
$$
|\mathbf{H}_{x,x}^{ii}-\mathbf{H}_{x,y}^{ii}|\leq L_i \|x-y\|
$$
for some finite constants $L_i$, $i\in[d]$, and all $x,y\in\mathcal{Z}$.
\end{assumption}
Assumption \textbf{A\ref{ass:differentiability}} allows us to differentiate in the following sense: $\partial^2 f(x,y) /\partial x^{(i)} \partial y^{(j)}  = \partial^2 \langle \phi(x),\phi(y)\rangle / \partial x^{(i)} \partial y^{(j)} =  \langle \partial \phi(x)/ \partial x^{(i)}, \partial\phi(y)/\partial y^{(j)}\rangle$. For background on this differentiability of $\phi$ under \textbf{A\ref{ass:differentiability}} see, e.g. \citet[Lem. 4.34]{steinwart2008support}. 
We shall denote by  $\partial \phi(x)$ the matrix whose $(k,i)$-th element is the partial derivative of $\lambda_k^{1/2} u_k(x)$ with respect to $x^{(i)}$, where $x=
(x^{(1)}, \ldots, x^{(d)})$, so that $\mathbf{H}_{x,y} \equiv \partial \phi(x)^\top \partial \phi(y)$. This highlights that the matrix $\mathbf{H}_{x,y}$ is a symmetric function of $(x,y)$.

Now, let us define the \textit{local latent dimension} at $z^\star \in \mathcal{Z}$ to be
\begin{equation}\label{eq:d_loc_defn}
d_\text{loc}(z^\star) := \text{rank}\left(\mathbf{H}_{z^\star,z^\star}\right).
\end{equation}
The quantity $d_\text{loc}(z^\star)$ captures the intrinsic dimensionality of the kernel-induced geometry in a neighbourhood of $z^\star$. Note that, by definition, $d_\text{loc}(z^\star) \leq d$ and $d_\text{loc}(z^\star)$ depends on $f$ but not on $\mu$ or $\mu_w$. Intuitively, $d_\text{loc}(z^\star)<d$ corresponds to locally lower-dimensional structure: near $z^\star$, the set $\{\phi(z):z\in\mathcal{Z}\}$ is well approximated by a $d_\text{loc}(z^\star)$-dimensional tangent space. In this sense, we regard the set $\{\phi(z):z\in\mathcal{Z}\}$ as a manifold, albeit we do not assume that $\phi$ is invertible.

We next introduce assumptions that allow us to quantify the concentration of $\mu_w$  around $z^\star$. We consider a parameter $\epsilon>0$, such that $\mu_w$ becoming more concentrated corresponds to $\epsilon\to 0$: from hereon in Section \ref{sec:LASE_trunc_error} we assume that $\mu_w$ depends on $\epsilon$ (although this dependence is not shown in the notation $\mu_w$), such that assumptions \textbf{A\ref{ass:epsilon_sigma}} and \textbf{A\ref{ass:higher_moments}} are satisfied.
\begin{assumption}
\label{ass:epsilon_sigma}
    With $Z \sim \mu_w$,
    $\mathbb{E} \left[ (Z-z^\star)(Z-z^\star)^\top \right] = \epsilon \mathbf{\Sigma,}$
    for some fixed, rank-$d$ matrix $\mathbf{\Sigma}$.
\end{assumption}
\begin{assumption}
\label{ass:higher_moments}
    With $Z \sim \mu_w$, $\mathbb{E}[\Vert Z-z^\star\Vert^q] = o(\epsilon)$ for $q=5/2,3$. 
\end{assumption}

The main result of Section \ref{sec:LASE_trunc_error} is Theorem \ref{thm:lase_trunc_err} below, which quantifies the behaviour of  $(\tilde{\lambda}_i)_{i\geq 1}$ as $\epsilon\to 0$. As we have seen in Sections \ref{sec:LASE_feature_map} and  \ref{sec:LASE_stat_error}, $(\tilde{\lambda}_i)_{i\geq 1}$ are the eigenvalues of $\mathcal{A}_w$ and $\mathcal{K}_w$. As a step to analysing the eigenvalues of $\mathcal{K}_w$ we introduce:
$$\mathcal{D}_w := \mathbb{E} \left[ \left( \phi(Z) - \phi(z^\star) \right)\left( \phi(Z) - \phi(z^\star) \right)^\top \right], \qquad\text{where}\quad Z \sim \mu_w,$$ with eigenvalues denoted $\lambda_1(\mathcal{D}_w)\geq \lambda_2(\mathcal{D}_w)\geq \ldots $.
We have:
$$
\mathcal{K}_w = \mathcal{D}_w + \Bar{\phi}_w\Bar{\phi}_w^\top -  \bm{\gamma}_w\bm{\gamma}_w^\top,
$$
where, with $Z\sim \mu_w$, $\Bar{\phi}_w \coloneqq \mathbb{E}[\phi(Z)]$ and $\bm{\gamma}_w\coloneqq \Bar{\phi}_w- \phi(z^\star)$.

The following proposition quantifies the asymptotics of the eigenvalues of $\mathcal{D}_w$ with respect to  $\epsilon$.
\begin{prop}
\label{prop:eigenvalues_of_D}
Assume \textbf{A\ref{ass:Z_in_R^d}}-\textbf{A\ref{ass:higher_moments}}. Then, for $i=1,\ldots, d_\text{loc}(z^\star)$, $\lambda_i(\mathcal{D}_w) = \Theta(\epsilon)$; and  $
\sum_{i> d_\text{loc}(z^\star)} \lambda_i(\mathcal{D}_w) = o(\epsilon),
$ as $\epsilon \to 0$.
\end{prop}
Proposition \ref{prop:eigenvalues_of_D} implies that, as $\epsilon \to 0$ a $\Theta(\epsilon)$ eigen-gap $\lambda_{d_{\text{loc}(z^\star)}}(\mathcal{D}_w)-\lambda_{d_{\text{loc}(z^\star)}+1}(\mathcal{D}_w)$ emerges, whilst the eigenvalues $\lambda_i(\mathcal{D}_w)$ for $i>d_{\text{loc}(z^\star)}$ are all $o(\epsilon)$.

\begin{thm}
\label{thm:lase_trunc_err}
Assume \textbf{A\ref{ass:Z_in_R^d}}-\textbf{A\ref{ass:higher_moments}} and $\bm{\gamma}_w \neq \bm{0}$. As $\epsilon \to 0$, the eigenvalues of $\mathcal{K}_w$ decay at the following rates:
\begin{equation}\label{eq:lam_decay_cases}\tilde{\lambda}_i = \begin{cases}
    \Theta(1) &\text{for } i=1 \\
    \Theta(\epsilon) &\text{for } i=2, \ldots, d_\text{loc}(z^\star)-1 \hspace{15pt}(\text{ignore if } d_\text{loc}(z^\star) =1,2) \\
    O(\epsilon) &\text{for } i=d_\text{loc}(z^\star) \hspace{64pt} (\text{ignore if } d_\text{loc}(z^\star) =1) \\
    O(\epsilon) &\text{for } i= d_\text{loc}(z^\star)+1 \\
    o(\epsilon) &\text{for } i>d_\text{loc}(z^\star)+1,
\end{cases}
\end{equation}
and 
\begin{equation}
\label{eq:lase_trunc_err}
\sum_{i> d_\text{loc}(z^\star)+1} \tilde{\lambda}_i = o(\epsilon).    
\end{equation}

It follows that there exists a monotone decreasing sequence $\{\epsilon_k\}_{k\geq1}$, with $\epsilon_k \to 0$ as $k\to \infty$, such that for some $r\geq1$ in the set $\{d_\text{loc}(z^\star)-1, d_\text{loc}(z^\star), d_\text{loc}(z^\star)+1 \}$ the eigengap $\tilde{\lambda}_r - \tilde{\lambda}_{r+1} = \Omega(\epsilon_k)$ as $k \to \infty$. 
 \end{thm}

In particular, Theorem~\ref{thm:lase_trunc_err} shows that when $\mu_w$ is sufficiently localised, choosing $r$ in the set $\{d_\text{loc}(z^\star)-1, d_\text{loc}(z^\star), d_\text{loc}(z^\star)+1 \}$ results in a non-vanishing eigengap and a vanishing truncation error.
The proofs of Proposition \ref{prop:eigenvalues_of_D} and Theorem \ref{thm:lase_trunc_err} can be found in Appendix \ref{app:trunc_error}. In the proof of Theorem \ref{thm:lase_trunc_err}, the unlikely situation where $\bm{\gamma}_w = \bm{0}$ can be handled using a similar argument, satisfying the same decay rates as in \eqref{eq:lam_decay_cases}, however we omit the details. 

Together with the statistical error bound of Section~\ref{sec:LASE_stat_error}, these results show that LASE achieves accurate low-dimensional recovery by balancing localisation-induced eigenvalue decay with finite-sample stability.

\subsection{Gaussian toy example}

In this section, we present a simplified Gaussian example to illustrate the scaling behavior of the statistical error of LASE in an analytically tractable setting. Specifically, we characterise the rate at which the local re-weighted measure $\mu_w$ must concentrate, with respect to $n$, in order for the weighted Frobenius norm of the embedding error to vanish asymptotically as $n \to \infty$.

Let $d>2$, and consider a multivariate Gaussian distribution $\mu$ with mean vector $\mathbf{0} \in \R^d$ and fixed, positive definite covariance matrix $\mathbf{\Sigma} \in \R^{d\times d}$, truncated to the ball $\mathcal{B}_R = \{ x \in \R^d : \|x\| \leq R \},$ where the radius $R>0$ is chosen such that, for $Z \sim \mathcal{N}_d(\mathbf{0},\mathbf{\Sigma})$, $\mathbb{P}(Z \in \mathcal{B}_R)=1-\delta$ for some fixed $\delta>0$, assumed to be small. The specific values of $\mathbf{\Sigma}$ affect only constant factors and do not impact the scaling results below.

Let $\mu_w$ denote a second multivariate Gaussian distribution with fixed mean vector $\mathbf{z}^\star \in \text{int}(\mathcal{B}_R)$, and covariance $\epsilon\mathbf{I}_d$, also truncated to $\mathcal{B}_R$. We can consider $\mu_w$ to be related to $\mu$ via $w: \mathcal{B}_R \to \R_{\geq0}$ where
$$\mu_w(\rd x) = w(x) \mu(\rd x).$$

Now consider a latent position model in which latent positions  $Z_1, \ldots, Z_n$ are sampled i.i.d. from $\mu$, and $\mathbf{A}$ is generated via a symmetric kernel function $f: \mathcal{B}_R \times \mathcal{B}_R \to [0,1]$, satisfying assumption \textbf{A\ref{ass:differentiability}}.
Under this setup, the model satisfies assumptions \textbf{A\ref{ass:Z_in_R^d}}-\textbf{A\ref{ass:higher_moments}}. 
Assume in addition that $f$ is locally non-degenerate at $z^\star$, in the sense that $\mathbf H_{z^\star,z^\star}$ is positive definite; under this condition, $d_{\mathrm{loc}}(z^\star)=d$.

By Theorem \ref{thm:lase_trunc_err}, there exists a dimension $r \in \{d-1, d, d+1\}$, such that the eigengap of the operator $\mathcal{A}_w$, associated with $\mu_w$ and $f$, satisfies 
$$\lambda_{r}(\mathcal{A}_w) - \lambda_{r+1}(\mathcal{A}_w) = \Omega(\epsilon) \quad \text{as } \epsilon \to 0.$$
Moreover, the weight function satisfies
$$w^* = \Theta(\epsilon^{-d/2}) \quad \text{as } \epsilon \to 0.$$
Let $\mathbf{W} = \text{diag}(w(Z_1),\ldots,w(Z_n))$ and let $\hat{\mathbf{X}}$
be the output of Algorithm \ref{alg:lase}.
Applying Theorem \ref{thm:concentration} with the rates above, we get that 
with probability greater than $1-2\eta$, there exists an orthogonal matrix $\mathbf{Q} \in \R^{r \times r}$ such that
\begin{equation}
\label{eq:gaussian_eg}
\frac{1}{\sqrt{n}}\left\Vert \mathbf{W}^{1/2} \left( \hat{\mathbf{X}} \mathbf{Q} - \mathbf{\Phi}_w^{(r)} \right) \right\Vert_F = O\left(\epsilon^{-(4+3d)/2} \sqrt{\frac{r\log (n/\eta)}{n}} \right), 
\end{equation}
where $\mathbf{\Phi}_w^{(r)} \in \R^{n \times r}$ has $i$-th row equal to $\phi^{(r)}_w(Z_i)$. 
It follows that the weighted Frobenius error bound in \eqref{eq:gaussian_eg} vanishes as $n\to \infty$ provided that $\epsilon$ shrinks sufficiently slowly:
$$\epsilon = \omega\left( \left(\frac{r\log (n/\eta)}{n}\right)^{\tfrac{1}{4 + 3d}}\right).$$

\subsection{Weight selection}
\label{weights_computation}

Thus far, for the purposes of theoretical analysis, we have treated the weight function $w$ (or equivalently the measure $\mu_w$) as given. In practice, while one may have a general idea of which region or nodes are of primary interest, translating this into a concrete selection of node weights depends on available information and computational goals.

In this section, we describe several practical strategies for constructing node weights $\{w_i\}_{i=1}^n$ given observed data and domain knowledge. The choice of weights will depend on whether auxiliary node information is available, whether a specific region or node is of interest, and the computational constraints of the problem.

\vspace{1em}
\noindent\textbf{Normalisation convention.}
The weighting functions described in this section are specified only up to a
multiplicative constant. This is sufficient for computing LASE embeddings, since
Algorithm~\ref{alg:lase} is invariant to the global scale of the weights.
However, normalisation becomes essential when analysing or comparing the spectra
associated with different weight choices. In such cases, we impose a consistent normalisation convention $\sum_{i=1}^n w_i = n$.

\vspace{1em}
\noindent\textbf{Attribute-based weighting.}  
When auxiliary data or metadata is available for each node (e.g., spatial coordinates, features, covariates), this information can be used to define a weighting function. For example, in spatial networks, nodes may be associated with feature vectors $x_i \in \mathbb{R}^d$, and one can define
\[
w_i = \exp\left( -\tau \, \|x_i - z^\star\|^2 \right)
\]
for some centre of interest $z^\star$ and concentration parameter $\tau > 0$, inducing localisation around $z^\star$. 

\vspace{1em}
\noindent\textbf{Graph-distance-based weighting.}  
In the absence of auxiliary features, graph structure itself can be used to define locality. A natural choice is to define weights as a function of shortest-path distance from a reference node $i_0$:
\[
w_i = \left( \frac{1}{1 + \text{GraphDist}(i, i_0)} \right)^p,
\]
for some power $p > 0$. Increasing $p$ results in stronger localisation around $i_0$. This approach does not require node attributes and can be computed directly from $\mathbf{A}$.

\vspace{1em}
\noindent\textbf{Subgraph-based weighting (hard thresholding).}  
A special case of LASE arises when weights are binary:
\[
w_i =
\begin{cases}
1, & i \in \mathcal{V}_{\text{loc}}, \\
0, & \text{otherwise},
\end{cases}
\]
where $\mathcal{V}_{\text{loc}} \subseteq \{1,\ldots,n\}$ denotes a selected set of nodes.
In this setting, $\mathbf{W}^{1/2} \mathbf{A} \mathbf{W}^{1/2}$ contains zeros outside the submatrix induced by $\mathcal{V}_{\text{loc}}$, effectively reducing LASE to Adjacency Spectral Embedding on an induced subgraph. Although our theoretical framework assumes strictly positive weights, this limiting case is plausible in practice and computationally efficient. Subgraph ASE can therefore be viewed as a hard-thresholded variant of LASE.

\paragraph{Soft-thresholded subgraphs.} To improve over binary subgraph selection, one can apply a smooth kernel (e.g., Gaussian) to define gradually decaying weights surrounding a region of interest. Section~\ref{sec:reconstruction_error} shows that such softly localised LASE embeddings outperform hard-thresholded subgraph ASE in reconstruction accuracy.

\paragraph{Hybrid weighting.} One can combine structural and attribute-based proximity via
\[
w_i=
\exp\left(
-\alpha \cdot \text{GraphDist}(i,i_0)
-\beta \cdot \|x_i - x_{i_0}\|^2
\right),
\]
for parameters $\alpha, \beta > 0$. This allows one to emphasise either graph structure or feature similarity depending on data availability.

\paragraph{Supervised or task-driven weighting.} In some settings, it may be advantageous to optimise weights to improve downstream tasks (e.g., classification, regression). For example, one could define $w$ to minimise predictive loss via gradient-based optimization. 

\vspace{1em}
\noindent\textbf{Trade-offs and tuning.}  
Weight selection involves a trade-off between \textit{local specificity} and \textit{statistical stability}. Highly concentrated weights (e.g., large $\tau$ or $p$) isolate local structure but may lead to increased variance due to fewer effective samples (reflected in a larger $w^* = \max_i w_i$, the empirical analogue of $\sup_x w(x)$). Conversely, broader weightings reduce variance but dilute local detail. 

Practically, we recommend tuning localisation parameters using one or more of the following approaches:

\begin{enumerate}
    \item \textit{Downstream performance optimization.} If the embedding is used for a supervised task (e.g., regression, classification, or reconstruction), sweep across parameter values and identify regions of stable performance or “elbows” in reconstruction error curves (see Section~\ref{sec:experiments}).
    
    \item \textit{Use of metadata.} Incorporate known information about the data by colouring or labelling points (see Section~\ref{sec:umap-lase}).
    
    \item \textit{Domain knowledge.} In settings where the latent scale of meaningful variation is known (e.g., distance scales in road networks), parameter values can be chosen to reflect that scale explicitly.
\end{enumerate}

In practice, we find that LASE is robust across a broad range of localisation parameters. Smoothly decaying weights, such as those derived from Gaussian kernels, consistently outperform hard-thresholded subgraph weights in reconstruction and visualization tasks (see Sections~\ref{sec:reconstruction_error} and~\ref{sec:visualisation}). However, in very large graphs, the computational efficiency of subgraph ASE may still be advantageous, particularly when rapid approximation is prioritised over embedding fidelity. In such cases, practitioners may prefer hard-thresholded methods as a scalable first-pass approach, before applying more refined LASE embeddings to regions of interest.

\section{Experiments}
\label{sec:experiments}

\subsection{Simulated Data}
\subsubsection{Eigenvalue decay}

We illustrate our result concerning the rate of eigenvalue decay in Theorem \ref{thm:lase_trunc_err} with the following simulated experiment. For dimensions $d=1,2,3,4$, we consider a multivariate Gaussian distribution $\mu_d = \mathcal{N}(\mathbf{0}_d, \sigma^2 \mathbf{I}_d)$, where $\mathbf{I}_d \in \R_d$ is the $d$-dimensional identity matrix. Let us define the \textit{precision} to be $\tau := 1/\sigma^2$, which may be interpreted as the level of concentration about the point $z^\star = \mathbf{0}_d$. For varying values of $\sigma$, $n=100$ latent positions $Z_1, \ldots, Z_n$ are drawn from $\mu_d$ and the resulting probability matrix $\mathbf{P}$ is generated such that $\mathbf{P}_{ij} = \exp(-\Vert Z_i-Z_j \Vert^2)$. Figure \ref{fig:sim_eig_decay} shows, for each $d$, the top 6 eigenvalues of $\mathbf{P}$ as the precision, $\tau$, increases. As expected from Theorem \ref{thm:lase_trunc_err}, in each of the plots we see a clear gap emerging between the $(d+1)$-th and the $(d+2)$-th eigenvalue as $\mu_d$ concentrates about $z^\star$. 

\begin{figure}
    \centering
    \includegraphics[width=1\linewidth]{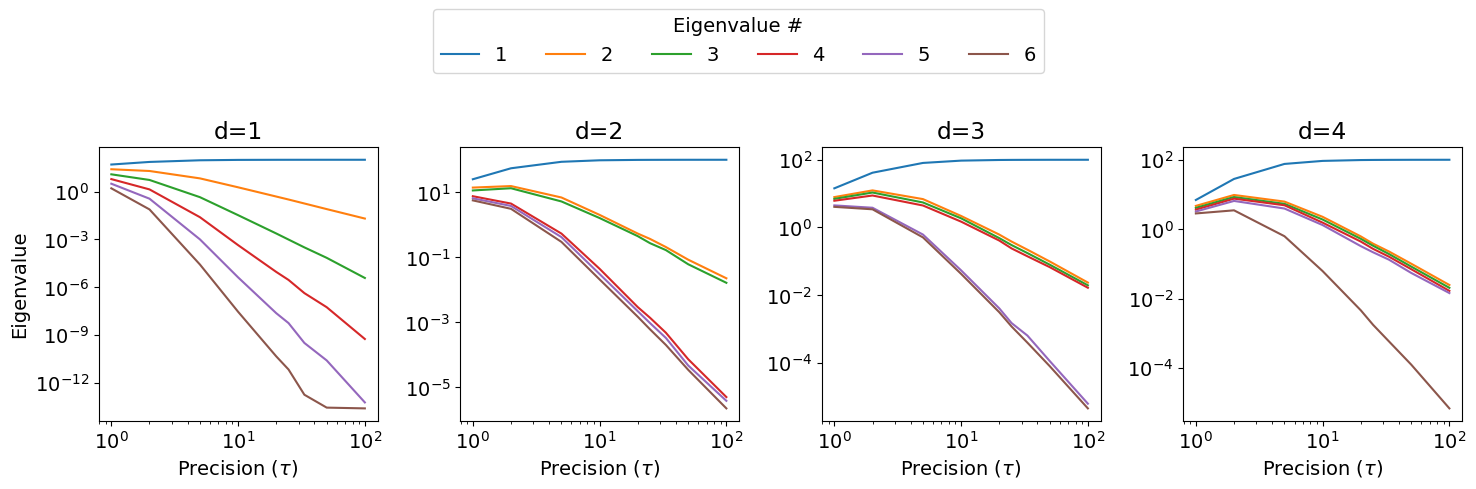}
    \caption{The top 6 eigenvalues of $\mathbf{P}$ are plotted as the measure $\mu_w$ concentrates about a point. In each plot, the dimension $d$ of the latent space $\mathcal{Z}$ is identified in the plot subtitle. Notice the distinct change in rate of decay between eigenvalues $d+1$ and $d+2$.}
    \label{fig:sim_eig_decay}
\end{figure}

\subsubsection{Reconstruction error}
\label{sec:reconstruction_error}
We also conduct a simulated experiment to examine the reconstruction error of LASE with parameterised weighting functions that tend towards a baseline approach of subgraph ASE on the region of interest. Our experiment is performed on graphs generated using 1000 latent positions $Z_i \sim \text{Uniform}[0,10]$, with edge probabilities given by $\mathbf{P}_{ij} = \exp(-\Vert Z_i-Z_j \Vert^2)$. 

\begin{figure}[h]
    \centering
    \includegraphics[width=1\linewidth]{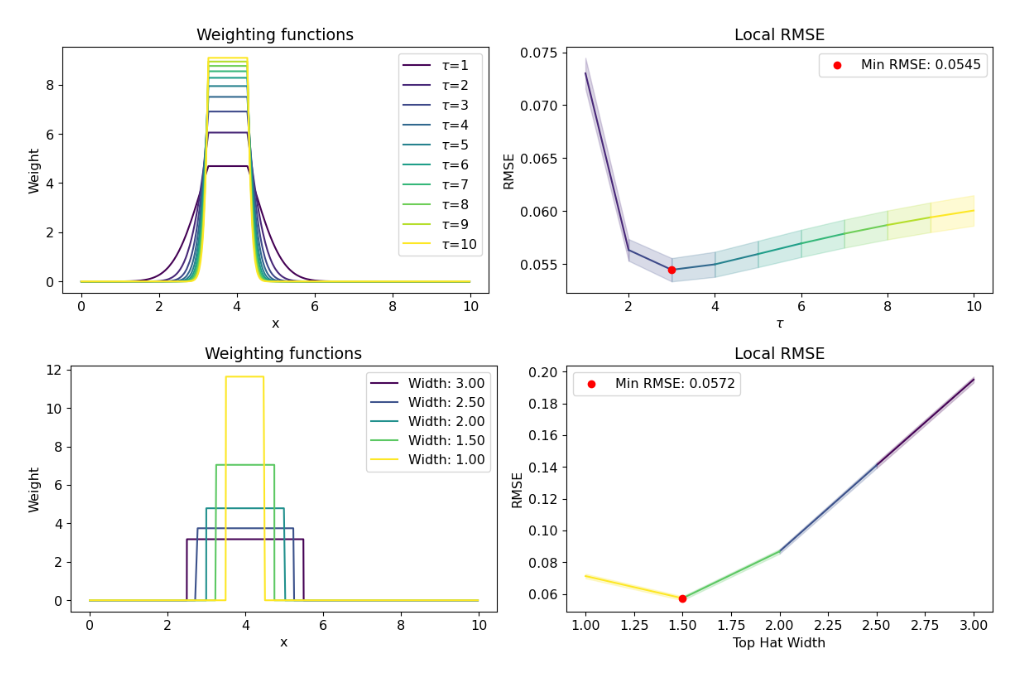}
    \caption{Comparison of reconstruction accuracy under different LASE weighting strategies. \textbf{Top left:} Continuous weighting functions \( w(Z_i) \) with varying concentration parameters \(\tau \in \{1, \dots, 10\}\), centered at \(x = 4\). 
    \textbf{Top right:} RMSE of reconstructed probabilities \(\hat{\mathbf{P}}_{ij}\) versus true \(\mathbf{P}_{ij}\), computed over nodes with latent positions in \([3.5, 4.5]\), for each \(\tau\).
    \textbf{Bottom left:} Top-hat weighting functions of varying width, also centered at \(x = 4\), corresponding to subgraph ASE on different-sized subgraphs. 
    \textbf{Bottom right:} RMSE of reconstructed probabilities for each top-hat width.
    Shaded regions indicate \(\pm 1\) SE over 10 sampled graphs. 
    }
    \label{fig:top_hat}
\end{figure}

We first consider a family of continuous weighting functions defined as
\begin{align*}
w(Z_i) = \left\{ 
\begin{aligned}
    &\exp(- \tau (Z_i - 4)^2)/N_\tau \quad & |Z_i-4|>0.5, \\
    &\exp(- \tau) /N_\tau \quad & |Z_i-4| \leq 0.5,
\end{aligned}
\right.
\end{align*}
where $N_\tau$ is the normalisation constant ensuring $\sum_i w(Z_i)/ n = 1$. The weighting functions for $\tau = 1,2, \ldots, 10$ are shown in the top left panel of Figure \ref{fig:top_hat}.
As $\tau$ increases, the weights become increasingly concentrated on the interval [3.5,4.5], approaching the subgraph ASE regime, where only a region of interest is considered. For each weighting, we perform LASE and reconstruct the probability matrix via clipped inner products,
$$\hat{\mathbf{P}}_{ij} := \max ( 0, \min(\langle \hat{X}_i, \hat{X}_j \rangle, 1)),$$
where $\hat{X}_i$ denotes the estimated embedding of node $i$. We evaluate RMSE over the nodes with latent positions in [3.5,4.5]. Results, averaged over 10 sampled graphs, are shown in the top right panel, along with $\pm1$ SE bars. We observe that overly heavy-tailed weightings (small $\tau$) lead to poor local reconstruction, while overly narrow weightings (large $\tau$) also degrade performance. The lowest RMSE occurs at $\tau=3$, highlighting the advantage of softly weighted LASE over hard-thresholded subgraph ASE.

Next, we consider a family of top-hat weighting functions of varying widths, centred at $x=4$, where nodes within the support of the weighting receive equal weight and others receive zero. These correspond to subgraph ASE on induced subgraphs of varying size. The weightings and corresponding RMSE results are shown in the bottom left and bottom right panels of Figure~\ref{fig:top_hat}, respectively. As before, RMSE is computed over the latent positions in [3.5,4.5]. The optimal subgraph for local reconstruction slightly exceeds this interval, suggesting that limited inclusion of neighboring nodes improves performance.

Crucially, the minimum RMSE achieved by the continuous weighting family was 0.0545, compared to 0.0572 for the top-hat (subgraph) approach. This demonstrates that LASE with smoothly decaying weights can outperform subgraph ASE in terms of reconstruction accuracy. However, this gain must be balanced against the computational efficiency and simplicity of subgraph-based methods.

\subsubsection{Visualisation}
\label{sec:visualisation}

In this experiment we investigate the ability of LASE, subgraph ASE and full ASE, to visually recover structures present in the latent positions of a graph. We begin by uniformly sampling 4000 latent positions on the square $[0,10]^2$. Then, among these random positions, we plant four small, visually recognisable shapes, each using 200 additional latent positions. The set of latent positions within radius 1 of each shape's center defines a neighbourhood, denoted $\mathcal{N}_{z^\star}$, where $z^\star$ is the central point. These neighbourhoods are illustrated in the first column of Figure \ref{fig:shapes}. 

\begin{figure}
    \centering
    \includegraphics[width=0.95\linewidth]{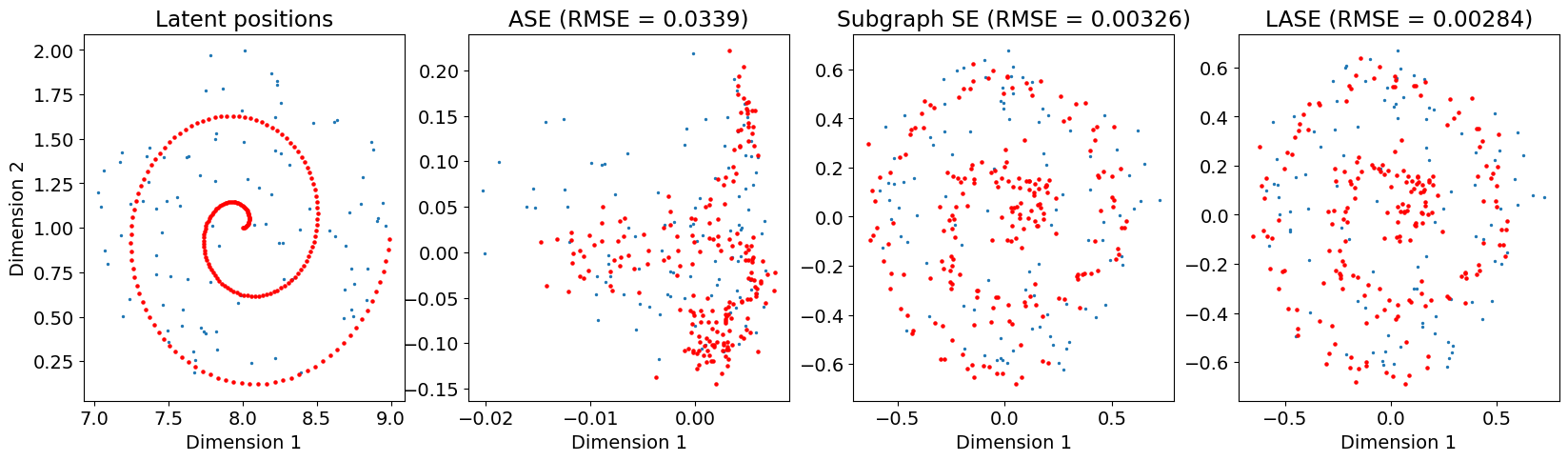}
    \includegraphics[width=0.95\linewidth]{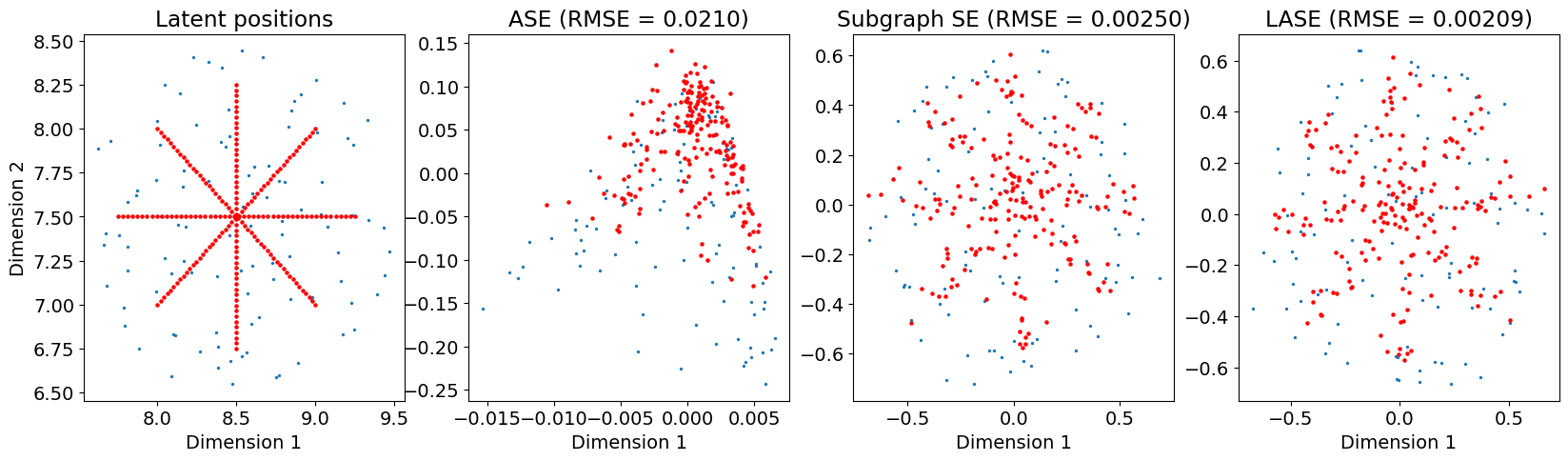}
    \includegraphics[width=0.95\linewidth]{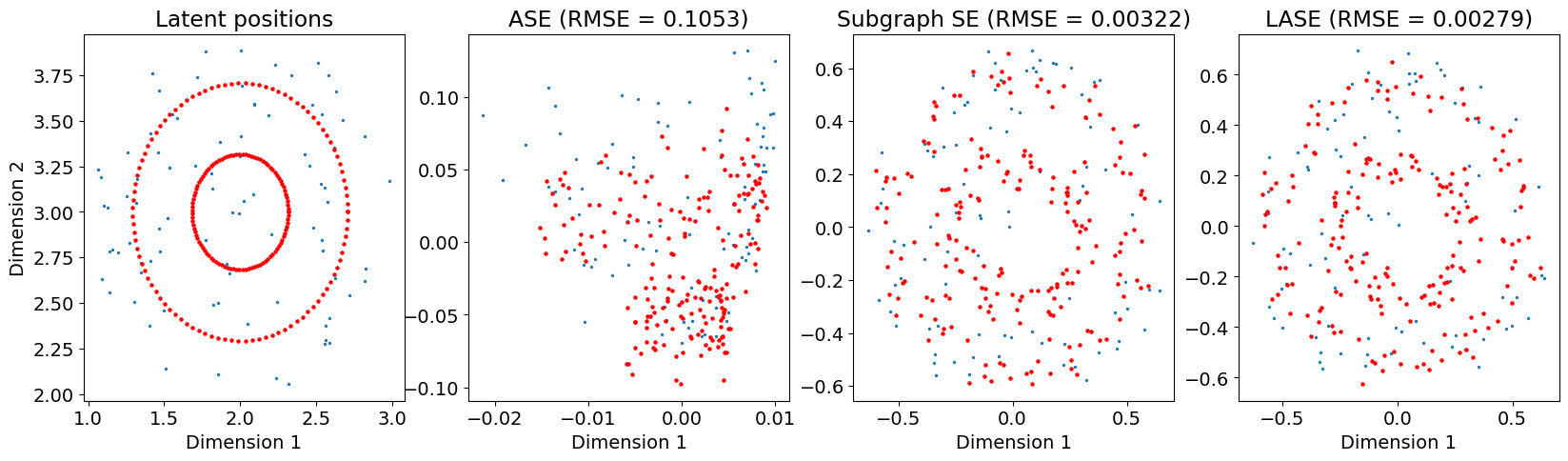}
    \includegraphics[width=0.95\linewidth]{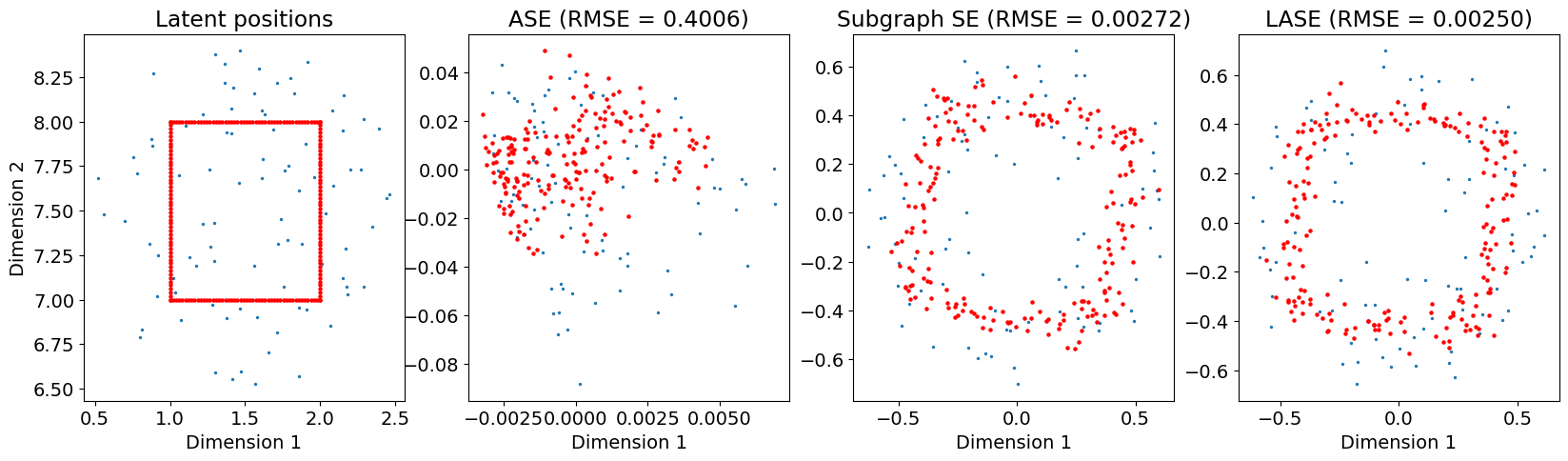}
    \caption{\textbf{Column 1}: Latent positions within radius 1 of a central point, $z^\star$. We call this the neighbourhood of $z^\star$, and represent it by $\mathcal{N}_{z^\star}$. \textbf{Column 2}: ASE is performed on the full graph into 3D, followed by PCA into 2D on embeddings of nodes in $\mathcal{N}_{z^\star}$. \textbf{Column 3}: Subgraph SE is performed into 3D on the subgraph consisting of \textit{at least} the nodes in $\mathcal{N}_{z^\star}$, followed by PCA into 2D on embeddings of nodes in $\mathcal{N}_{z^\star}$. \textbf{Column 4}: LASE is performed into 3D on the full graph, with weights corresponding to latent position $x \in \mathcal{Z}$ set to $\exp(- 0.4 \times \Vert x-z^\star \Vert^2$), followed by PCA into 2D on embeddings of nodes in $\mathcal{N}_z$. The reconstruction RMSE associated with each embedding is recorded in the titles.}
    \label{fig:shapes}
\end{figure}

Next, we generate an adjacency matrix $\mathbf{A}$ from the full set of latent positions; an edge is formed between nodes $i$ and $j$ with probability $\mathbf{P}_{ij} = \exp(-\frac{1}{2}\|z_i - z_j\|^2)$. We perform ASE on the full graph into three dimensions, and denote the embedding of node $i$ by $\hat{X}_i$. To reduce dimensionality for visualisation, we perform PCA on the subset $\{\hat{X}_i : i \in\mathcal{N}_{z^\star}\}$ for each shape, yielding the two-dimensional embedding shown in the second column of Figure \ref{fig:shapes}. The RMSE between the true probabilities $P_{ij}$ and the estimated probabilities $\hat{P}_{ij} = \min(\max(0,\langle \hat{X}_i, \hat{X}_j\rangle), 1)$ is reported in the figure headings.

For subgraph ASE, we perform ASE into three dimensions on a subgraph containing \textit{at least} all nodes in $\mathcal{N}_{z^\star}$. To determine the subgraph, we incrementally expand a radius around $z^\star$ and include all latent positions within it, selecting the radius that minimizes the RMSE within $\mathcal{N}_{z^\star}$. As before, we apply PCA to the embedded vectors $\{\hat{X}_i : i \in\mathcal{N}_{z^\star}\}$, and plot the resulting 2D embeddings and RMSEs in the third column of Figure \ref{fig:shapes}.

For LASE, we set the weight for node $i$ equal to $w_i = \exp(-\tau \|z_i - z^\star\|^2)$ (up to a global normalisation), where the concentration parameter $\tau$ is chosen to minimise the RMSE within $\mathcal{N}_z$. After performing LASE with these weights into 3D, we apply PCA on $\{X_i : i \in\mathcal{N}_{z^\star}\}$, reducing the point cloud to 2D. The resulting embeddings and RMSEs are shown in the fourth column of Figure \ref{fig:shapes}.

All four shapes are much more clearly recognisible in the subgraph ASE and LASE embeddings than in the full ASE. Moreover, the RMSEs are considerably smaller for the two local embedding methods, with LASE always performing marginally better than subgraph ASE in this respect. Comparing the outputs of subgraph ASE and LASE visually, each shape appears slightly `sharper' in the LASE embeddings. We note that in this example, where $d_\text{loc}(z) = 2$ for all $z \in [0,10]^2$, for the latent shapes to be visually recognisable in the local embeddings, it was necessary to first embed into three dimensions and then reduce to two dimensions via PCA. Embedding directly into two dimensions resulted in embeddings that did not visually resemble the shapes in the latent positions. The choice of three dimensions is in line with Theorem \ref{thm:lase_trunc_err}, which suggests embedding into a dimension in the set \{1,2,3\}.

\subsection{Road Network Data}
\label{sec:road_network}

\subsubsection{Linear Regression to predict node features}

To evaluate the practical efficacy of LASE in capturing local geometry, we perform inference on a real-world road network obtained from OpenStreetMap (OSM). We focus on the road network surrounding the city centre of Bristol, UK, where nodes represent intersections and edges represent road segments. The objective of this experiment is to demonstrate that local embeddings more accurately reconstruct physical node attributes (specifically latitude and longitude) compared to global embeddings, supporting the hypothesis that the network possesses a locally low-dimensional structure that is recoverable through LASE. This approach is motivated by the spectral properties of the Bristol road network illustrated in Figure \ref{fig:road_network_eig_decay}. While the eigenvalues of the global adjacency matrix decay slowly and lack a discernible eigengap (Figure \ref{fig:road_network_eig_decay}, left), the localised matrix $\mathbf{W}^{1/2}\mathbf{AW}^{1/2}$ exhibits a sharp decay before an ``elbow'' at $r=3$, with a notable gap between the third and fourth eigenvalues (Figure \ref{fig:road_network_eig_decay}, right). As the statistical error in Theorem \ref{thm:concentration} is controlled by the eigengap $\tilde{\lambda}_r-\tilde{\lambda}_{r+1}$, the emergence of such a gap at $r=3$ suggests performing LASE into three dimensions.

\paragraph{Remark on PSD assumption.} Applying our theoretical results to this dataset requires assuming the generative kernel $f$ is positive semi-definite. In spatial graphs like road networks, the probability of an edge is naturally modelled as a decreasing function of geographic distance. Standard proximity measures used for this purpose, such as the Gaussian kernel, are well-known to be positive definite, making the PSD assumption a natural fit for modelling spatial connectivity.

\paragraph{Experimental Setup.} We select 10 nodes $\{z_1, \dots, z_{10}\}$ from the network at random to serve as neighbourhood centres. For each $z_i$, we define a local neighbourhood $\mathcal{N}_m(z_i)$ consisting of the $m$ closest nodes according to Euclidean distance in their true coordinates, with $m \in [100, 300]$. Within the smallest neighbourhood ($m=100$), we randomly sample 10 nodes to serve as a held-out test set, assuming their coordinates are unknown. The remaining nodes constitute the training set.

As a baseline, we predict the test nodes' coordinates using the mean coordinates of their 1-hop neighbours (if a node has no neighbours with known coordinates, a new test node is drawn). For the embedding-based approaches, we produce three representations of dimension $d \in \{3, 20\}$:
\begin{enumerate}
    \item \textbf{Full ASE:} We compute the Adjacency Spectral Embedding on the entire graph and retain only the rows corresponding to the $m$ nodes in $\mathcal{N}_m(z_i)$.
    \item \textbf{Subgraph ASE:} We perform ASE on the subgraph induced by the $m+k$ closest nodes to $z_i$. We tune the expansion parameter $k \in \{0, 10, \dots, 100\}$ to maximize the $R^2$ of the linear regression on the training data.
    \item \textbf{LASE:} We tune the localization parameter $\tau$ to maximize the training $R^2$. The weights for node $x$ are defined as $w(x) = \exp[-\tau \cdot \text{dist}(x, z_i)]$, where $\text{dist}(\cdot, z_i)$ evaluates the distance to node $z_i$ based on the nodes' true coordinates (for training points) and baseline predictions (for test points).
\end{enumerate}

For each embedding method, we fit a linear regression from the latent positions to the true coordinates of the training set and evaluate the predictive performance on the test set using Mean Squared Error (MSE). The results, averaged over the 10 centers, are illustrated in Figure~\ref{fig:regression}.

\begin{figure}[t]
    \centering
    \includegraphics[width=1\linewidth]{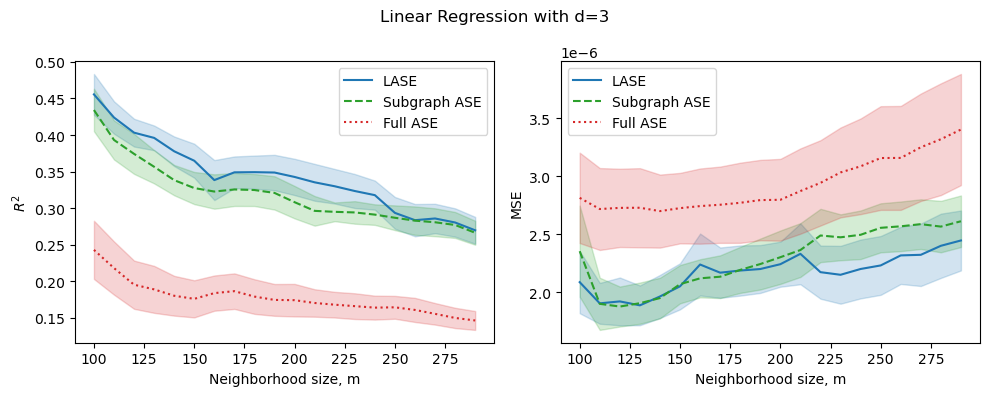}
    \includegraphics[width=1\linewidth]{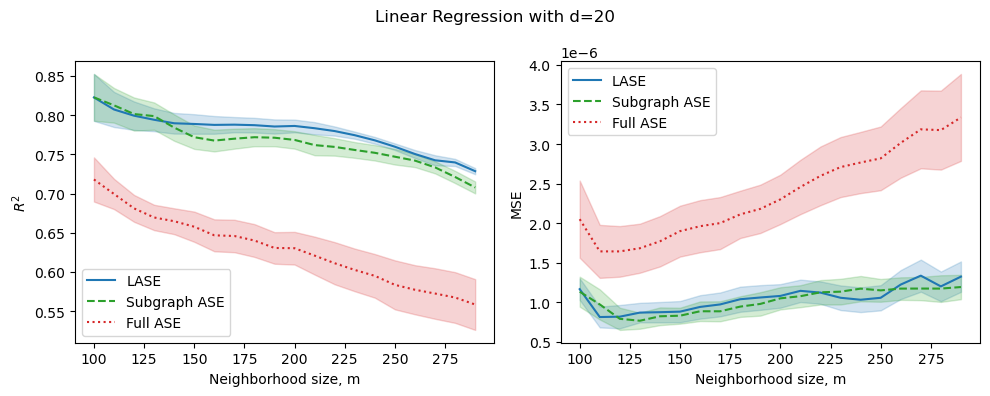}
    \caption{Linear regression of the true latitude and longitude coordinates of nodes onto embeddings (LASE, Subgraph ASE and Full ASE) of dimension $d=3$ (top row) and $d=20$ (bottom row). For all plots, the x-axis is the size of local neighbourhood (number of nodes) considered for the regression - see the main text for how this is determined. The left-hand plots show the mean $R^2$ statistic ($\pm$1 SE) over 10 randomly selected neighbourhood centres. The right-hand plots show the mean-squared error ($\pm$1 SE) for 10 randomly selected test points in the 100-node neighbourhood of the centres. }
    \label{fig:regression}
\end{figure}

We observe that for all embedding methods, the $R^2$ statistic decreases as the neighbourhood size $m$ increases. This trend supports the intuition that the network's geometry is more effectively approximated by a low-dimensional linear model over smaller, local patches. Notably, both LASE and Subgraph ASE consistently achieve significantly higher $R^2$ and lower MSE than Full ASE across all values of $m$. While Subgraph ASE and LASE perform comparably, LASE offers a more flexible ``soft'' localization that avoids the arbitrary boundaries of a hard subgraph cutoff. These results suggest that global embeddings (Full ASE) fail to resolve local geometric details required for accurate spatial inference, whereas LASE successfully isolates and preserves this local information.

\subsubsection{Global visualisation via combined LASE embeddings}
\label{sec:umap-lase}

In this section, we present a method to construct global visualisations of large graphs by combining local embeddings obtained via LASE.
We revisit the Bristol road network from the previous section and additionally consider a larger network from central London, again sourced from OpenStreetMap and centred on “Westminster, London, UK”.

Our simulated experiments suggest that embedding small subgraphs yields comparable accuracy to weighting the full nodeset continuously, with significant computational advantages. We leverage this by constructing global visualisations from local subgraph embeddings.

The procedure is as follows. We iteratively build overlapping subgraphs by randomly selecting a node not yet assigned to any subgraph, and including its $m$-hop neighborhood. This continues until every node belongs to at least one subgraph. We then apply ASE to each subgraph and compute the Euclidean distances between embedded nodes. These local distances populate a global distance matrix $\mathbf{D}$: for nodes co-occurring in multiple subgraphs, we average their distances; for nodes that never appear together, we assign a large finite value (e.g., 10× the maximum observed distance).

We refer to this method as `UMAP-LASE', wherein we input $\mathbf{D}$ to the UMAP algorithm \citep{mcinnes2018umap} with \texttt{metric=`precomputed'}. UMAP's ability to recover low-dimensional representations from local neighbourhood distances makes it well-suited to our setting. Intuitively, by our local-PCA exposition of LASE in Section \ref{sec:LASE_feature_map}, the distances between embeddings from separate subgraphs are comparable, and should approximate the true local geometry of $\{\phi(Z_1), \ldots, \phi(Z_n)\}$ up to error terms from LASE estimation and truncation.
As a baseline, we compare against a simpler alternative, `UMAP-ASE', in which we embed the entire graph via ASE into $r$ dimensions, and apply UMAP directly to the resulting point cloud. In this case, distance distortions arise from the ASE estimation and truncation error.

\begin{figure}[h]
    \centering
    \includegraphics[width=1\linewidth]{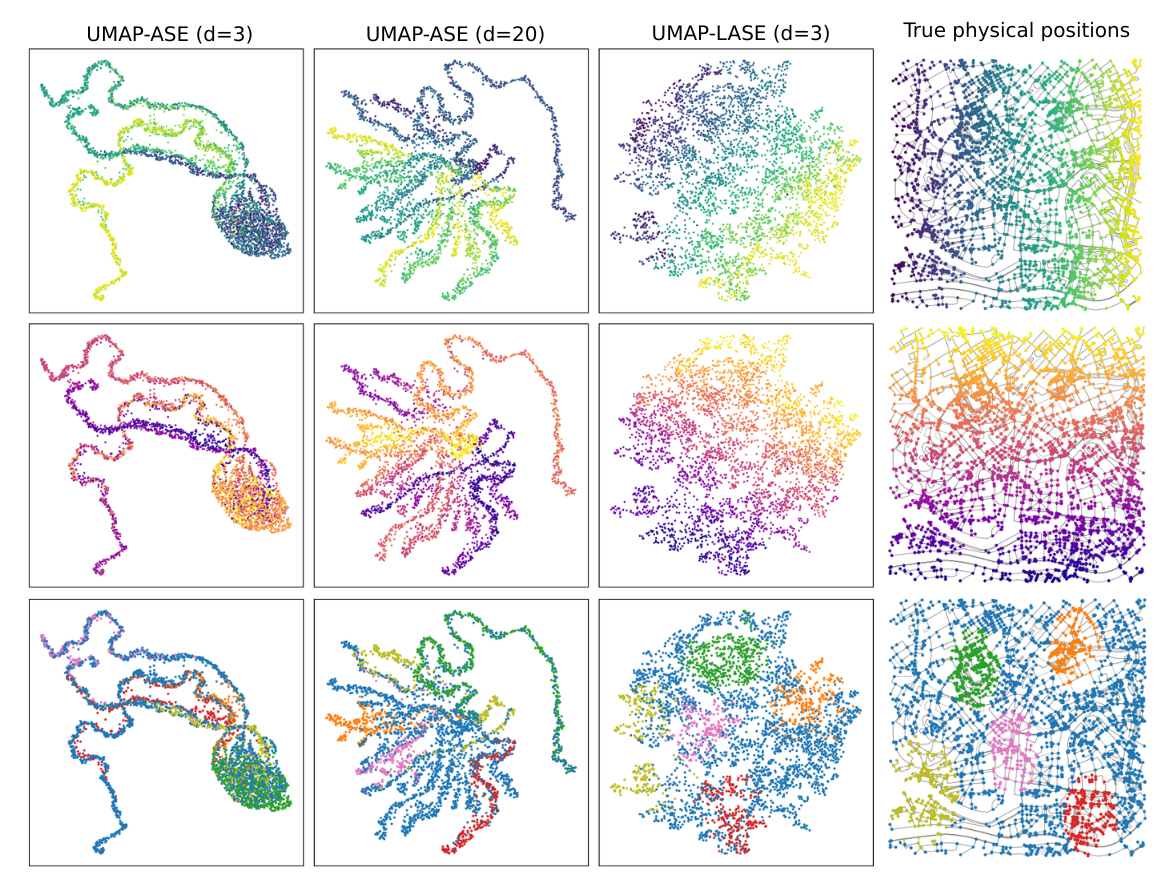}
    \caption{Global embeddings of the Bristol road network ($n=3857$), with true coordinates in the right-most column. The first two columns show UMAP applied to ASE embeddings of the full network in $d=3$ and $d=20$ dimensions. The third column shows UMAP-LASE using $m=10$ and $d=3$. Rows indicate colouring by longitude, latitude, and selected regions, respectively.}
    \label{fig:bristol_umap}
\end{figure}

\begin{figure}[h]
    \centering
    \includegraphics[width=1\linewidth]{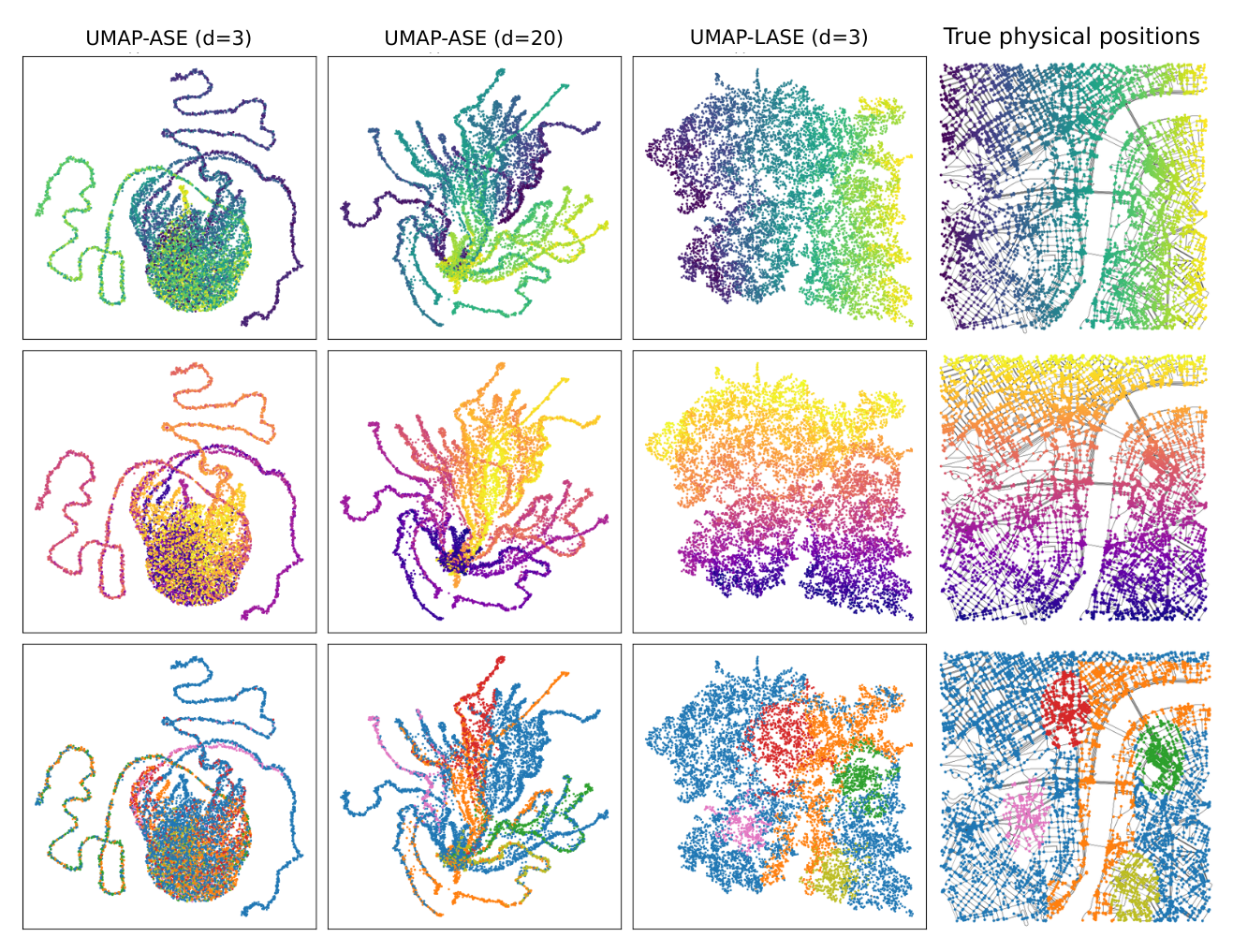}
    \caption{Global embeddings of the London road network ($n=9478$), analogous to Figure~\ref{fig:bristol_umap}. UMAP-LASE yields improved geometric fidelity of geographic features such as the River Thames.}
    \label{fig:london_umap}
\end{figure}

Figures~\ref{fig:bristol_umap} and~\ref{fig:london_umap} compare the global embeddings produced by UMAP-ASE and UMAP-LASE for the Bristol and London road networks. In both cases, the right-most column shows the true geographic layout of the network. The first two columns correspond to UMAP-ASE into $d=3$ and $d=20$ dimensions respectively. The third column shows the output of UMAP-LASE with $m=10$ and $d=3$, which internally uses 109 subgraphs for Bristol and 181 for London. 
Across all settings, UMAP parameters were set to \texttt{n\_neighbors=25} and \texttt{min\_dist=1}. We observed that the \texttt{min\_dist} parameter significantly affected cluster compactness, with \texttt{min\_dist=1} resulting in the best visualisations for all methods, whereas results were relatively robust to \texttt{n\_neighbors} in the range of 10–100. In both datasets, UMAP-LASE more faithfully preserved structural features. For example, in the London network, the embedded positions of nodes along opposing banks of the River Thames (highlighted in orange in the third row) for UMAP-LASE resembles the true physical positions. In fact, we even see clear separation of the banks at the start and end points of the stretch of river included in the network. In the UMAP-LASE embeddings, on the other hand, this structure is less apparent.

\begin{table}[ht]
    \centering
    \begin{tabular}{lccc}
        \toprule
        \textbf{Network} & \textbf{UMAP-ASE ($d=3$)} & \textbf{UMAP-ASE ($d=20$)} & \textbf{UMAP-LASE ($d=3$)} \\
        \midrule
        Bristol & 17.3 & 16.9 & 5.7 \\
        London  & 26.8 & 13.9 & 19.9 \\
        \bottomrule
    \end{tabular}
    \caption{Run-times (in seconds) for UMAP applied to different embeddings of the Bristol and London road networks.}
    \label{tab:umap_runtimes}
\end{table}

Table \ref{tab:umap_runtimes} reports the run-times for UMAP applied to different embeddings of the Bristol and London road networks. Interestingly, the UMAP-ASE procedure is faster for $d=20$ than $d=3$ for both networks. This somewhat counter-intuitive behaviour reflects the fact that the 3D embedding, while lower-dimensional, more severely compresses the network’s structure, leading to densely packed neighbourhoods that are harder to disentangle during optimisation. This phenomenon, sometimes referred to as the \textit{crowding problem}, is well-documented in the context of t-SNE and related methods \citep{van2008visualizing}. In contrast, higher-dimensional embeddings (e.g., 20D ASE) can retain more of the original graph’s local geometry, resulting in better-separated neighbourhoods and smoother convergence during optimisation. Applying UMAP to low-dimensional LASE does not appear to suffer from the crowding problem, producing embeddings that more closely resemble the true coordinates and running significantly faster than UMAP-ASE for both $d=3$ and $d=20$ on the Bristol network. For the London network, the subgraph splitting, embedding, and $\mathbf{D}$-matrix computation step took 4.5,s, while the UMAP step took 15.4s, which is comparable to the UMAP-ASE ($d=20$) run-time. For even larger graphs than those considered here, our procedure has the benefit of not requiring a high-dimensional spectral decomposition of a very large matrix.
In Appendix \ref{app:experiments}, we also show the superiority of UMAP-LASE ($d=3$) over applying UMAP to a 3D Node2Vec embedding \citep{grover2016node2vec}.

\section{Discussion}\label{sec:discussion}

In this paper, we introduced \emph{Local Adjacency Spectral Embedding} (LASE), a generalisation of classical ASE designed to emphasise localised structure in networks. By incorporating node weights into the adjacency spectral decomposition, LASE enables embeddings that reflect user-defined priorities. Importantly, the LASE algorithm preserves the interpretability and mathematical rigour of spectral methods.

Theoretically, we showed that LASE arises naturally as the solution to a localised version of the classical Schmidt approximation problem under a re-weighted latent position model. Our analysis provides bounds on approximation and statistical error, extending standard spectral convergence results to the localised setting. Notably, this framework remains flexible: the weighting function $w$ may be chosen to reflect prior knowledge without disrupting the theoretical guarantees.

Empirically, we demonstrated that LASE offers substantial improvements over classical ASE for local reconstruction and visualisation tasks. Moreover, we present a new technique for producing a global visualisation of graph-valued data using a combination of LASE and UMAP, exploiting the benefits and intermediary mechanisms of each method.

\paragraph{Future work} 

Our results open several promising avenues for future work:

\begin{itemize}
    \item \textbf{Learning weights from data.} While we focused on user-defined or heuristic weight functions, a natural extension is to learn the weights $w_i$ directly from data in a task-specific or data-adaptive manner. This could be accomplished through gradient-based optimisation, or by training a GNN to predict weights that improve downstream performance.

    \item \textbf{Integration with neural architectures.} LASE can serve as a principled preprocessing step or layer in hybrid models that combine spectral and deep learning methods. For example, LASE embeddings could be used as input features to GNNs, or LASE could be incorporated into spectral attention mechanisms for localised reasoning.

    \item \textbf{Theoretical generalisations.} While our analysis assumes a latent position model with a positive-definite kernel, it would be interesting to extend the framework to indefinite kernels and directed or weighted graphs.

    \item \textbf{Applications in scientific networks.} For real-world networks where local structure matters, future work can explore LASE as a tool for local structure discovery or interactive embeddings where one can "zoom in" and "move around" the network to reveal rich local structure, akin to Google maps.

\end{itemize}

We believe that LASE opens a new perspective on graph representation learning: one that prioritises local fidelity, theoretical tractability, and practical flexibility.

\bibliography{local}

\newpage
\appendix

\section{Background about kernels and the Mercer feature map }\label{app:ase}

\begin{thm}
\label{low_rank_approximation}
[\cite{hsing2015theoretical} \textit{(Theorem 4.6.8)] 
Let $f$ be a symmetric and nonnegative-definite kernel with the eigendecomposition
$$f(x,y) = \sum_{k=1}^\infty \lambda_k u_k(x) u_k(y).$$
Then, for any positive integer $r$,
$$\min_{\rank(f^{(r)})=r} \iint \left| f(x,y) - f^{(r)}(x,y)\right|^2  \mu (\mathrm{d}x) \mu (\mathrm{d}y) = \sum_{k=r+1}^\infty \lambda_k^2,  $$
where the minimum is achieved by $f^{(r)}(x,y) = \sum_{k=1}^r \lambda_k u_k(x) u_k(y)$. }
\end{thm}

\begin{lem}\label{lem:homeomorphism} Let $\phi:\mathcal{Z}\to \ell_2$ be the Mercer feature map associated with the latent position model as in Definition \ref{def:latent_position_model}. If for all $x, y \in \mathcal{Z}$ with $x \neq y$, there exists $a \in \mathcal{Z}$ such that $f(x, a) \neq f(y, a)$, then $\phi$ is an injective embedding, and consequently a homeomorphism between $\mathcal{Z}$ and $\mathcal{M}\coloneqq \{\phi(z) : z\in\mathcal{Z}\}$. 
\end{lem}
\begin{proof}
We need to prove that $\phi$ is continuous, invertible on its image $\mathcal{M}$ (injective), and has a continuous inverse. 
Continuity follows from the identity
$$
\|\phi(x)-\phi(y)\|_{\ell_2}^2=f(x,x)+f(y,y)-2f(x,y)
$$
and the continuity of $f$ as per Definition \ref{def:latent_position_model}.

Now for the proof of invertibility of $\phi$, suppose for purposes of contradiction that $\phi$ is not invertible, i.e., there exist $x,y \in \mathcal{Z}$ with $x\neq y$ such that $\phi(x)=\phi(y)$. Then for any $a \in \mathcal{Z}$,
$$
f(x,a)=\langle \phi(x),\phi(a)\rangle = \langle \phi(y),\phi(a)\rangle =f(y,a).
$$
But this contradicts the condition in the statement of the lemma that there exists some $a \in \mathcal{Z}$ where $f(x,a) \neq f(y,a)$. Therefore, $\phi$ must be invertible on $\mathcal{M}$.

It remains to prove that the inverse of $\phi$ is continuous (with respect to the $\ell_2$ distance). But this is automatic due to a general result in the theory of metric spaces \citep[Prop. 13.26]{sutherland2009introduction} concerning the inverse of a continuous, bijective mapping with a compact domain. 
\end{proof}

\section{Proofs for and supporting results for Section \ref{sec:LASE_feature_map}}
\label{reconstruction_error_appendix}

\subsection{Relating \texorpdfstring{$\mathcal{A}_w$}{Aw} and \texorpdfstring{$\mathcal{K}_w$}{Kw}}\label{app:relating_A_K}
The following lemma characterises the non-zero eigenvalues and corresponding eigenvectors of $\mathcal{K}_w$ defined in \eqref{eq:K_w_defn}, and will be used in the proof of Theorem \ref{thm:projection_thm}.
\begin{lem}
\label{lem:covariance_eigenvalues}
The distinct non-zero eigenvalues of $\mathcal{K}_w$ are equal to those of $\mathcal{A}_w$. Furthermore, $\mathcal{K}_w$ and $\mathcal{A}_w$ have the same number of orthonormal eigenvectors and eigenfunctions, respectively, associated with each distinct eigenvalue.
    With $(\Tilde{u}_k)_{k\geq 1}$ being orthonormal eigenfunctions of $\mathcal{A}_w$ with associated eigenvalues $(\Tilde{\lambda}_k)_{k \geq 1}$, for any $k \geq 1$ such that $\Tilde{\lambda}_k > 0$,
    $$\bm{v}_k := \Tilde{\lambda}_k^{-1/2} \int \Tilde{u}_k(x) \phi(x) \mu_w (\mathrm{d}x) $$
    is a member of $\ell_2$ satisfying $\mathcal{K}_w \bm{v}_k = \Tilde{\lambda}_k \bm{v}_k$, $\Vert \bm{v}_k \Vert = 1$ and $\langle \bm{v}_j, \bm{v}_k \rangle = 0$ for $j \neq k$.
\end{lem}

\begin{proof}
We have that
\begin{align*}
    \mathcal{K}_w \bm{v}_k &= \Tilde{\lambda}_k^{-1/2} \int \phi(y) \phi(y)^\top \mu_w (\mathrm{d}y)\int \Tilde{u}_k(x) \phi(x) \mu_w (\mathrm{d}x) \\
    &= \Tilde{\lambda}_k^{-1/2} \int \phi(y) \int \Tilde{u}_k(x) f(y,x) \mu_w (\mathrm{d}x) \mu_w (\mathrm{d}y) \\
    &= \Tilde{\lambda}_k^{1/2} \int \phi(y) \Tilde{u}_k(y) \mu_w (\mathrm{d}y) \\
    &= \Tilde{\lambda}_k \bm{v}_k.
\end{align*}
Therefore $\Tilde{\lambda}_k, \bm{v}_k$ is an eigenvalue, eigenvector pair for $\mathcal{K}_w$. Furthermore, we have
\begin{align*}
    \Vert \bm{v}_k \Vert^2 
    = \Tilde{\lambda}_k^{-1} \iint \Tilde{u}_k(x) \Tilde{u}_k(y) f(x,y) \mu_w (\mathrm{d}x) \mu_w (\mathrm{d}y) 
    = \int \Tilde{u}_k(x)^2 \mu_w (\mathrm{d}x) = 1,
\end{align*}
and for $j \neq k$,
\begin{align*}
    \langle \bm{v}_j, \bm{v}_k \rangle &= \Tilde{\lambda}_j^{-1/2} \Tilde{\lambda}_k^{-1/2} \iint \Tilde{u}_j(x) \Tilde{u}_k(y) f(x,y) \mu_w (\mathrm{d}x) \mu_w (\mathrm{d}y) \\
    &= \Tilde{\lambda}_j^{-1/2} \Tilde{\lambda}_k^{1/2} \int \Tilde{u}_j(x) \Tilde{u}_k(x) \mu_w (\mathrm{d}x) =0.
\end{align*}

Now, let $\{\bm{v}_k\}_{k \geq1}$ be orthonormal eigenvectors of $\mathcal{K}_w$ corresponding to eigenvalues $\{\Tilde{\lambda}_k\}_{k\geq1}$, and consider the functions
$$
g(\Tilde{u}_k) := \Tilde{\lambda}_k^{-1/2} \int \Tilde{u}_k(x) \phi(x) \mu_w (\mathrm{d}x),
$$
so that $\bm{v}_k= g(\Tilde{u}_k)$, and
$$u_{v_k}(\cdot) = h(\bm{v}_k)(\cdot) := \Tilde{\lambda}_k^{-1/2}\phi(\cdot)^\top \bm{v}_k.$$
For any $k \geq 1$ such that $\Tilde{\lambda}_k > 0$, we have that
\begin{align*}
    \mathcal{A}_w u_{v_k}(x) &= \Tilde{\lambda}_k^{-1/2} \int f(x,y) \phi(y)^\top \bm{v}_k \mu_w (\mathrm{d}y) \\
    &= \Tilde{\lambda}_k^{-1/2}  \phi(x)^\top \left( \int \phi(y) \phi(y)^\top \mu_w (\mathrm{d}y) \right) \bm{v}_k \\
    &= \Tilde{\lambda}_k^{-1/2}  \phi(x)^\top \mathcal{K}_w \bm{v}_k \\
    &= \Tilde{\lambda}_k^{1/2} \phi(x)^\top \bm{v}_k \\
    &= \Tilde{\lambda}_k u_{v_k}(x),
\end{align*}
and
\begin{align*}
    \Vert h(\bm{v}_k) \Vert^2 &= \Tilde{\lambda}_k^{-1} \bm{v}_k^\top \mathbb{E}_{\mu_w} \left[ \phi(X) \phi(X)^\top \right] \bm{v}_k = \Vert \bm{v}_k
\Vert^2 =1 \end{align*}

Therefore, $\Tilde{\lambda}_k, u_{v_k}(\cdot)$ is an eigenvalue, eigenfunction pair for $\mathcal{A}_w$. Since the relation holds in both directions, it follows that $\mathcal{A}_w$ and $\mathcal{K}_w$ have the same non-zero eigenvalues. 

Furthermore, we have that 
$$(g \circ h)(\bm{v}_k) = \bm{v}_k \hspace{20pt} \text{and} \hspace{20pt} (h \circ g)(\Tilde{u}_k) = \Tilde{u}_k,$$
so there is a 1-1 correspondence between the eigenvectors and eigenfunctions given by the functions $g$ and $h$. It follows that $\mathcal{A}_w$ and $\mathcal{K}_w$ have the same number of orthonormal eigenvectors and eigenfunctions, respectively, associated with each distinct eigenvalue.
\end{proof}

\subsection{Proof of Theorem \ref{thm:projection_thm}}

\begin{proof}[Proof of Theorem \ref{thm:projection_thm}]

 Let $Z \sim \mu_w$, and let $\bm{q}_1, \ldots, \bm{q}_r \in \ell_2$ be an orthonormal basis for some subspace $U$, and define the orthogonal projection operator $\boldsymbol{\Pi}_r: \ell_2 \to U$ as $\boldsymbol{\Pi}_r := \sum_{i=1}^r \bm{q}_i \langle \cdot, \bm{q}_i \rangle$.

First, we show that
$$ \mathbb{E} \left[ \Vert \phi(Z) - \boldsymbol{\Pi}_r \phi(Z) \Vert^2  \right] = \mathbb{E} \left[ \Vert \phi(Z) \Vert^2  \right] - \mathbb{E} \left[ \Vert \boldsymbol{\Pi}_r \phi(Z) \Vert^2  \right].$$

In the following, we use that since $\boldsymbol{\Pi}_r$ is an orthogonal projection, for every $\bm{x},\bm{y} \in \ell_2$, $\langle \bm{x}, \boldsymbol{\Pi}_r \bm{y} \rangle = \langle \boldsymbol{\Pi}_r \bm{x}, \boldsymbol{\Pi}_r \bm{y} \rangle = \langle \boldsymbol{\Pi}_r \bm{x}, \bm{y} \rangle$.
\begin{align*}
    &\mathbb{E} \left[ \Vert \phi(Z) - \boldsymbol{\Pi}_r \phi(Z) \Vert^2  \right]\\ 
    &= \mathbb{E} \left[ \big\langle \phi(Z) - \boldsymbol{\Pi}_r \phi(Z), \phi(Z) - \boldsymbol{\Pi}_r \phi(Z) \big\rangle \right] \\
    &= \mathbb{E} \left[ \big\langle \phi(Z),\phi(Z) \big\rangle  + \big\langle \boldsymbol{\Pi}_r \phi(Z), \boldsymbol{\Pi}_r\phi(Z) \big\rangle - 2 \big\langle \phi(Z), \boldsymbol{\Pi}_r\phi(Z) \big\rangle \right] \\
    &= \mathbb{E} \left[ \big\langle \phi(Z),\phi(Z) \big\rangle  + \big\langle \boldsymbol{\Pi}_r \phi(Z), \boldsymbol{\Pi}_r\phi(Z) \big\rangle - 2 \big\langle \boldsymbol{\Pi}_r \phi(Z), \boldsymbol{\Pi}_r\phi(Z) \big\rangle \right] \\
    &= \mathbb{E} \left[ \big\langle \phi(Z),\phi(Z) \big\rangle  - \big\langle \boldsymbol{\Pi}_r \phi(Z), \boldsymbol{\Pi}_r\phi(Z) \big\rangle \right] \\
    &=  \mathbb{E} \left[ \Vert \phi(Z) \Vert^2  \right] - \mathbb{E} \left[ \Vert \boldsymbol{\Pi}_r \phi(Z) \Vert^2  \right].
\end{align*}
Therefore, 
$$\argmin_{\boldsymbol{\Pi}_r}\mathbb{E} \left[ \Vert \phi(Z) - \boldsymbol{\Pi}_r \phi(Z) \Vert^2  \right] = \argmax_{\boldsymbol{\Pi}_r}\mathbb{E} \left[ \Vert \boldsymbol{\Pi}_r \phi(Z) \Vert^2  \right].$$

Using $\Vert \phi(x) \Vert^2=f(x,x)=\|\phi_w(x)\|^2$, we also have,
$$\mathbb{E} \left[ \Vert \phi(Z) \Vert^2  \right]= \mathbb{E}[f(X,X)] = \mathbb{E} \left[ \|\phi_w(Z)\|^2 \right] = \mathbb{E} \left[ \sum_{i=1}^\infty \Tilde{\lambda}_i |\Tilde{u}_i(Z)|^2 \right] =  \sum_{i=1}^\infty \Tilde{\lambda}_i .$$

Now, define $\alpha_i(\cdot) := \langle \bm{q}_i, \phi(\cdot) \rangle$, and note that $\langle \bm{q}_i, \bm{q}_j\rangle = \delta_{ij}$, where $\delta_{ij}$ is the Kronecker delta. Then we have that

\begin{align*}
    \mathbb{E} \left[ \Vert \boldsymbol{\Pi}_r \phi(Z) \Vert^2  \right] 
    &= \mathbb{E} \left[ \big\langle \boldsymbol{\Pi}_r \phi(Z), \boldsymbol{\Pi}_r \phi(Z) \big\rangle  \right] \\
    &= \mathbb{E} \left[ \big\langle \sum_{i=1}^r \alpha_i(Z) \bm{q}_i, \sum_{j=1}^r \alpha_j(Z) \bm{q}_j  \big\rangle  \right] \\
    &= \mathbb{E} \left[ \sum_{i,j =1}^r \alpha_i(Z) \alpha_j(Z) \langle \bm{q}_i, \bm{q}_j \rangle \right] \\
    &= \mathbb{E} \left[ \sum_{i=1}^r \alpha_i(Z)^2 \right] \\
    &= \mathbb{E} \left[ \sum_{i=1}^r \langle \bm{q}_i, \phi(Z)\rangle^2 \right] \\
    &= \mathbb{E} \left[ \sum_{i=1}^r \bm{q}_i^\top \phi(Z) \phi(Z)^\top \bm{q}_i \right] \\
    &=  \sum_{i=1}^r \bm{q}_i^\top \mathbb{E}\left[  \phi(Z) \phi(Z)^\top \right]\bm{q}_i \\
     &=  \sum_{i=1}^r \bm{q}_i^\top \mathcal{K}_w \bm{q}_i. 
\end{align*}
It follows by the Courant-Fischer-Weyl min-max principle and Lemma \ref{lem:covariance_eigenvalues}, that
$$\max_{\bm{\Pi}_r}\mathbb{E} \left[ \Vert \boldsymbol{\Pi}_r \phi(Z) \Vert^2  \right] = \sum_{i=1}^r \Tilde{\lambda}_i,$$
and therefore 
$$\min_{\boldsymbol{\Pi}_r} \mathbb{E} \left[ \Vert \phi(Z) - \boldsymbol{\Pi}_r \phi(Z) \Vert^2  \right] = \sum_{i = r+1}^\infty \Tilde{\lambda}_i,$$
proving part 1) of the statement of the theorem.
Moreover, the projection that achieves this minimum is $\mathbf{V}_r \mathbf{V}_r^\top ( \cdot) = \sum_{i=1}^r \bm{v}_i \langle \cdot, \bm{v_i} \rangle$, which is part 2) of the statement of the theorem.

Lastly, representing a projected point $\mathbf{V}_r \mathbf{V}_r^\top \phi(x)$ in the basis $\{\bm{v}_1, \ldots, \bm{v}_r\}$,
\begin{align*}
    \left[ \mathbf{V}_r^\top \phi(x) \right]_i 
    &= \bm{v}_i^\top \phi(x) \\
    &= \Tilde{\lambda}_i^{-1/2} \int \Tilde{u}_i(y) \phi(y)^\top \phi(x) \mu_w (\mathrm{d}y) \\
    &= \Tilde{\lambda}_i^{-1/2} \int \Tilde{u}_i(y) f(y,x) \mu_w (\mathrm{d}y) \\
    &= \Tilde{\lambda}_i^{1/2} \Tilde{u}_i(x) ,
\end{align*}
we see that $\mathbf{V}_r^\top \phi(x)$ coincides with the truncated feature map $\phi_w^{(r)}(x)$ associated with $\mathcal{A}_w$. That proves part 3) of the statement of the theorem.
\end{proof}

\section{Proofs and supporting results for Section \ref{sec:LASE_stat_error}}
\label{app:concentration result}

Let us define the symmetrically weighted operator $\mathcal{A}_\text{sym}: L_2(\mathcal{Z}, \mu) \to L_2(\mathcal{Z}, \mu)$ by
$$\mathcal{A}_\text{sym} g(x) = \int w(x)^{1/2} f(x,y) w(y)^{1/2} g(y) \mu (\mathrm{d}y),$$
for any function $g: \mathcal{Z} \to \R$ such that $g \in L_2(\mathcal{Z}, \mu)$. The following lemma makes the link between the Mercer feature map associated with $\mathcal{A}_\text{sym}$ and $\mathcal{A}_w$.
\begin{lem}
\label{lem:A_sym_eigs}
     \textit{$\mathcal{A}_\text{sym}$ has the same eigenvalues as $\mathcal{A}_w$,  $(\Tilde{\lambda}_k)_{k \geq 1}$, and has corresponding eigenfunctions  $(u^*_k)_{k \geq 1}$ which are orthonormal with respect to the measure $\mu$, where $u^*_i(\cdot) = w(\cdot)^{1/2}\Tilde{u}_i(\cdot)$. It follows that the Mercer feature map associated with $\mathcal{A}_\text{sym}$ is $\phi_\text{sym}(\cdot) = w(\cdot)^{1/2} \phi_w(\cdot)$.}
\end{lem}

\begin{proof} Let $\{ \tilde{\lambda}_i, \tilde{u}_i\}_{i \geq 1}$ denote the eigen-value/function pairs for $\mathcal{A}_w$. Then, using that $\mu_w(\mathrm{d}x) = w(x) \mu(\mathrm{d}x)$, we have
\begin{align*}
    \int f(x,y) \tilde{u}_i(y) \mu_w(\mathrm{d}y) = \int f(x,y)\tilde{u}_i(y) w(y) \mu(\mathrm{d}y) = \tilde{\lambda}_i \tilde{u}_i(x) 
\end{align*}
for all $i \geq 1$, and making the substitution $\tilde{u}_i(\cdot) = w(\cdot)^{-1/2} u^*_i(\cdot)$ and multiplying by $w(x)^{1/2}$, we get
$$ \int w(x)^{1/2} f(x,y) w(y)^{1/2} u^*_i(y) \mu(\mathrm{d}y) = \lambda u^*_i(x).$$
Therefore $\{\lambda_i, u^*_i\}$ is an eigen-value/function pair for $\mathcal{A}_\text{sym}$.

A similar calculation holds for the other direction, showing that if $\{\lambda_i, u_i^*\}_{i \geq 1}$ are the eigenvalue, eigenfunction pairs for $\mathcal{A}_\text{sym}$ then for all $i\geq1$, $\lambda_i$ is also an eigenvalue for $\mathcal{A}_w$, and has corresponding eigenfunction $u_i(\cdot) =  w(\cdot)^{-1/2} u^*_i(\cdot)$. Thus, $\mathcal{A}_\text{sym}$ and $\mathcal{A}_w$ have the same eigenvalues.

To show that $\{u^*_i\}_{i\geq1}$ are orthonormal with respect to $\mu$, we use that $\{\tilde{u}_i \}_{i\geq 1}$ are orthonormal eigenfunctions w.r.t. $\mu_w$, i.e.
$$\int \tilde{u}_i(x) \tilde{u}_j(x) \mu_w(\mathrm{d}x) = \begin{cases}
    1 &\text{if } i=j \\
    0 &\text{otherwise}.
\end{cases}$$
Making the substitution $\tilde{u}_i(\cdot) = w(\cdot)^{-1/2} u^*_i(\cdot)$, and using that $\mu_w(\mathrm{d}x) = w(x) \mu(\mathrm{d}x)$, we get that
$$\int u^*_i(x) u^*_j(x) \mu(\mathrm{d}x) = \begin{cases}
    1 &\text{if } i=j \\
    0 &\text{otherwise}.
\end{cases}.$$

Finally, the Mercer feature map associated with $\mathcal{A}_\text{sym}$ is $\phi_\text{sym}(\cdot) := [\lambda_1^{1/2} u^*_1(\cdot), \lambda_2^{1/2} u^*_2(\cdot), \ldots ]$, and therefore we can write it as $\phi_\text{sym}(\cdot) = w(\cdot)^{1/2} \phi_w(\cdot)$.
\end{proof}

Define $\tilde{\mathbf{A}} := \mathbf{W}^{1/2}\mathbf{AW}^{1/2}$, $\tilde{\mathbf{P}} := \mathbf{W}^{1/2}\mathbf{PW}^{1/2}$ and recall $w^* := \sup_{x \in \mathcal{Z}} w(x)$. For a fixed $r \geq 1$, let $\mathbf{S}_{\tilde{\mathbf{A}}}$ be the diagonal matrix consisting of the $r$ largest eigenvalues of $\tilde{\mathbf{A}}$, and let $\mathbf{U}_{\tilde{\mathbf{A}}}$ be the matrix comprised of the corresponding eigenvectors. Let $\mathbf{S}_{\tilde{\mathbf{P}}}$ and $\mathbf{U}_{\tilde{\mathbf{P}}}$ be defined similarly.

The strategy for proving Theorem \ref{thm:concentration} closely follows the proof of \citet[Thm 3.1]{tang2013universally}, with the integral operator $\mathcal{A}$ there replaced with the symmetrically weighted operator $\mathcal{A}_\text{sym}$, i.e. the operator associated with kernel $f_\text{sym}(x,y) = w(x)^{1/2} f(x,y) w(y)^{1/2}$ and measure $\mu$. 
 For conciseness, all probability statements related to $\tilde{\mathbf{A}}$ or associated spectral components are assumed to hold conditionally on latent positions $ Z_1, \ldots, Z_n  \overset{\text{i.i.d.}}{\sim} \mu$.

\begin{proposition}
\label{prop:a-p}
With probability at least $1- \eta$, we have
$$\Vert \tilde{\mathbf{A}} - \tilde{\mathbf{P}} \Vert \leq 2w^* \sqrt{n \log (n / \eta)}.$$    
\end{proposition}

\proof
From \cite{tang2013universally}, we have that with probability at least $1- \eta$,
$$\Vert \mathbf{A} - \mathbf{P} \Vert \leq 2 \sqrt{n \log (n / \eta)}.$$
Therefore,
\begin{align*}
    \Vert \tilde{\mathbf{A}} - \tilde{\mathbf{P}} \Vert
    &= \Vert \mathbf{W}^{1/2} \left(\mathbf{A} - \mathbf{P} \right) \mathbf{W}^{1/2} \Vert \\
    &\leq \Vert \mathbf{W}^{1/2} \Vert \Vert \mathbf{A} - \mathbf{P} \Vert \Vert\mathbf{W}^{1/2} \Vert \\
    &\leq w^* \Vert \mathbf{A} - \mathbf{P} \Vert \\
    &\leq  2w^* \sqrt{n \log (n / \eta)},
\end{align*}
with probability at least $1- \eta$.
\endproof

The following proposition, adapted from \cite{tang2013universally}, shows that the projection matrix for the subspace spanned by $\tilde{\mathbf{A}}$ is close to the projection matrix for the subspace spanned by $\tilde{\mathbf{P}}$. 

\begin{proposition}
\label{prop:proj_a-p}
    \textit{ Let $\mathcal{P}_{\tilde{\mathbf{A}}} = \mathbf{U}_{\tilde{\mathbf{A}}} \mathbf{U}_{\tilde{\mathbf{A}}}^\top$ and $\mathcal{P}_{\tilde{\mathbf{P}}} = \mathbf{U}_{\tilde{\mathbf{P}}} \mathbf{U}_{\tilde{\mathbf{P}}}^\top$. Denote by $\delta_r$ the quantity $\lambda_r (\mathcal{A}_\text{sym}) - \lambda_{r+1} (\mathcal{A}_\text{sym}) $, and suppose $\delta_r > 0$. If $n$ is such that $\delta_r \geq 8 w^* (1 + \sqrt{2}) \sqrt{n log (n / \eta)}$, then with probability at least $1 - 2\eta$,
    $$\Vert \mathcal{P}_{\tilde{\mathbf{A}}} - \mathcal{P}_{\tilde{\mathbf{P}}} \Vert \leq 4 w^* \sqrt{ \frac{\log(n/ \eta)}{n \delta_r^2} }.$$}
\end{proposition}

\proof
The proof of Proposition \ref{prop:proj_a-p} is a generalisation of the proof of Proposition 3.3 in \cite{tang2013universally}, using that the kernel $f_\text{sym}$ is bounded above by $w^*$, rather than 1 as in the original proof.

Let $\kappa = \sup_{x \in \mathcal{Z}} f_\text{sym}(x,x)$. By \citet[Prop. 10]{rosasco2010on}, we have with probability at least $1- \eta$,
$$\sup_{i\geq1}\left|\lambda_i(\mathcal{A}_\text{sym}) - \frac{\lambda_i(\tilde{\mathbf{P}})}{n}\right| \leq \frac{2\sqrt{2} \kappa \sqrt{\log(2/\eta)}}{\sqrt{n}}.$$
Using that  $f_\text{sym}(x,y) = w(x)^{1/2} f(x,y) w(y)^{1/2}$, it follows that $\kappa = \sup_{x \in \mathcal{Z}} w(x) = w^*$. Thus, with probability at least $1-\eta$,
$$\frac{\lambda_r(\tilde{\mathbf{P}})}{n} - \frac{\lambda_{r+1}(\tilde{\mathbf{P}})}{n} \geq \delta_d - 4\sqrt{2}w^* \sqrt{\frac{\log(2/\eta)}{n}}. $$

Let $S_1$ and $S_2$ be defined
\begin{align*}
    &S_1 = \left\{\lambda : \lambda \geq \lambda_r(\tilde{\mathbf{P}}) - 2w^*\sqrt{n \log(n/\eta)} \right\} \\
     &S_2 = \left\{\lambda : \lambda < \lambda_{r+1}(\tilde{\mathbf{P}}) + 2w^*\sqrt{n \log(n/\eta)} \right\}.
\end{align*}
Then with probability $1-\eta$,
\begin{equation}
\label{eq:dist_s1_s2}
\begin{aligned}
    \text{dist}(S_1, S_2) &\geq n \delta_r - 4 \sqrt{2}w^* \sqrt{n \log(2/\eta)} - 4w^*\sqrt{n \log(n/\eta)} \\
    &\geq n \delta_r - 4(1+ \sqrt{2})w^* \sqrt{n \log(n/\eta)}.
\end{aligned}
\end{equation}
Suppose for now that $S_1$ and $S_2$ are disjoint, i.e. dist$(S_1, S_2) > 0$. Let $\mathcal{P}_{\tilde{\mathbf{A}}}(S_1)$ be the matrix for the orthogonal projection  onto the subspace spanned by the eigenvectors of $\tilde{\mathbf{A}}$ whose corresponding eigenvalues lie in $S_1$. Let $\mathcal{P}_{\tilde{\mathbf{P}}}(S_1)$ be defined similarly. Then by the Davis-Kahan sin $\Theta$ theorem \citep{davis1970rotation}, we have
$$\| \mathcal{P}_{\tilde{\mathbf{A}}}(S_1) -  \mathcal{P}_{\tilde{\mathbf{P}}}(S_1)\| \leq \frac{\Vert \tilde{\mathbf{A}} - \tilde{\mathbf{P}} \Vert}{\text{dist}(S_1, S_2)}.$$
Then by equation \eqref{eq:dist_s1_s2} and Proposition \ref{prop:a-p}, we have with probability at least $1- \eta$,
$$\| \mathcal{P}_{\tilde{\mathbf{A}}}(S_1) -  \mathcal{P}_{\tilde{\mathbf{P}}}(S_1)\| \leq \frac{ 2w^* \sqrt{n \log (n / \eta)}}{ n \delta_r - 4(1+ \sqrt{2})w^* \sqrt{n \log(n/\eta)}} \leq 4w^*\sqrt{\frac{ \log (n / \eta)}{n \delta_r^2}},$$
provided that $4(1+ \sqrt{2})w^* \sqrt{n \log(n/\eta)} \leq n \delta_r/2$.

Finally, we note that if $4(1+ \sqrt{2})w^* \sqrt{n \log(n/\eta)} \leq n \delta_r/2$ then $S_1$ and $S_2$ are disjoint. Therefore the eigenvalues of $\tilde{\mathbf{P}}$ that lie in $S_1$ are exactly the $r$ largest eigenvalues of $\tilde{\mathbf{P}}$ and $\mathcal{P}_{\tilde{\mathbf{P}}}(S_1) = \mathbf{U}_{\tilde{\mathbf{P}}} \mathbf{U}_{\tilde{\mathbf{P}}}^\top$. Similarly, if $\Vert \tilde{\mathbf{A}} - \tilde{\mathbf{P}} \Vert \leq 2w^* \sqrt{n \log (n / \eta)}$, then  $\mathcal{P}_{\tilde{\mathbf{A}}}(S_1) = \mathbf{U}_{\tilde{\mathbf{A}}} \mathbf{U}_{\tilde{\mathbf{A}}}^\top$. The result of the proposition is thus shown.
\qed



Let $\mathscr{H}$ be the RKHS for $f_\text{sym}$. Define the linear operator $ \mathscr{A}_{\mathscr{H},n}$ on $\mathscr{H}$ as:
$$ \mathscr{A}_{\mathscr{H},n} \eta = \frac{1}{n} \sum_{i=1}^n \left\langle \eta, f_\text{sym} ( \cdot, Z_i)\right\rangle_{\mathscr{H}} f_\text{sym} (\cdot, Z_i).$$
The eigenvalues of $\mathscr{A}_{\mathscr{H},n}$ and $\tilde{\mathbf{P}}$ coincide, and the former is the extension of the latter as an operator on $\R^n$ to an operator on $\mathscr{H}$. In other words, $\mathscr{A}_{\mathscr{H},n}$ is a linear operator on $\mathscr{H}$ induced by $f_\text{sym}$ and $Z_1, \ldots, Z_n$.

The following lemma shows that the rows of $\mathbf{U}_{\tilde{\mathbf{P}}} \mathbf{S}_{\tilde{\mathbf{P}}}^{1/2} $ correspond to projecting $\phi_\text{sym}(Z_i)$ using $\hat{\mathcal{P}}_r$, where $\hat{\mathcal{P}}_r$ is the projection onto the $r$-dimensional subspace spanned by the eigenfunctions associated with the $r$ largest eigenvalues of $\mathscr{A}_{\mathscr{H},n}$.

\begin{lem}
\label{lem:proj_upsp}
    \textit{Let $\hat{\mathcal{P}}_r$ be the projection onto the subspace spanned by the eigenfunctions corresponding to the $r$ largest eigenvalues of $\mathscr{A}_{\mathscr{H},n}$. The rows of $\mathbf{U}_{\Tilde{\mathbf{P}}} \mathbf{S}_{\Tilde{\mathbf{P}}}^{1/2}$ then correspond, up to some orthogonal transformation, to projections of the feature map $\phi_\text{sym}$ onto $\R^r$ via $\hat{\mathcal{P}}_r$, that is, there exists a unitary matrix $\mathbf{Q} \in \R^{r \times r}$ such that
\begin{equation}
    \mathbf{U}_{\Tilde{\mathbf{P}}} \mathbf{S}_{\Tilde{\mathbf{P}}}^{1/2} \mathbf{Q} = [ \imath(\hat{\mathcal{P}}_r(\phi_\text{sym}(Z_1))^\top | \cdots | \imath(\hat{\mathcal{P}}_r(\phi_\text{sym}(Z_n))^\top  ]^\top,
\end{equation}
where $\imath$ is the isometric isomorphism of a finite dimensional Hilbert space onto $\R^r$.}
\end{lem}

The proof of Lemma \ref{lem:proj_upsp} follows the same steps as the proof of Lemma 3.4 in \cite{tang2013universally}, with the operator there replaced with $\mathscr{A}_{\mathscr{H},n}$ as defined above, so the details are omitted. 

Now we will make use of the above results to prove Theorem \ref{thm:concentration}.  We first show that the projection of $\tilde{\mathbf{A}}$ onto the subspace spanned by $\mathbf{U}_{\tilde{\mathbf{A}}}$ is close, in some sense, to the projection of $\tilde{\mathbf{P}}$ onto the subspace spanned by $\mathbf{U}_{\tilde{\mathbf{P}}}$. Then, applying results on the convergence of the spectra of $\tilde{\mathbf{P}}$ to the spectra of $\mathcal{A}_\text{sym}$, we show that the subspace spanned by $\mathbf{U}_{\tilde{\mathbf{A}}}$ is also close, in a sense, to the subspace spanned by $\phi_\text{sym}^{(r)}$.
\begin{proof}[Proof of Theorem \ref{thm:concentration}]
Note that the sum of any row of $\Tilde{\mathbf{A}}$ is bounded by $w^*n$, thus $\Vert  \Tilde{\mathbf{A}}\Vert \leq w^* n$. Similarly $\Vert  \Tilde{\mathbf{P}}\Vert \leq w^* n$. On combining propositions \ref{prop:a-p} and \ref{prop:proj_a-p}, we get, with probability at least $1-2\eta$,
\begin{align*}
    \Vert \mathcal{P}_{\tilde{\mathbf{A}}} \tilde{\mathbf{A}} - \mathcal{P}_{\tilde{\mathbf{P}}} \tilde{\mathbf{P}} \Vert 
    &\leq  \Vert \mathcal{P}_{\tilde{\mathbf{A}}} (\tilde{\mathbf{A}} - \tilde{\mathbf{P}}) \Vert +  \Vert (\mathcal{P}_{\tilde{\mathbf{A}}}  - \mathcal{P}_{\tilde{\mathbf{P}}} )\tilde{\mathbf{P}} \Vert \\
    &\leq 2w^* \sqrt{n \log(n/ \eta)} + 4\delta_r^{-1} {w^*}^2 \sqrt{n \log(n/ \eta)} \\
    &\leq 6 \delta_r^{-1} {w^*}^2 \sqrt{n \log(n/ \eta)}
\end{align*}

By Lemma A.1 in the appendix of \cite{tang2013universally}, adjusted to our framework, there exists an orthogonal $\mathbf{Q} \in \R^{r \times r}$ such that
\begin{align*}
     \Vert \mathbf{U}_{\Tilde{\mathbf{A}}} \mathbf{S}_{\Tilde{\mathbf{A}}}^{1/2} \mathbf{Q} - \mathbf{U}_{\Tilde{\mathbf{P}}} \mathbf{S}_{\Tilde{\mathbf{P}}}^{1/2} \Vert 
     &\leq 6 \delta_r^{-1} {w^*}^2 \sqrt{n \log(n/ \eta)} \frac{\sqrt{r \Vert \mathcal{P}_{\tilde{\mathbf{A}}} \tilde{\mathbf{A}} \Vert} + \sqrt{r \Vert \mathcal{P}_{\tilde{\mathbf{P}}} \tilde{\mathbf{P}} \Vert}}{\lambda_r (\Tilde{\mathbf{P}})} \\
     &\leq 12 \delta_r^{-1} {w^*}^{5/2} \frac{n  \sqrt{r\log(n/ \eta)}}{\lambda_r (\Tilde{\mathbf{P}})}.
\end{align*}

Note that, by Theorem B.2 in the appendix of \cite{tang2013universally}, $\lambda_r (\Tilde{\mathbf{P}}) \geq n \lambda_r (\mathcal{A}_\text{sym})/2$ provided that $n$ satisfies  $ \lambda_r (\mathcal{A}_\text{sym}) > 4\sqrt{2} \times w^*\sqrt{n^{-1} \log(n/ \eta)}$.
Therefore, we have
\begin{equation}
\label{3.1 proof eq1}
    \Vert \mathbf{U}_{\Tilde{\mathbf{A}}} \mathbf{S}_{\Tilde{\mathbf{A}}}^{1/2} \mathbf{Q} - \mathbf{U}_{\Tilde{\mathbf{P}}} \mathbf{S}_{\Tilde{\mathbf{P}}}^{1/2} \Vert_F \leq 24  \delta_r^{-1} {w^*}^{5/2} \frac{ \sqrt{r\log(n/ \eta)}}{\lambda_r (\mathcal{A}_\text{sym})} \leq 24  \delta_r^{-2} {w^*}^{5/2} \sqrt{r\log(n/ \eta)},
\end{equation}
with probability at least $1-2\eta$.

Now, by Lemma \ref{lem:proj_upsp}, the rows of $\mathbf{U}_{\Tilde{\mathbf{P}}} \mathbf{S}_{\Tilde{\mathbf{P}}}^{1/2}$ are (up to some orthogonal transformation) the projections of the feature map $\phi_\text{sym}$ onto $\R^r$ via $\hat{\mathcal{P}}_r$.
 On the other hand, $\phi_\text{sym}^{(r)}(Z)$ is the projection of $f_\text{sym}(\cdot, Z)$ onto $\R^r$ via $\mathcal{P}_r$, where $\mathcal{P}_r$ is the projection onto the subspace spanned by the eigenfunctions corresponding to the $r$ largest eigenvalues of $\mathcal{A}_\text{sym}$. By Theorem B.2 in the Appendix of \cite{tang2013universally}, for all $Z$, we have
$$\Vert \hat{\mathcal{P}}_r f_\text{sym}(\cdot, Z) - \mathcal{P}_r f_\text{sym}(\cdot, Z)\Vert_\mathscr{H}  \leq \Vert \hat{\mathcal{P}}_r - \mathcal{P}_r \Vert_\text{HS} \Vert f_\text{sym}(\cdot, Z) \Vert_\mathscr{H} \leq 2\sqrt{2} {w^*}^2 \frac{\sqrt{\log (1/\eta)}}{\delta_r \sqrt{n}},$$
with probability at least $1-2\eta$. We therefore have, for some orthogonal $\Tilde{\mathbf{Q}} \in \R^{r \times r}$,
\begin{equation}
\label{3.1 proof eq2}
    \Vert  \mathbf{U}_{\Tilde{\mathbf{P}}} \mathbf{S}_{\Tilde{\mathbf{P}}}^{1/2} \Tilde{\mathbf{Q}}  - \mathbf{\Phi}_\text{sym}^{(r)}\Vert_F \leq 2\sqrt{2} {w^*}^2 \frac{\sqrt{\log (1/\eta)}}{\delta_r}
    \leq 3 \delta_r^{-2} {w^*}^3 \sqrt{r \log (n /\eta)}
\end{equation}
with probability at least $1-2\eta$, where $\mathbf{\Phi}_\text{sym}^{(r)}$ has $i$-th row equal to $\phi_\text{sym}^{(r)}(Z_i)$. Then on combining equations (\ref{3.1 proof eq1}) and (\ref{3.1 proof eq2}), we get
\begin{equation}
\label{sym_frob_norm}
    \left\Vert \mathbf{U}_{\Tilde{\mathbf{A}}} \mathbf{S}_{\Tilde{\mathbf{A}}}^{1/2} \mathbf{Q} - \mathbf{\Phi}_\text{sym}^{(r)}  \right\Vert_F \leq 27 \delta_r^{-2} {w^*}^{3} \sqrt{r \log(n / \eta)},
\end{equation}
and equation (\ref{conc_thm:eq_1}) in the statement of the theorem follows from pre-multiplying by $\mathbf{W}^{1/2} \mathbf{W}^{-1/2}$ and using that $\mathbf{\Phi}_w^{(r)} = \mathbf{W}^{-1/2} \mathbf{\Phi}_\text{sym}^{(r)}$.

With $X_i^\top$ denoting the $i$-th row of $\mathbf{W}^{-1/2}  \mathbf{U}_{\Tilde{\mathbf{A}}} \mathbf{S}_{\Tilde{\mathbf{A}}}^{1/2}$ (as in Algorithm \ref{alg:lase}), notice that $w(Z_i)^{1/2}X_i^\top\mathbf{Q}$ equals the $i$-th row of $\hat{\mathbf{\Phi}}^{(r)}_\text{sym} :=\mathbf{U}_{\Tilde{\mathbf{A}}} \mathbf{S}_{\Tilde{\mathbf{A}}}^{1/2} \mathbf{Q}$. To show equation (\ref{eq:3.1.2}) in the theorem, we first note that as the $\{ Z_i\}_{i=1}^n$ are independent and identically distributed, the $\{w(Z_i)^{1/2}X_i^\top\mathbf{Q}\}_{i=1}^n$ are exchangeable and hence identically distributed. 
Let $\eta = n^{-2}$. By conditioning on the event in equation (\ref{sym_frob_norm}), we have
\begin{align}
\begin{split}
\label{3.1 proof eq3}
    \mathbb{E} \left[ \Vert w(Z_i)^{1/2}X_i^\top\mathbf{Q} - \phi_\text{sym}^{(r)}(Z_i) \Vert \right] &\leq \sqrt{ \mathbb{E} \left[ \Vert w(Z_i)^{1/2}X_i^\top\mathbf{Q} - \phi_\text{sym}^{(r)}(Z_i) \Vert^2 \right]} \\
    &\leq \sqrt{\frac{1}{n}\mathbb{E} [ \Vert \hat{\mathbf{\Phi}}^{(r)}_\text{sym} - \mathbf{\Phi}_\text{sym}^{(r)} \Vert^2_F]} \\
    &\leq \frac{1}{\sqrt{n}} \sqrt{\left(1 - \frac{2}{n^2} \right) (27 \delta_r^{-2} {w^*}^{3} \sqrt{ 3r \log(n)})^2 + \frac{2}{n^2}2n} \\
    &\leq 27 \delta_r^{-2} {w^*}^{3} r\sqrt{ \frac{6 \log(n)}{n}},
\end{split}
\end{align}
because the worst case bound is $ \Vert \hat{\mathbf{\Phi}}^{(r)}_\text{sym} - \mathbf{\Phi}_\text{sym}^{(r)} \Vert_F^2 \leq 2nw^*$ with probability 1. 
Recalling that  $\phi_\text{sym}(\cdot) = w(\cdot)^{1/2} \phi_w(\cdot)$,
equation (\ref{eq:3.1.2}) follows from (\ref{3.1 proof eq3}), Markov's inequality and pre-multiplication by $w(Z_i)^{1/2} w(Z_i)^{-1/2}$.
\end{proof}

\section{Proofs and supporting results for Section \ref{sec:LASE_trunc_error}}
\label{app:trunc_error}

\begin{lemma}\label{lem:bias} Assume \textbf{A\ref{ass:Z_in_R^d}} -  \textbf{A\ref{ass:epsilon_sigma}}. Then 
$$
\|\bar{\phi}_w-\phi(z^\star)\|^2 = O(\epsilon).
$$
\end{lemma}

\begin{proof}
Under \textbf{A\ref{ass:phi(x)_not_zero}}, $z^\star$ is an interior point of $\mathcal{Z}$, so there exists $\delta>0$ such that $B_\delta\coloneqq \{x\in\mathbb{R}^d:\|x-z^\star\|<\delta\}$ is a subset of $\mathcal{Z}$. 
With $Z\sim \mu_w$, we may write:
\begin{multline}
\|\bar{\phi}_w-\phi(z^\star)\|^2= \|\mathbb{E}[\phi(Z)-\phi(z^\star)]\|^2 \\  \leq \mathbb{E}\left[\| \phi(Z)-\phi(z^\star) \|^2\mathbf{1}[Z\in B_\delta]\right]+\mathbb{E}\left[\| \phi(Z)-\phi(z^\star) \|^2\mathbf{1}[Z\notin B_\delta]\right].\label{eq:bias_decomp}
\end{multline}
For the first term in \eqref{eq:bias_decomp},  define $z_t\coloneqq t Z+ (1-t)z^\star$ for $t\in[0,1]$. Note that on the event $\{Z\in B_\delta\}$ we have  $z_t\in B_\delta$ hence $z_t\in\mathcal{Z}$. Then:
$$
\phi(Z) - 
\phi(z^\star) = \int_0^1 \frac{\mathrm{d}}{\mathrm{d}t}\phi(z_t)\mathrm{d}t=\int_0^1 \partial \phi(z_t)\mathrm{d}t \cdot(Z-z^\star),  
$$
where $\partial \phi(x)$ is the matrix whose $(k,i)$ element is the partial derivative of $\lambda_k^{1/2} u_k(x)$ with respect to $x_i$, where $x=(x^{(1)}, \ldots, x^{(d)}) \in \mathcal{Z}$. We have:
\begin{align}
\|\phi(Z)-\phi(z^\star)\|^2& = (Z-z^\star)^{\top}\cdot \left[\int_0^1 \partial \phi(z_t)\mathrm{d}t\right]^{\top}\cdot \int_0^1 \partial \phi(z_t)\mathrm{d}t \cdot(Z-z^\star)\nonumber\\
&= (Z-z^\star)^{\top}\cdot \int_0^1\left[ \partial \phi(z_t)\right]^{\top} \mathrm{d}t\cdot \int_0^1 \partial \phi(z_t)\mathrm{d}t \cdot(Z-z^\star)\nonumber\\
&= (Z-z^\star)^{\top}\cdot \int_0^1 \int_0^1\left[ \partial \phi(z_s)\right]^{\top} \partial \phi(z_t)\mathrm{d}s \mathrm{d}t \cdot(Z-z^\star)\nonumber\\
& = (Z-z^\star)^{\top}\cdot \int_0^1 \int_0^1\mathbf{H}_{z_s,z_t}\mathrm{d}s \mathrm{d}t \cdot(Z-
z^\star)\label{eq:lipschitz_proof}.
\end{align}
Under \textbf{A\ref{ass:differentiability}} the elements of the matrix $\mathbf{H}_{x,y}$ are continuous in $x,y$, and then uniformly bounded in $x,y$ since $\mathcal{Z}$ is compact. It follows that  $\sup_{x,y}\|\mathbf{H}_{x,y}\|\leq \sup_{x,y}\|\mathbf{H}_{x,y}\|_F<\infty$. This fact combined with \eqref{eq:lipschitz_proof} establishes that there exists a finite constant $c$ such that
$$
\| \phi(Z)-\phi(z^\star) \|^2\mathbf{1}[Z\in B_\delta] \leq c \|Z-z^\star\|^2.
$$
Combined with \textbf{A\ref{ass:epsilon_sigma}} and the fact that $\mathbb{E}[\|Z-z^\star\|^2]=\mathrm{tr}\mathbb{E}[(Z-z^\star)(Z-z^\star)^\top]$ this shows that the first term on the r.h.s. of \eqref{eq:bias_decomp} is $O(\epsilon)$. 

For the second term on the r.h.s of \eqref{eq:bias_decomp}, consider the bound:
\begin{align*}
\mathbb{E}\left[\| \phi(Z)-\phi(z^\star) \|^2\mathbf{1}[Z\notin B_\delta]\right] \leq 2 \mathbb{P}\left(Z\notin B_\delta \right) \leq \frac{2}{\delta} \mathbb{E}[\|Z-z^\star\|^2] = O(\epsilon),
\end{align*}
where the first inequality uses 
$$
\sup_{x,y\in\mathcal{Z}}\|\phi(x)-\phi(y)\|^2 = \sup_{x,y\in\mathcal{Z}} f(x,x)+f(y,y)-2f(x,y)\leq 2,
$$
the second inequality is an application of Markov's inequality, and the final equality holds by \textbf{A\ref{ass:higher_moments}}.

\end{proof}

\begin{proof}[Proof of proposition \ref{prop:eigenvalues_of_D}]

Consider the decomposition:
\begin{equation}\label{eq:D_w_first+decomp}
\mathcal{D}_w  = \epsilon \mathbf{S} +\mathcal{D}_w - \epsilon \mathbf{S}
\end{equation}
where 
\begin{align*}
\mathbf{S}  \coloneqq \partial \phi(z^\star) \boldsymbol{\Sigma} \partial\phi(z^\star)^\top =\frac{1}{\epsilon} \mathbb{E} \bigg[ \partial\phi(z^\star) (Z-z^\star)(Z-z^\star)^\top \partial\phi(z^\star)^\top \bigg], 
\end{align*}
with the equality holding under \textbf{A\ref{ass:epsilon_sigma}}.
The matrix $\mathbf{S}$ does not depend on $\epsilon$. Noting that under \textbf{A\ref{ass:epsilon_sigma}} we have $\boldsymbol{\Sigma} \succ 0$, hence $\boldsymbol{\Sigma}^{1/2}$ is full rank, the rank of $\mathbf{S}$ is equal to that of $ \boldsymbol{\Sigma}^{1/2} \partial\phi(z^\star)^\top \partial\phi(z^\star) \boldsymbol{\Sigma}^{1/2} = \boldsymbol{\Sigma}^{1/2} \mathbf{H}_{z^\star,z^\star} \boldsymbol{\Sigma}^{1/2}$, in  turn equal  to $d_{\text{loc}}(z^\star)$ as defined in  \eqref{eq:d_loc_defn}. Thus $\epsilon \mathbf{S}$ is rank $d_{\text{loc}}(z^\star)$ with non-zero eigenvalues which are all positive and $\Theta(\epsilon)$.

We now turn to the term $\mathcal{D}_w - \epsilon \mathbf{S}$ in \eqref{eq:D_w_first+decomp}. Since $z^\star$ is an interior point of $\mathcal{Z}$ under \textbf{A\ref{ass:Z_in_R^d}}, there exists $\delta>0$ such that $B_\delta\coloneqq\{x\in\mathbb{R}^d:\|x-z^\star\|<\delta\}$ is a subset of $\mathcal{Z}$. 

We decompose further:
\begin{equation}\label{eq:D_w-eS_decomp}
 \mathcal{D}_w - \epsilon \mathbf{S} = \mathbf{R} +\mathbf{Q}
\end{equation}
where
\begin{align*}
\mathbf{R}& \coloneqq  \mathbb{E}\left[\left\{(\phi(Z)-\phi(z^\star))(\phi(Z)-\phi(z^\star))^\top  - \partial\phi(z) (Z-z^\star)(Z-z^\star)^\top \partial\phi(z)^\top\right\} \mathbf{1}[X\in B_\delta] \right], \\
\mathbf{Q}& \coloneqq  \mathbb{E}\left[\left\{(\phi(Z)-\phi(z^\star))(\phi(Z)-\phi(z^\star))^\top - \partial\phi(z) (Z-z^\star)(Z-z^\star)^\top \partial\phi(z)^\top\right\}  \mathbf{1}[Z\notin B_\delta]\right].
\end{align*}
We first bound the nuclear (a.k.a. Schatten-$1$) norm $\|\cdot\|_\star$ of $\mathbf{Q}$:
\begin{align}
\|\mathbf{Q}\|_\star & \leq\sup_{x\in\mathcal{Z}}\|\phi(x)-\phi(z)\|^2\mathbb{P}(Z\notin B_\delta) + \|\mathbf{H}_{z^\star,z^\star}\| \mathbb{E}\left[\|Z-z^\star\|^2 \mathbf{1}[Z\notin B_\delta]\right]\nonumber \\
&\leq \left(2 +\|\mathbf{H}_{z^\star,z^\star} \|\sup_{x\in\mathcal{Z}}\|x-z^\star\|^2\right)  \mathbb{P}(Z\notin B_\delta)\nonumber\\
&\leq  \left(2 + \|\mathbf{H}_{z^\star,z^\star} \| \sup_{x\in\mathcal{Z}}\|x-z^\star\|^2\right) \frac{1}{\delta^3} \mathbb{E}[\|Z-z^\star\|^3] = o(\epsilon),\label{eq:Q_bound}
\end{align}
where the first inequality uses convexity and the triangle inequality of the nuclear norm, together with the facts that for a vector $u$, $\|uu^\top \|_\star=\|u\|^2$ and $\mathbf{H}_{z^\star,z^\star} = \partial\phi(z^\star)^{\top }\partial \phi(z^\star)$; the second inequality uses  $\|\phi(x)-\phi(z^\star)\|^2 = f(x,x)+f(z^\star,z^\star)-2f(x,z^\star)\leq 2$, noting $\sup_{x\in\mathcal{Z}}\|x-z^\star\|^2$ is finite since $\mathcal{Z}$ is compact under \textbf{A\ref{ass:Z_in_R^d}}; the final inequality holds by Markov's inequality, and the final equality holds by \textbf{A\ref{ass:higher_moments}}, noting that $\mathbf{H}_{z^\star,z^\star}$ does not depend on $\epsilon$.

We now turn to bounding the nuclear norm of the matrix $\mathbf{R}$ in \eqref{eq:D_w-eS_decomp}.   Define $Y_t \coloneqq tZ + (1-t)z^\star$ with $t\in[0,1]$. Note that on the event $\{Z\in B_\delta\}$, we have $Y_t\in B_\delta$ and hence $Y_t$ is in the interior of $\mathcal{Z}$, where both $\phi(\cdot)$ and $\partial\phi(\cdot)$ are well defined (recall \textbf{A\ref{ass:differentiability}}). Here $\partial \phi(x)$ is the matrix whose $(k,i)$ element is the partial derivative of $\lambda_k^{1/2} u_k(x)$ with respect to $x_i$, where $x=(x^{(1)}, \ldots, x^{(d)})^\top \in \mathcal{Z}$. On the event $\{Z\in B_\delta\}$, we may then write:
\begin{align}
    \phi(Z) - \phi(z^\star) &= \int_0^1 \frac{d}{dt} \phi (Y_t) \rd t = \int_0^1 \partial\phi(Y_t) \cdot \dot{Y_t} \rd t = \int_0^1 \partial\phi(Y_t) \rd t \cdot(Z-z^\star) \nonumber\\
    &= \partial\phi(z^\star)\cdot (Z-z^\star) + \int_0^1 [\partial\phi(Y_t) - \partial\phi(z^\star)] \rd t\cdot (Z-z^\star).\label{eq:X-z_proof}
\end{align}
Using the identity \eqref{eq:X-z_proof}, we have 
\begin{equation}
\mathbf{R} = \mathbf{R}_1+\mathbf{R}_2+\mathbf{R}_2
\end{equation}
where\begin{align*}
& \mathbf{R}_1 \coloneqq  \mathbb{E} \bigg[ \partial\phi(z^\star) (Z-z^\star)(Z-z^\star)^\top \left( \int_0^1 [\partial \phi(Y_t) - \partial \phi(z^\star)] \rd t\right)^\top\cdot  \mathbf{1}[Z\in B_\delta]\bigg],\\
& \mathbf{R}_2 \coloneqq \mathbf{R}_1^{\top}, \\
& \mathbf{R}_3 \coloneqq  \mathbb{E} \bigg[ \int_0^1 [\partial \phi(Y_t) - \partial \phi(z^\star)] \rd t  (Z-z^\star)(Z-z^\star)^\top  \left(\int_0^1 [\partial \phi(Y_t) - \partial \phi(z^\star)] \rd t\right)^{\top} \cdot \mathbf{1}[Z\in B_\delta] \bigg].
\end{align*}
On the event $\{Z\in B_\delta\}$, consider the following bound on the integral term which appears in $\mathbf{R}_1$, $\mathbf{R}_2$ and $\mathbf{R}_3$:
\begin{align}
    \left\Vert \int_0^1 [\partial \phi (Y_t) - \partial \phi(z^\star)] \rd t \right\Vert_F 
    &\leq   \int_0^1 \left\Vert \partial \phi (Y_t) - \partial \phi(z^\star) \right\Vert_F \rd t\nonumber \\
    &= \int_0^1 \left[ \langle \partial \phi (Y_t), \partial \phi (Y_t) \rangle_F  + 
    \langle \partial \phi(z^\star), \partial \phi(z^\star) \rangle_F  - 
    2\langle \partial \phi (Y_t), \partial \phi(z^\star) \rangle_F  \right]^{1/2} \rd t\nonumber \\
    &= \int_0^1 \left[ \text{tr}(\mathbf{H}_{Y_t,Y_t}) + \text{tr}(\mathbf{H}_{z^\star,z^\star}) - 2\text{tr}(\mathbf{H}_{Y_t, z^\star} ) \right]^{1/2} \rd t \nonumber\\
    &= \int_0^1 \left[ \sum_{j=1}^d (\mathbf{H}^{jj}_{z^\star,z^\star} - \mathbf{H}^{jj}_{Y_t, z^\star} + \mathbf{H}^{jj}_{Y_t,Y_t} -  \mathbf{H}^{jj}_{Y_t, z^\star}  ) \right]^{1/2} \rd t \nonumber\\
    &\leq  \int_0^1 \left[ \sum_{j=1}^d L_{j} t \Vert Z-z^\star \Vert +  L_{j} t \Vert Z-z^\star \Vert  \right]^{1/2} \rd t\nonumber \\
    &= \frac{4}{3}\left(\sum_{j=1}^d L_j \right)^{1/2} \cdot \Vert Z-z^\star \Vert^{1/2}\label{eq:intergral_est}, 
\end{align}
where $\langle \cdot, \cdot \rangle_F$ is the Frobenius inner product, and the constants $L_j$ are as in \textbf{A\ref{ass:differentiability}}. We obtain:
\begin{align}
\|\mathbf{R}_1\|_\star =\|\mathbf{R}_2\|_\star&\leq \mathbb{E}\left[\left\lVert\partial\phi(z^\star) (Z-z^\star)(Z-z^\star)^\top \left( \int_0^1 [\partial \phi(Y_t) - \partial \phi(z^\star)] \rd t\right)^\top\right\rVert_\star \mathbf{1}[Z\in B_\delta]\right] \nonumber\\
& \leq  \mathbb{E}\left[ \left\lVert \partial\phi(z^\star) (Z-z^\star)(Z-z^\star)^\top \right\rVert_F \left\lVert  \int_0^1 [\partial \phi(Y_t) - \partial \phi(z^\star)] \rd t \right\rVert_F \mathbf{1}[Z\in B_\delta] \right] \nonumber \\
&\leq \mathbb{E}\left[ \left\lVert \partial\phi(z^\star) \right\rVert_F  \left\lVert (Z-z^\star)(Z-z^\star)^\top \right\rVert_F \left\lVert  \int_0^1 [\partial \phi(Y_t) - \partial \phi(z^\star)] \rd t \right\rVert_F \mathbf{1}[Z\in B_\delta] \right] \nonumber\\
& = \mathbb{E}\left[ \mathrm{tr}(\mathbf{H}_{z^\star,z^\star})^{1/2}  \|Z-z^\star\|^2 \left\lVert \int_0^1 [\partial \phi(Y_t) - \partial \phi(z^\star)] \rd t \right\rVert_F \mathbf{1}[Z\in B_\delta] \right]\nonumber \\
&\leq c \mathbb{E}\left[\|Z-z^\star\|^{5/2}\right] = o(\epsilon),\label{eq:R_1_R_2_bound}
\end{align}
where the second inequality uses Schatten's matrix version of Holder's inequality (specifically $\|\mathbf{A}\mathbf{B}\|_{\star}\leq \|\mathbf{A}\|_F \|\mathbf{B}\|_F$); the third inequality uses the submultiplicativity of the Frobenius norm; the first equality uses $\mathbf{H}_{z^\star,z^\star}=\partial \phi(z^\star)^{\top }\partial\phi(z^\star)$; the last inequality holds for some constant $c$ which does not depend on $\epsilon$ using \eqref{eq:intergral_est}; and the final equality holds by \textbf{A\ref{ass:higher_moments}}.

With the short-hand for the vector: $v\coloneqq  \int_0^1 [\partial \phi(Y_t) - \partial \phi(z^\star)] \rd t \cdot  (Z-z^\star)$,
\begin{align}
\|\mathbf{R}_3\|_\star &\leq \mathbb{E} \bigg[ \left\lVert v v^{\top}  \right\rVert  \mathbf{1}[Z\in B_\delta]\bigg]\nonumber\\
& = \mathbb{E}\bigg[ \|v\|^2 \mathbf{1}[Z\in B_\delta] \bigg] \nonumber\\
&\leq \mathbb{E}\bigg[ \|Z-z^\star\|^2 \left\lVert \int_0^1 [\partial \phi(Y_t) - \partial \phi(z^\star)] \rd t\right\rVert_F^2 \mathbf{1}[Z\in B_\delta] \bigg] \nonumber\\
&\leq c   \mathbb{E}\bigg[ \|Z-z^\star\|^3 \bigg] = o(\epsilon),\label{eq:R_3_bound}
\end{align}
for some constant $c$, where $\|vv^\top \|_\star = \|v\|^2$, \eqref{eq:intergral_est} and \textbf{A\ref{ass:higher_moments}} have been used.

Combining \eqref{eq:Q_bound}, \eqref{eq:R_1_R_2_bound}, \eqref{eq:R_3_bound} and returning to \eqref{eq:D_w-eS_decomp} we have shown:
$$
\|\mathcal{D}_w-\epsilon \mathbf{S}\|_\star = o(\epsilon). 
$$
Then recalling the discussion of the eigenvalues of $\mathbf{S}$ above \eqref{eq:D_w-eS_decomp} and using Weyl's inequality, we have
\begin{equation}\label{eq:D_w_first_d_eigs}
\max_{1\leq i\leq d_{\text{loc}}(z)}|\lambda_i(\mathcal{D}_w)-\epsilon\lambda_i(\mathbf{S})|\leq \|\mathcal{D}_w-\epsilon \mathbf{S}\|  \leq \|\mathcal{D}_w-\epsilon \mathbf{S}\|_\star =o(\epsilon) ,
\end{equation}
which implies the first $d_{\text{loc}}(z)$ eigenvalues of $\mathcal{D}_w$ are all $\Theta(\epsilon)$, which is the first claim of the proposition.

It remains to prove the second claim of the proposition. By \citep[Thm. II]{kato1987variation} with $p=1$, there exist $j(\cdot)$, $k(\cdot)$ and $\ell(\cdot)$ mapping ${1,2,\ldots}$ to itself such that 
$$
\sum_{i\geq 1} |\lambda_{j(i)}(\mathcal{D}_w) - \epsilon \lambda_{k(i)}(\mathbf{S})| \leq \sum_{i\geq 1} | \lambda_{\ell(i)}(\mathcal{D}_w - \epsilon\mathbf{S})| = \|\mathcal{D}_w - \epsilon\mathbf{S}\|_\star.
$$
Then, recalling that $\mathbf{S}$ has exactly $d_{\text{loc}(z)}$ non-zero eigenvalues, all of which are positive,
\begin{align*}
\sum_{i > d_{\text{loc}(z)}} \lambda_i(\mathcal{D}_w) 
&=\left| \sum_{i \geq 1} \lambda_i(\mathcal{D}_w) - \sum_{i \leq d_{\text{loc}}(z)} \lambda_i(\mathcal{D}_w)\right|\\
&\leq  \left| \sum_{i \geq 1} \lambda_i(\mathcal{D}_w) - \sum_{i \geq 1 } \epsilon\lambda_i(\mathbf{S})\right|+  \left| \sum_{i \geq 1} \lambda_i(\epsilon \mathbf{S}) - \sum_{i \leq d_{\text{loc}}(z)} \lambda_i(\mathcal{D}_w)\right| \\
& = \left| \sum_{i\geq 1} \lambda_{k(i)}(\mathcal{D}_w) - \epsilon \lambda_{\ell(i)}(\mathbf{S}) \right|+  \left| \sum_{i\leq d_{\text{loc}(z)}}  \epsilon\lambda_i( \mathbf{S}) - \sum_{i \leq d_{\text{loc}}(z)} \lambda_i(\mathcal{D}_w)\right| \\
&\leq   \sum_{i\geq 1} \left| \lambda_{k(i)}(\mathcal{D}_w) - \epsilon \lambda_{\ell(i)}(\mathbf{S}) \right|  + d_{\text{loc}}(z)\max_{1\leq i\leq d_{\text{loc}}(z)}|\epsilon\lambda_i(\mathbf{S})-\lambda_i(\mathcal{D}_w)| \\
&\leq (1+d_{\text{loc}(z)}) \|\mathcal{D}_w - \epsilon\mathbf{S}\|_\star = o(\epsilon),
\end{align*}
using \eqref{eq:D_w_first_d_eigs} for the finality equality.

\end{proof}

\begin{lem}
\label{lem:big_O_epsilon}
   Let $a: \R_{\geq0} \to \R_{\geq0}$ be such that $ a(\epsilon) = O(\epsilon)$ as $\epsilon \to 0$. Then there exists a monotone decreasing sequence $\{\epsilon_k\}_{k\geq1}$, with $\epsilon_k \to 0$, such that along this sequence either 
    $$a(\epsilon_k) = \Theta(\epsilon_k) \quad \text{or} \quad a(\epsilon_k) = o(\epsilon_k).$$
\end{lem}

\begin{proof}
Define $b(\epsilon):= a(\epsilon)/\epsilon$. Since $a(\epsilon)=O(\epsilon)$, then 
$$L = \limsup_{\epsilon\to0} b(\epsilon)$$
is such that $0 \leq L < \infty.$ Now let us split into two cases:

\textbf{Case 1}: $L > 0$.
\newline By the definition of $\limsup$, there exists a sequence $\epsilon_k \to 0$ such that $\lim_{k \to \infty} b(\epsilon_k) = L$.
In particular, for sufficiently large $k$ we have
$$b(\epsilon_k) = \frac{a(\epsilon_k)}{\epsilon_k} \geq \frac{L}{2} \implies a(\epsilon_k) \geq \frac{L}{2} \epsilon_k.$$
On the other hand, since $b$ is bounded, say $b(\epsilon) \leq M$ for all small $\epsilon$, we also have
$$a(\epsilon_k) = b(\epsilon_k) \epsilon_k \leq M \epsilon_k.$$
Thus
$$\frac{L}{2} \epsilon_k \leq a (\epsilon_k) \leq M \epsilon_k$$
for all large $k$, and therefore $a(\epsilon_k) = \Theta(k)$.

\textbf{Case 2}: $L=0$.
In this case, by the definition of $\limsup$, \textit{any} decreasing sequence $\epsilon_k \to 0$ will be such that
$$\lim_{k\to\infty} \frac{a(\epsilon_k)}{\epsilon_k} = 0,$$
and thus $a(\epsilon_k) = o(\epsilon_k)$

In either case, we have produced a monotone decreasing $\epsilon_k \to 0$ so that $\{a(\epsilon_k)\}_{k\geq1}$ is \textit{either} $\Theta(\epsilon_k)$ \textit{or} $o(\epsilon_k)$.
\end{proof}

 \begin{proof}[Proof of Theorem \ref{thm:lase_trunc_err}]
We have that
$$\|\bar{\phi}_w\| \geq \|\phi(z^\star)\| - \|\bar{\phi}_w - \phi(z^\star)\|,$$
Moreover, by Lemma \ref{lem:bias} $\|\bar{\phi}_w - \phi(z^\star)\| = O(\epsilon)$, and thus  $\|\bar{\phi}_w - \phi(z^\star)\| \to 0$ as $\epsilon \to 0$.
Using Assumption \textbf{A\ref{ass:phi(x)_not_zero}} that $\phi(z^\star) \neq 0$, it follows that for $\epsilon<\epsilon_0$, $\|\bar{\phi}_w - \phi(z^\star)\| \leq \tfrac{1}{2}\|\phi(z^\star)\|$, and therefore
$$\|\bar{\phi}_w\| \geq \|\phi(z^\star)\| - \tfrac{1}{2}\|\phi(z^\star)\| =
\tfrac{1}{2}\|\phi(z^\star)\|.$$
We also have that
$$\|\bar{\phi}_w\| \leq \|\phi(z^\star)\| + \|\bar{\phi}_w - \phi(z^\star)\|,$$
and therefore, for $\epsilon > \epsilon_0$, 
$$\|\bar{\phi}_w\| \leq \|\phi(z^\star)\| + \tfrac{1}{2}\|\phi(z^\star)\| =
\tfrac{3}{2}\|\phi(z^\star)\|.$$
Thus, for $\epsilon<\epsilon_0$, $\|\bar{\phi}_w\|$ is bounded above and below by positive constants, and it follows that $\|\bar{\phi}_w\| = \Theta(1)$ as $\epsilon \to 0$.

Next, using a corollary of Weyl's inequality \citep[Cor. 4.3.3]{horn2012matrix}, often referred to as ``Weyl interlacing", we have that the eigenvalues of $\mathcal{D}_w$ and the eigenvalues of the rank-$2$ perturbation,
$$
\mathcal{K}_w = \mathcal{D}_w + \Bar{\phi}_w\Bar{\phi}_w^\top -  \bm{\gamma}_w\bm{\gamma}_w^\top,
$$
are interlaced in the following way:
\begin{equation}
\label{eq:interlacing_lower}
     \lambda_{i}(\mathcal{K}_w) \geq \lambda_{i+1}(\mathcal{D}_w) \hspace{10pt} \text{for } i=1,2, \ldots,
\end{equation}
and
\begin{equation}
\label{eq:interlacing_upper}
    \lambda_{i-1}(\mathcal{D}_w) \geq \lambda_{i}(\mathcal{K}_w )  \hspace{10pt} \text{for } i=2,3, \ldots .
\end{equation}
This is due to the rank-2 perturbation having exactly one positive and one negative eigenvalue.

Thus, by \eqref{eq:interlacing_upper}, we have that
$$\sum_{i> d_\text{loc}(z^\star)+1} \lambda_i(\mathcal{K}_w) \leq \sum_{i> d_\text{loc}(z^\star)+1} \lambda_{i-1}(\mathcal{D}_w) = \sum_{i> d_\text{loc}(z^\star)} \lambda_{i}(\mathcal{D}_w) = o(\epsilon),$$
where the final equality is due to Proposition \ref{prop:eigenvalues_of_D}. Recalling from Lemma \ref{lem:covariance_eigenvalues} that the spectra of $\mathcal{K}_w$ and $\mathcal{A}_w$ are equal, we have established \eqref{eq:lase_trunc_err}. It follows that $ \lambda_i(\mathcal{K}_w) = o(\epsilon)$ for all $i > d_\text{loc}(z^\star)+1$, which is the final case considered in \eqref{eq:lam_decay_cases} 

Next, let us examine the top eigenvalue $\lambda_1(\mathcal{K}_w)$. We have that $\lambda_1(\Bar{\phi}_w\Bar{\phi}_w^\top) = \Vert \bar{\phi}_w\Vert^2$ and $\lambda_1(\bm{\gamma}_w \bm{\gamma}_w^\top) = \| \bm{\gamma}_w \|^2$. Then using Lemma \ref{lem:bias}, we know that $$\lambda_1(\bm{\gamma}_w \bm{\gamma}_w^\top) = \| \bm{\gamma}_w \|^2 =  \|\bar{\phi}_w-\phi(z^\star)\|^2 = O(\epsilon).$$
Also, from Proposition \ref{prop:eigenvalues_of_D}, we have $\lambda_1(\mathcal{D}_w)=\Theta(\epsilon)$. Therefore, by the triangle inequality applied to the spectral norm,
$$\| \mathcal{D}_w - \bm{\gamma}_w \bm{\gamma}_w^\top\| \leq \|\mathcal{D}_w\| + \|\bm{\gamma}_w \bm{\gamma}_w^\top\| = O(\epsilon).$$
Now, by Weyl's inequality, we get the upper bound:
$$\lambda_1(\mathcal{K}_w) \leq \lambda_1(\Bar{\phi}_w\Bar{\phi}_w^\top) + \| \mathcal{D}_w - \bm{\gamma}_w \bm{\gamma}_w^\top\| = \|\bar{\phi}_w\|^2 + O(\epsilon).$$
To get a lower bound, we can inspect the Rayleigh quotient of $\mathcal{K}_w$ at $\bar{\phi}_w$, which satisfies
$$\frac{\bar{\phi}_w^\top \mathcal{K}_w \bar{\phi}_w}{\|\bar{\phi}_w\|^2} = \|\bar{\phi}_w\|^2 + \frac{\bar{\phi}_w^\top (\mathcal{D}_w - \bm{\gamma}_w \bm{\gamma}_w^\top) \bar{\phi}_w}{\|\bar{\phi}_w\|^2}
\geq \|\bar{\phi}_w\|^2 - \|\mathcal{D}_w - \bm{\gamma}_w \bm{\gamma}_w^\top\|.$$
Then since $\lambda_1(\mathcal{K}_w)$ is the maximum Rayleigh quotient,
$$\lambda_1(\mathcal{K}_w) \geq \|\bar{\phi}_w\|^2 - \|\mathcal{D}_w - \bm{\gamma}_w \bm{\gamma}_w^\top\| \geq \|\bar{\phi}_w\|^2 -O(\epsilon).$$
Combining the upper and lower bounds, we have that $\lambda_1(\mathcal{K}_w) =  \|\bar{\phi}_w\|^2 + O(\epsilon).$ Then, using the earlier result that $\| \bar{\phi}_w \| = \Theta(1)$, we have that $\lambda_1(\mathcal{K}_w) = \Theta(1)$.


It remains to examine the eigenvalues $\lambda_i(\mathcal{K}_w)$ for  $i=2,\ldots, d_\text{loc}(z^\star) +1$ as in \eqref{eq:lam_decay_cases}.
We first assume that $d_\text{loc}(z^\star) > 2$ (we consider the cases $d_\text{loc}(z^\star)= 1,2$ separately, later).
For $i=2,\ldots, d_\text{loc}(z^\star) -1,$ by Proposition \ref{prop:eigenvalues_of_D} and equations \eqref{eq:interlacing_lower} and \eqref{eq:interlacing_upper}, 
$$\Theta(\epsilon) = \lambda_{i-1}(\mathcal{D}_w) \geq \lambda_i(\mathcal{K}_w) \geq \lambda_{i+1}(\mathcal{D}_w)  = \Theta(\epsilon),$$ 
and therefore $\lambda_i(\mathcal{K}_w) = \Theta(\epsilon)$, which is the second case in \eqref{eq:lam_decay_cases}.

For $i=d_\text{loc}(z^\star), d_\text{loc}(z^\star) + 1 ,$ by Proposition \ref{prop:eigenvalues_of_D} and equation \eqref{eq:interlacing_upper}, we have
$$\Theta(\epsilon) = \lambda_{i-1}(\mathcal{D}_w) \geq \lambda_i(\mathcal{K}_w),$$ 
and therefore $\lambda_i(\mathcal{K}_w) = O(\epsilon)$, which are the third and fourth cases in \eqref{eq:lam_decay_cases}.

Next, consider the case $ d_\text{loc}(z^\star)=1$. We have that for $i = d_\text{loc}(z^\star)+1$, by Proposition \ref{prop:eigenvalues_of_D} and equation \eqref{eq:interlacing_upper},
$$\Theta(\epsilon) = \lambda_{i-1}(\mathcal{D}_w) \geq \lambda_{i}(\mathcal{K}_w)$$
and therefore $\lambda_i(\mathcal{K}_w) = O(\epsilon)$, which is the fourth case in \eqref{eq:lam_decay_cases}.

Finally, consider the case $ d_\text{loc}(z^\star)=2$. We have that for $i = d_\text{loc}(z^\star), d_\text{loc}(z^\star)+1$, by Proposition \ref{prop:eigenvalues_of_D} and equation \eqref{eq:interlacing_upper},
$$\Theta(\epsilon) = \lambda_{i-1}(\mathcal{D}_w) \geq \lambda_{i}(\mathcal{K}_w) $$
and therefore $\lambda_i(\mathcal{K}_w) = O(\epsilon)$, which are the third and fourth cases in \eqref{eq:lam_decay_cases}. This concludes the proof of the decay rates for all the cases shown in \eqref{eq:lam_decay_cases}.

It remains to prove the final claim of the theorem, concerning the eigengap. We will split this into three cases:

\textbf{Case 1}: $d_\text{loc}(z^\star)>2$. 
\newline From \eqref{eq:lam_decay_cases} we know that $\tilde{\lambda}_{d_\text{loc}(z^\star)} = O(\epsilon)$. By Lemma \ref{lem:big_O_epsilon}, this implies that there exists a monotone decreasing sequence $\{\epsilon_k\}_{k\geq1}$, with $\epsilon_k \to 0$, such that along this sequence either $\tilde{\lambda}_{d_\text{loc}(z^\star)} = \Theta(\epsilon_k)$ or $\tilde{\lambda}_{d_\text{loc}(z^\star)} = o(\epsilon_k)$. If $\tilde{\lambda}_{d_\text{loc}(z^\star)} = o(\epsilon_k)$, then we have that $\tilde{\lambda}_{d_\text{loc}(z^\star)-1} - \tilde{\lambda}_{d_\text{loc}(z^\star)} = \Theta(\epsilon_k)$. On the other hand, if along the sequence, $\tilde{\lambda}_{d_\text{loc}(z^\star)} = \Theta(\epsilon_k)$, then from \eqref{eq:lam_decay_cases} we know $\tilde{\lambda}_{d_\text{loc}(z^\star)+1} = O(\epsilon)$. Again using Lemma \ref{lem:big_O_epsilon}, this implies that there exists a monotone decreasing subsequence $\{\epsilon_k^\prime\}_{k\geq1}$ of $\{\epsilon_k\}_{k\geq1}$, with $\epsilon_k^\prime \to 0$ as $k\to \infty$, such that either $\tilde{\lambda}_{d_\text{loc}(z^\star)+1} = \Theta(\epsilon_k^\prime)$ or $\tilde{\lambda}_{d_\text{loc}(z^\star)+1} = o(\epsilon_k^\prime)$. If $\tilde{\lambda}_{d_\text{loc}(z^\star)+1} = o(\epsilon_k^\prime)$, then we have $\tilde{\lambda}_{d_\text{loc}(z^\star)} - \tilde{\lambda}_{d_\text{loc}(z^\star)+1} = \Theta(\epsilon_k^\prime)$. If $\tilde{\lambda}_{d_\text{loc}(z^\star)+1} = \Theta(\epsilon_k^\prime)$, then we have $\tilde{\lambda}_{d_\text{loc}(z^\star)+1} - \tilde{\lambda}_{d_\text{loc}(z^\star)+2} = \Theta(\epsilon_k^\prime)$, and the result is shown.

\textbf{Case 2}: $d_\text{loc}(z^\star)=2$. 
\newline By \eqref{eq:lam_decay_cases}, we have $\tilde{\lambda}_{d_\text{loc}(z^\star)} = O(\epsilon)$. Then by  Lemma \ref{lem:big_O_epsilon}, this implies that there exists a monotone decreasing sequence $\{\epsilon_k\}_{k\geq1}$, with $\epsilon_k \to 0$ as $k\to \infty$, such that either $\tilde{\lambda}_{d_\text{loc}(z^\star)} = \Theta(\epsilon_k)$ or $\tilde{\lambda}_{d_\text{loc}(z^\star)} = o(\epsilon_k)$. If $\tilde{\lambda}_{d_\text{loc}(z^\star)} = o(\epsilon_k)$, then $\tilde{\lambda}_{d_\text{loc}(z^\star)-1} - \tilde{\lambda}_{d_\text{loc}(z^\star)} = \Theta(1)$. If $\tilde{\lambda}_{d_\text{loc}(z^\star)} = \Theta(\epsilon_k)$, then we follow the same reasoning as the $d_\text{loc}(z^\star)>2$ case to show that there exists a monotone decreasing subsequence $\{\epsilon_k^\prime\}_{k\geq1}$ of $\{\epsilon_k\}_{k\geq1}$, with $\epsilon_k^\prime \to 0$ as $k\to \infty$, such that along this subsequence either $\tilde{\lambda}_{d_\text{loc}(z^\star)} - \tilde{\lambda}_{d_\text{loc}(z^\star)+1} = \Theta(\epsilon_k^\prime)$ or $\tilde{\lambda}_{d_\text{loc}(z^\star)+1} - \tilde{\lambda}_{d_\text{loc}(z^\star)+2} = \Theta(\epsilon_k^\prime)$.

\textbf{Case 3}: $d_\text{loc}(z^\star)=1$.
By \eqref{eq:lam_decay_cases} $\tilde{\lambda}_{d_\text{loc}(z^\star)+1} = O(\epsilon)$, which by Lemma \ref{lem:big_O_epsilon} implies there exists a  monotone decreasing sequence $\{\epsilon_k\}_{k\geq1}$, with $\epsilon_k \to 0$ as $k\to \infty$, such that either $\tilde{\lambda}_{d_\text{loc}(z^\star)+1} = \Theta(\epsilon_k)$ or $\tilde{\lambda}_{d_\text{loc}(z^\star)+1} = o(\epsilon_k)$. If $\tilde{\lambda}_{d_\text{loc}(z^\star)+1} = \Theta(\epsilon_k)$, then $\tilde{\lambda}_{d_\text{loc}(z^\star)+1} - \tilde{\lambda}_{d_\text{loc}(z^\star)+2} = \Theta(\epsilon_k)$. $\tilde{\lambda}_{d_\text{loc}(z^\star)+1} = o(\epsilon_k)$, then $\tilde{\lambda}_{d_\text{loc}(z^\star)} - \tilde{\lambda}_{d_\text{loc}(z^\star)+1} = \Theta(1)$. This proves the final case.
     
 \end{proof}

 \section{Additional experiments}
 \label{app:experiments}

 \subsection{Node2Vec followed by UMAP}

We compare our UMAP-LASE (and UMAP-ASE) techniques with an analogous approach using a Node2Vec embedding, which we refer to as UMAP-Node2Vec.  
Figures \ref{fig:bristol_n2v} and \ref{fig:london_n2v} illustrate that, in both networks, UMAP-LASE ($d=3$) more accurately reproduces the true physical positions compared to UMAP-Node2Vec ($d=3$).  

For both networks, Node2Vec was run with \texttt{walk\_length=30}, \texttt{num\_walks=50}, and \texttt{workers=4}. The corresponding run-times are summarised in Table \ref{tab:umap_runtimes_new}. Notably, UMAP-LASE achieves a significant speed advantage over UMAP-Node2Vec, highlighting one of its major practical benefits.

 \begin{figure}[h]
     \centering
     \includegraphics[width=1\linewidth]{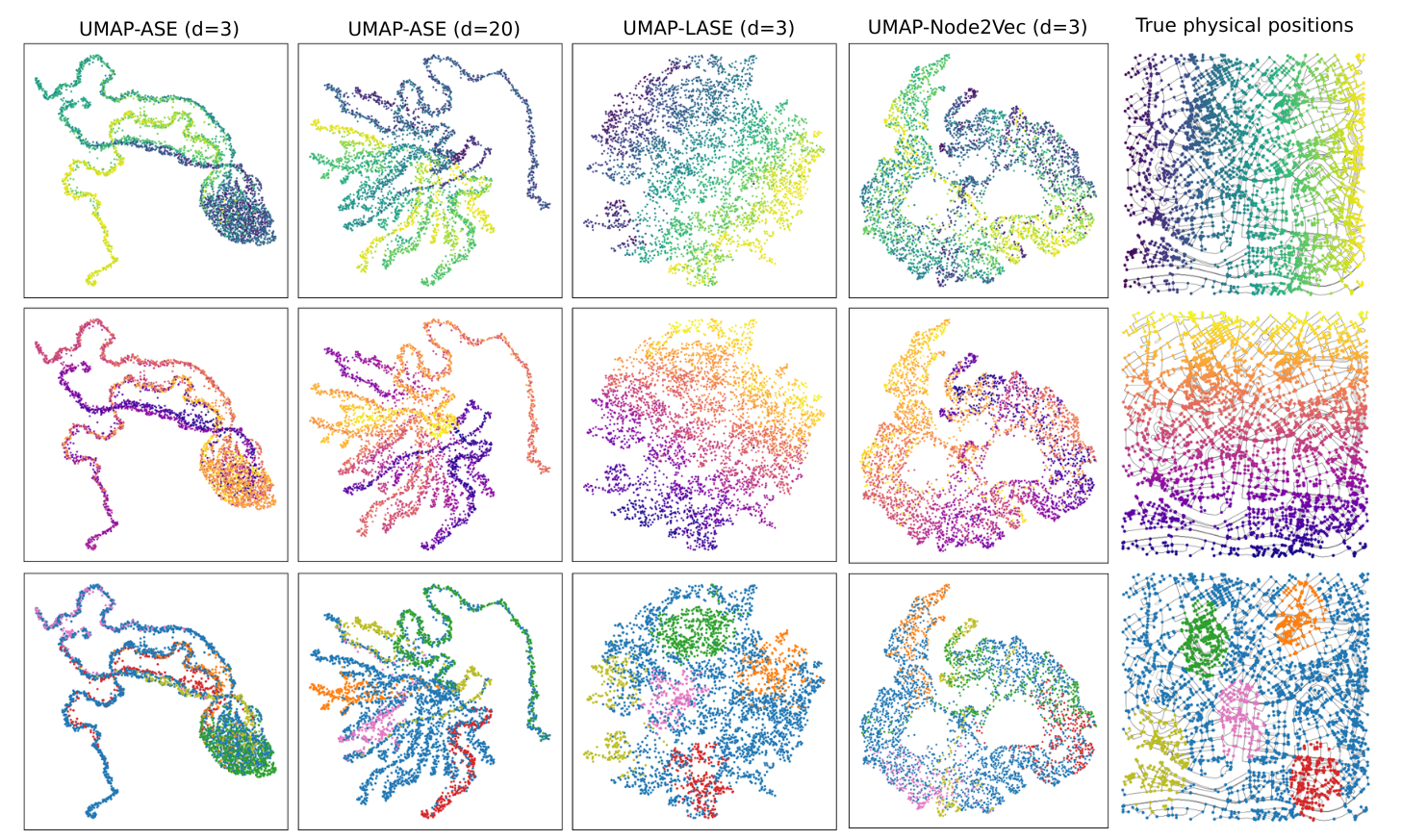}
     \caption{Global embeddings of the Bristol road network, analogous to Figure~\ref{fig:bristol_umap}, with an additional column for UMAP applied to a 3-dimensional Node2Vec embedding. }
     \label{fig:bristol_n2v}
 \end{figure}

 \begin{figure}[h]
     \centering
     \includegraphics[width=1\linewidth]{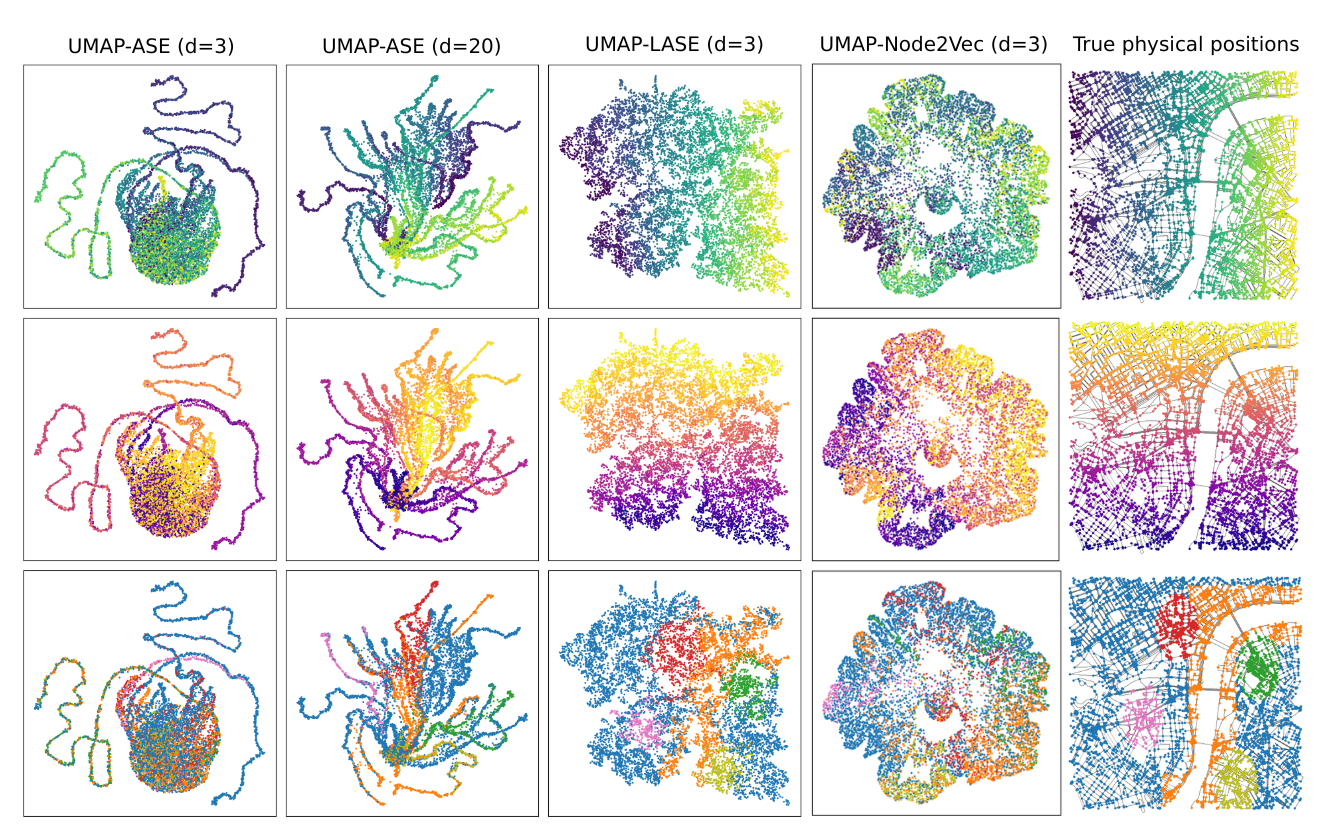}
     \caption{Global embeddings of the London road network, analogous to Figure~\ref{fig:london_umap}, with an additional column for UMAP applied to a 3-dimensional Node2Vec embedding. }
     \label{fig:london_n2v}
 \end{figure}

 \begin{table}[h]
    \centering
    \begin{tabular}{lcc}
        \toprule
        \textbf{Method} & \textbf{Bristol} & \textbf{London} \\
        \midrule
        UMAP-ASE ($d=3$)     & 17.3  & 26.8 \\
        UMAP-ASE ($d=20$)    & 16.9  & 13.9 \\
        UMAP-LASE ($d=3$)    & 5.7   & 19.9 \\
        UMAP-Node2Vec ($d=3$)    & 122.5 & 308.5 \\
        \bottomrule
    \end{tabular}
    \caption{Run-times (in seconds) for UMAP applied to different embeddings of the Bristol and London road networks.}
    \label{tab:umap_runtimes_new}
\end{table}

\end{document}